%% file: mythesis.tex
\pdfoutput=1

\documentclass[11pt,Chicago]{uuthesis2e}

\usepackage {amsthm}

\usepackage{mathpazo}

\fourlevels

\usepackage[letterpaper]{geometry}
\usepackage {amsmath}
\usepackage {amssymb}
\usepackage {bm}
\usepackage {bibnames}

\usepackage {citesort}
\usepackage {graphicx}
\usepackage {graphpap}
\usepackage {longtable}
\usepackage {multirow}
\usepackage {tikz}
\usetikzlibrary {decorations.markings}
\usepackage {varioref}
\usepackage {graphics}
\usepackage {uuthesis-2016-h}  

\usepackage {mythesis}
\usepackage{mystyle}

\author                 {Sunipa Dev}
\title                  {The Geometry of Distributed Representations for Better Alignment, Attenuated Bias, and Improved Interpretability}
\thesistype             {dissertation}

\dedication             {For my family.}

\degree                 {Doctor of Philosophy}

\approvaldepartment     {Computing}
\department             {School of Computing}
\graduatedean           {David B Kieda}
\departmentchair        {Mary Hall}

\committeechair         {Jeffrey Phillips}
\firstreader            {Vivek Srikumar}
\secondreader           {Suresh Venkatasubramanian}
\thirdreader            {Kai-Wei Chang}
\fourthreader           {Yan Zheng}
\chairtitle             {Professor}

\submitdate             {December 2020}
\copyrightyear          {2020}

\chairdateapproved      {10/02/2020}
\firstdateapproved      {10/02/2020}
\seconddateapproved     {10/02/2020}
\thirddateapproved      {10/02/2020}
\fourthdateapproved     {10/02/2020}

\begin{document}

\frontmatterformat
\titlepage
\copyrightpage
\dissertationapproval
\setcounter {page}     {2}             
\preface    {abstract} {Abstract}
\dedicationpage
\tableofcontents
\listoffigures
\listoftables
%



\preface{acknowledge}{Acknowledgements}


\maintext       

\pagestyle{headings} 

\include {chap1}

\include {background}

\include {chap2}

\include {chap3}
\include {chap-visa}
\include {chap4}

\include {chap5}

\include {conclusion}

\numberofappendices = 1
\appendix       

\include {appa}
\nocite{*}
\bibliographystyle{abbrv}
\bibliography{icpbib}
\end {document}

%% file: chap1.tex

\chapter{Introduction}
\label{chap: intro}
Distributed embeddings are an efficient way to represent feature spaces in Natural Language Processing (NLP), computer vision and other domains in data mining and machine learning and have been used extensively for over a decade now. They meaningfully compress large scale data with many features into relatively few dimensions.
This efficiency of distributed embeddings in space and dimensionality reduction is specifically appreciated in domains like NLP, where, to represent the vocabulary of most languages in the form of simple bags of words models or one hot encodings would mean word representations of about $10^5$ or more dimensions per word. This leads to immense inefficiency in terms of using or storage of very large and sparse vectors for large vocabularies. The trade-off in using efficient distributed embeddings, however, lies in losing the interpretability that the full $10^5$ or so dimensional embedding would have had. The exact information encoded by a specific dimension is occluded and no longer known. Similar notions of a lack of interpretability in representations are also seen in other such embeddings, such as interaction or transaction based user and merchant representations and knowledge graph representations. Thus, understanding the structure and relative geometry of points after being represented in a distributed fashion is an important challenge that needs to be addressed. An increase in the explainability of these representations plays a direct and significant role in enriching respective downstream tasks along with making them more coherent, as we show in this work.

This dissertation strategically works towards achieving a greater understanding of the structure of distributed embeddings. Some direct implications are (i) understand and leverage the stability of relative associations between points for different target tasks, (ii) enhanced interpretability, or an understanding of where concept subspaces are located in the embedding space, and (ii) fairness by isolation and disentanglement of invalid and pernicious associations between concept subspaces in such embeddings. The goal here is to identify, isolate and decouple with precision, subspaces capturing specific features in the embedding. 
This also serves an important purpose of enriching different embeddings by selectively retaining or removing features as suited for specific downstream tasks, in a cost efficient, post-processing manner. 

The questions and methods described in this dissertation are applicable to a wide variety of distributed representations as demonstrated in Chapter \ref{chap: visa} . However, a primary focus here has been text representations. The motivation is as follows:

\begin{enumerate}

\item  They are the most abundant form of consistent embeddings. There are multiple embeddings of the same vocabulary (from different data sources and embedding mechanisms) and also embeddings for multiple languages and vocabularies. 

\item When created from sufficiently large and clean corpora such as the Wikipedia dumps of languages, they also have lesser noise than other embeddings such as points of interest embeddings from FourSquare data or merchant transaction based embeddings or online user activity data as described in a later chapter. The latter type is user-generated and is more prone to erratic changes as compared to language, making language representations a more controlled setting. 

\item Finally, word embeddings have interpretable markers in the form of word meanings. Thus, relationships between word representations are more interpretable as well. This implies that filtering or choosing sections of data or embeddings can be done with high confidence regarding coherence.
\end{enumerate}

In this chapter, we define the different problems and questions addressed in this dissertation.

\section{Embedding Structure and Alignment}
Embedding data into high-dimensional spaces is a realm much explored in different modules of data mining and machine learning. There are many different ways to embed the same set of points in some $n$ dimensional space. 
A set of points $A$, when embedded using different mechanisms in some $n$ dimensional space, occupy different regions of this space. They are oriented differently, have different origins and are likely scaled differently as well. However, we hypothesize that the relative distances and orientation between the points in $A$, irrespective of the mechanism used should be the same. Further, a change in the data used to model these points should also not impact the relative associations between the points. In Chapter \ref{chap: alignment}, we bolster our hypodissertation by extending the method by Horn \etal of aligning different high-dimensional embeddings of the same set of points. We demonstrate~\cite{absor} how using linear transformations consisting of at most a single step each of rotation, translation and scaling, we can align point sets onto each other as long as a one-to-one correspondence between the two point sets is known. We use word embeddings for our experiments and show how that has direct implications on downstream tasks such as language translation. 

\section{Identifying Concept Subspaces}

As described above, distributed representations have enabled us to compress large amounts of data and features into few dimensions but at the cost of a lack of interpretability of what features or information gets encoded in what dimension.
 Different features or concepts are thus captured in surreptitious ways in a combination of the explicit dimensions. These concepts no longer are orthogonal to each other in this space, rather have been arranged to compress a large number of concepts and features into few dimensions. For instance, in language representations, these concepts could be sentiments or attributes like gender or race. A lack of orthogonality between specific subspaces can thus be problematic, such as race and sentiments, such as the notion of good versus bad.

To intercept this and for increasing the interpretability of these language representations, identification of subspaces is an important task. Thus, we try to understand and locate within representations, the subspaces which capture different concepts. In Chapter \ref{chap: bias paper 1}, we define a new method to identify concept subspaces using small lists of related seed words. 
Further, in Chapter \ref{chap: visa}, we extend our methods to noisier embeddings with lesser structure than language representations. We demonstrate how to identify subspaces in transaction-based user and merchant representations where we do not have seed words for determination of concepts. 


\section{Disentangling Subspaces}
Associations in language representations are derived from data. While this allows it to learn associations, it can lead to the learning of associations that are invalid or socially biased. For instance, the association between different occupations and gender or age and ability does not exist in language itself but is learned from data by the embeddings. It is also seen in other types of data sources, such as race and location or zip code in income and transaction data. 

Further, these representations have been known to not just imbibe such invalid associations but also amplify them. This can lead to unwanted and often harmful outcomes when these representations are used as features for different modeling tasks.

In Chapters \ref{chap: bias paper 1}, \ref{chap: bias paper 2}, and \ref{chap: bias paper 3}, we introduce different ways to decouple invalid and biased associations between word groups and compare with other existing methods with a variety of metrics, to demonstrate their efficiency in debiasing representations while retaining valid information contained by the embeddings.
Chapter \ref{chap: bias paper 1} introduces a one-step, continuous, post-processing based approach for decoupling subspaces in context-free, static embeddings such as \GloVe representations. This approach relies on the removal of a subspace. Chapter \ref{chap: bias paper 2} extends this approach to be applicable to contextual embeddings which are more fluid in nature, such as BERT.
Finally, Chapter \ref{chap: bias paper 3} introduces a more controlled approach of decoupling than the removal based method. This method is based instead on orthogonalization of specific subspaces which minimizes information loss along with subspace disentanglement.

\section{Measuring Validity of Associations}
Associations learned by representations are data reliant and are not always valid as discussed in the previous section. As a result, it is important to verify the validity of such associations especially when they might be about sensitive or protected attributes. 

In Chapters \ref{chap: bias paper 1}, \ref{chap: bias paper 2}, and \ref{chap: bias paper 3}, we define a range of different intrinsic and extrinsic probes and metrics to examine and quantify the amount of valid and invalid associations along specific concepts in the embeddings. While our intrinsic probes are rooted in vector distances and thus, can be applicable to representations other than word embeddings, our extrinsic ones are specifically for language processing. The probes together help underline the need for such specific tests to score representations on, before deployment into different real world applications.

%% file: background.tex
\chapter{Background Knowledge}
\label{chap: background}
Before we detail each of the problems introduced in Chapter \ref{chap: intro}, in this chapter, we visit some of the concepts that recur through the dissertation and are integral to it. 
\section{Word Embeddings}
Word embeddings are an increasingly popular application of neural networks wherein enormous text corpora are taken as input and words therein are mapped to a vector in some high-dimensional space. %
On an intrinsic level, these word vector representations estimate the similarity between words based on the context of their nearby text, or to predict the likelihood of seeing words in the context of another.  Richer properties were discovered such as synonym similarity, linear word relationships, and analogies such as \word{man} : \word{woman} :: \word{king} : \word{queen}.  Extrinsically as well, their use is now standard in training complex language models and in a variety of downstream tasks. 

Two of the first approaches to be used commonly are \WordToVec~\cite{Mik1,wordtovec} and \GloVe~\cite{glove}. 
These word vector representations began as attempts to estimate the similarity between words based on the context of their nearby text or to predict the likelihood of seeing words in the context of another.  Other more powerful properties were discovered.  Consider each word gets mapped to a vector $\ve{word} \in \R^d$.  

\begin{itemize}
\item \emph{Synonym similarity:}  
Two synonyms (e.g., $\ve{car}$ and $\ve{automobile}$) tend to have small Euclidean distances and large inner products and are often nearest neighbors.  

\item \emph{Linear relationships:}  
For instance, the vector subtraction between countries and capitals (e.g., $\ve{Spain}-\ve{Madrid}$, $\ve{France}-\ve{Paris}$, $\ve{Germany}-\ve{Berlin}$) are similar.  Similar vectors encode gender (e.g., $\ve{man}-\ve{woman}$), tense ($\ve{eat}-\ve{ate}$), and degree ($\ve{big}-\ve{bigger}$).  

\item \emph{Analogies:} 
The above linear relationships could be transferred from one setting to another.  For instance the gender vector $\ve{man}-\ve{woman}$ (going from a female object to a male object) can be transferred to another more specific female object, say $\ve{queen}$.  Then the result of this vector operation is $\ve{queen} + (\ve{man} - \ve{woman})$ is close to the vector $\ve{king}$ for the word ``\s{king}.''   This provides a mechanism to answer analogy questions such as ``\s{woman:man::queen:}?'' 

\item \emph{Classification:}  
More classically~\cite{RR07,RR08,hashkernels,featurehash}, one can build linear classifiers or regressors to quantify or identify properties like sentiments.  
\end{itemize}

\subsection{Word Embedding Mechanisms}
There are several different mechanisms today to represent text as high-dimensional vectors. They can primarily be divided into mechanisms: first, those that produce a single vector for each word (e.g., \GloVe~\cite{glove}), thus leading to a dictionary like structure, and the second (ELMO~\cite{Peters:2018}, BERT~\cite{bert}) producing a function instead of a vector for a word such that given a context, the vector for the word is generated. This implies that different word senses (river $bank$ versus financial institution $bank$) lead to differently embedded representations of the same word.

We will describe some of the word embedding methods here, which we have used in the rest of the dissertation:

\begin{itemize}
   
 \item \textbf{\RAW:} is the most direct of representations. It is explicit and has as many dimensions as there are words, each dimension corresponds with the rate of co-occurrences with a particular word; it is in some sense what other sophisticated models such as \GloVe are trying to understand and approximate at much lower dimensions. 

 \item \textbf{\GloVe:} or Global Vectors, uses an unsupervised learning algorithm~\cite{glove} for obtaining relatively low dimensional(~300) vector representations for words based on their aggregated global word-word co-occurrence statistics from a corpus.

 \item \textbf{\WordToVec:} builds representations of words so the cosine similarity of their embeddings can be used to predict a word that would fit in a given context and can be used to predict the context that would be an appropriate fit for a given word.

 \item \textbf{FastText:} scales~\cite{arm2016bag} these methods of deriving word representations to be more usable with larger data with less time cost. They also include sub-word information~\cite{bojanowski2016enriching} to be able to generate word embeddings for words unseen during training. We use FastText word representations for Spanish and French from the library provided (\url{https://fasttext.cc/docs/en/crawl-vectors.html}) for our translation experiments in Section \ref{app : langauges}.

 \item \textbf{ELMo:} or Embeddings from Language Models is a deep, contextualized word representation~\cite{Peters:2018} that produces contextual representations of words along with modeling their semantics and syntax. This helps represent the polysemy of words. These word representations are trained on large corpora and produce not static vectors but rather, learned functions of a deep bidirectional language model (biLM). ELMo typically produces three layers of representations for every token or word, which are combined in different ways (concatenation or interpolation) to produce a single representation.

 \item \textbf{BERT: } or Bidirectional Encoder Representations from Transformers~\cite{bert} uses a transformer based architecture to produce contextual representations for words. The bi-directional training model helps it to read whole sentences and contexts and not just sequentially in a left-to-right manner that is common. 
BERT commonly has two forms, BERT base and BERT large. While the former has 12 layers of representations, the latter has 24.

 \item \textbf{RoBERTa: } or Robustly Optimized BERT Pretraining Approach improves upon BERT by tuning and modifying some of the key hyperparameters of BERT, including the data size, learning rate and the next sentence prediction task that BERT is trained on.  

\end{itemize}
Contextual representations are now more commonly used in language processing tasks because of their state-of-the-art performance and their ability to distinguish between word senses.

We discuss in detail the settings (data, hyperparameters, pretrained models, etc.) we used for each of these embeddings for specific tasks in subsequent chapters.

\section{Biases in Language Representations}
The word ``bias" may have several subtly different meanings. The notion of bias used in this dissertation predominantly refers to the association of stereotypical terms, words and perceptions with protected attributes of people, such as gender, race, ethnicity, age, etc.

It has been observed~\cite{crime,credit} that such societal biases creep into different machine learning tasks and applications in many ways, one of which are the datasets they are trained on. Language representations also face the same predicament~\cite{Caliskan183,debias,ravfogel2020null}, wherein the biases in underlying text get picked up and amplified by them, such as the differential association of occupations to genders (doctor : male and nurse : female). We explore this in more detail in subsequent chapters. 

A universally accepted and broad definition for what and how bias manifests as in language representations is not determined and is perhaps, one of the ultimate goals that this active area of research is trying to achieve~\cite{blodgett2020language}.
In this dissertation, we define bias in language representation to be the invalid and stereotypical associations made by the representations about the aforementioned protected attributes.  

An important consideration when understanding bias in language modeling is the cause of bias. Possible causes consist of the underlying training text that the embeddings are derived from, the modeling choices made, the training data and methods of the language modeling task itself and so on. A lot of different metrics and probes that try to evaluate bias, capture bias from one or more of these sources. While the values alone of these metrics' results hold little value on their own, when different attempts are made to mitigate biases by altering only specific sections of the language modeling process, a change in said scores can tell us about the impact of altering these sections on biases contained. In this dissertation, we focus on biases in language representations, their sources and mitigation and evaluation techniques.

Representational harm~\cite{barocas2016big}, that we focus on, consists of stereotypes about different demographical and social groups. It is a reflection of the stereotypes that preexist in society that often get amplified\cite{credit,ZhaoWYOC17} and perpetuated by resulting in unfair associations of attributes with groups of people. This bias is observed in language representations in terms of invalid vector distances and associations, unfair contextual or sentence associations or pernicious outcomes in tasks where these representations have been used as features. We explore and quantify this in much greater detail Chapters \ref{chap: bias paper 1}, \ref{chap: bias paper 2}, and \ref{chap: bias paper 3}. However, this representational harm can lead to allocational harm~\cite{barocas2016big,blodgett2020language} or the differential allocation of resources, wealth, jobs, etc. as these representations are actively used in different tasks and applications. In Chapter \ref{chap: bias paper 3} especially, we see that representational biases do not just get reflected intrinsically and in vector distances, but rather affect downstream tasks of contextual nature, such as textual entailment.  

Bias in representations thus constitute an important challenge to be addressed before we accept the widespread usage of word representations across different domains and tasks.

\section{Existing Tests and Metrics}
Some common intrinsic ways to evaluate the quality of embeddings rely on word vector associations. These are based on vector distances and are thus more meaningful in the context-free setting of word representations.

\subsection{Similarity and Analogy Tests}
In an embedding space, if two vectors have a high cosine similarity, i.e., are aligned closely, they are seen to be similar mathematically. However, if the words they represent are also similar in the language itself, then, the similarity of the representations is meaningful. This is what is measured by different similarity tests that score the mathematical similarities of the representations against human assigned similarities of the words. The human assigned scores and mathematical scores are compared using Spearman Coefficient typically, with a high score implying that meaningful associations were learned by the embedding. Some commonly used similarity tests were developed using word pair lists for the same~\cite{wsim}.

The ability of word representations to understand linear relations has been a major proponent for their usage in general. As a result, word embeddings can meaningfully complete analogies such as ``\word{man} : \word{woman} :: \word{king} : \word{queen}" or ``\word{big} : \word{bigger} :: \word{small} : \word{smaller}", wherein providing the first three words in the analogy prompts the embedding to complete it with the fourth word in the analogy.

This property was expanded into a metric to evaluate the embeddings on their ability to understand relationships between words. A commonly used analogy based test~\cite{wordtovec}, developed by Google is comprised of over $15,000$ such analogies.
\subsection{WEAT Test for Biases}
\label{sec: weat}
There also is a need to test word embeddings for the biases they carry, as has been observed by earlier work~\cite{Caliskan183} and in this dissertation.

Word Embedding Association Test (WEAT) was defined as an analogue to Implicit Association Test (IAT) by Caliskan \etal \cite{Caliskan183}. It checks for human like bias associated with words in word embeddings. For example, it found career oriented words (executive, career, etc.) more associated with male names and male gendered words (``man",``boy" etc.) than female names and gendered words and family oriented words (``family", ``home" etc.) more associated with female names and words than male. We list a set of words used for WEAT by Calisan \etal and that we used in our work below.

For two sets of target words X and Y and attribute words A and B, the WEAT test statistic is first calculated as:
\begin{equation}
\label{weat}
s(X,Y,A,B) = \frac{1}{n_1}\sum_{x\in X} s(x,A,B) - \frac{1}{n_2}\sum_{y\in Y} s(y,A,B)
\end{equation}
where, 
\[s(w,A,B) = mean_{a \in A} cos(a,w) - mean_{b \in B} cos(b,w)\] and $cos(a,b)$,
is the cosine distance between vector a and b and $n_1$ and $n_2$ are the number of words in $X$ and $Y$ respectively.

The score in Equation \ref{weat} is normalized by the standard deviation over $s(w,A,B)$ for all words $w$, in $X \cup Y$ to get the WEAT score. So, closer to 0 this value is, the less bias or preferential association target word groups have to the attribute word groups.
Here target words are occupation words or career/family oriented words and attributes are male/female words or names. Some example sets of words used are listed here and the full list of words used for the test is listed in Appendix \ref{app: words weat}.

Career : \{executive, management, professional, corporation, salary, office, business\} 

Family : \{home, parents, children, family, cousins, marriage, wedding, relatives\}

Male names : \{john, paul, mike, kevin, steve, greg, jeff, bill\} 

Female names : \{amy, joan, lisa, sarah, diana, kate, ann, donna\} 

Male words : \{male, man, boy, brother, he, him, his, son\} 

Female words : \{female, woman, girl, she, her, hers, daughter\} 

So, a typical WEAT test would evaluate the differential association of male and female words (or names) with career and family.

\section{Applications in Other Distributed Representations}
The primary focus of this dissertation is word representations. However, the methods described apply across different domains and types of distributed representations, as demonstrated in Chapter \ref{chap: visa} where we extend some methods on transaction data based merchant representations. These merchant representations are created in a way similar to word embeddings and each high-dimensional point represents a unique merchant or point of sale. We discuss more about the embeddings, the way they are derived and the extensions of some of our methods onto them in Chapter \ref{chap: visa}.

There are different domains with structured data where distributed representations are used diversely to capture contextual information. For instance, node2vec~\cite{grover2016node2vec} represents networks with low-dimensional embeddings based off of simulations of random walks across its nodes. Representations of knowledge graphs~\cite{petar2016rdf2vec} are made in a similar manner.
Another representation, Tag2Image~\cite{gong2012multiview} embeds images and text tags associated jointly in the same space. All of these representations use the context in sequences of data to create their embeddings, thus capturing the structure of the underlying data. They mirror the similarities between data points by the vector similarities, like in word representations. 

As distributed representations are used more and more pervasively to understand and represent structured data efficiently, there is an increased requirement to understand the spatial structure of the data and the feature subspaces contained. The methods described in this dissertation extend to most of these scenarios.

%% file: chap2.tex
\chapter{Closed Form Word Embedding Alignment}
\label{chap: alignment}
In this Chapter, we develop a family of techniques to align word embeddings which are derived from different source datasets or created using different mechanisms (e.g., GloVe or word2vec). 
Our methods are simple and have a closed form to optimally rotate, translate, and scale to minimize the root mean squared errors or maximize the average cosine similarity  between two embeddings of the same vocabulary into the same dimensional space.  
Our methods extend approaches known as Absolute Orientation, which are popular for aligning objects in three-dimensions.  
We prove new results for optimal scaling and for maximizing cosine similarity.  		
Then we demonstrate how to evaluate the similarity of embeddings from different sources or mechanisms, and that certain properties like synonyms and analogies are preserved across the embeddings and can be enhanced by simply aligning and averaging ensembles of embeddings. 

\section{Structure of Word Embeddings}

\vspace{-1mm}

Embedding complex data objects into a high-dimensional, but easy to work with, feature space has been a popular paradigm in data mining and machine learning for more than a decade~\cite{RR07,RR08,hashkernels,featurehash}.  
This has been especially prevalent recently as a tool to understand language, with the popularization through \WordToVec~\cite{Mik1,wordtovec} and \GloVe~\cite{glove}.  These approaches take as input a large corpus of text, and map each word, which appears in the text to a vector representation in a high-dimensional space (typically $d=300$ dimensions).  

These word vector representations, which began as attempts to estimate similarity between words based on the context of their nearby text, exhibit other interesting properties such as synonym similarities between the representations or linear relationships between the representations as described in Chapter \ref{chap: background}.

At least in the case of \GloVe, these linear substructures are not accidental; the embedding aims to preserve inner product relationships.  Moreover, these properties all enforce the idea that these embeddings are useful to think of inheriting a Euclidean structure, i.e., it is safe to represent them in $\R^d$ and use Euclidean distance.  

However, there is nothing extrinsic about any of these properties.  A rotation or scaling of the entire dataset will not affect synonyms (nearest neighbors), linear substructures (dot products), analogies, or linear classifiers.  A translation will not affect distance, analogies, or classifiers, but will affect inner products since it effectively changes the origin.  These substructures (i.e., metric balls, vectors, halfspaces) can be transformed in unison with the embedded data.  
Indeed Euclidean distance is the only metric on $d$-dimensional vectors that is rotation invariant.

The intrinsic nature of these embeddings and their properties adds flexibility that can also be a hindrance.  In particular, we can embed the same dataset into $\R^d$ using two approaches, and these structures cannot be used across datasets.  Or two data sets can both be embedded into $\R^d$ by the same embedding mechanism, but again the substructures do not transfer over.  That is, the same notions of similarity or linear substructures may live in both embeddings, but have different meaning with respect to the coordinates and geometry.  
This makes it difficult to compare approaches; the typical way is to just measure a series of accuracy scores, for instance, in recovering synonyms~\cite{Lev2,Mik1}.  However, these single performance scores do not allow deeper structural comparisons.  

%

Another issue is that it becomes challenging (or at least messier) to build ensemble structures for embeddings.  For instance, some groups have built word vector embeddings for enormous datasets (e.g., \GloVe embedding using 840 billion tokens from Common Crawl, or the \WordToVec embedding using 100 billion tokens of Google News), which costs at least tens of thousands of dollars in cloud processing time.  Given several such embeddings, how can these be combined to build a new single better embedding without revisiting that processing expense?  How can a new (say specialized) data set from a different domain use a larger high-accuracy embedding?

In this chapter, we provide a simple closed form method to optimally align two embeddings.  These methods find optimal rotation (technically an orthogonal transformation) of one dataset onto another, and can also solve for the optimal scaling and translation.  
They are optimal in the sense that they minimize the sum of squared errors under the natural Euclidean distance between all pairs of common data points, or they can maximize the average cosine similarity.  

The methods we consider are easy to implement, and are based on $3$-dimensional shape alignment techniques common in robotics and computer vision called ``absolute orientation.''
We observe that these approaches extend to arbitrary dimensions $d$; the same solution for the optimal orthogonal transformation was also recently derived again by Smith \etal (2017)~\cite{SmithTHH17}.  

In this chapter, we also show that an approach to choose the optimal scaling of one dataset onto another~\cite{Hor87} does not affect the optimal choice of rotation.  Hence, the choice of translation, rotation, and scaling can all be derived with simple closed form operations.  

We then apply these methods to align various types of word embeddings, providing new ways to compare, translate, and build ensembles of them.  
We start by aligning data sets to themselves with various types of understandable noise; this provides a method to calibrate the error scores reported in other settings.  
We also demonstrate how these aligned embeddings perform on various synonym and analogy tests, whereas without alignment the performance is very poor.  The results with scaling, translation, and weighting all consistently improve upon the results for only rotation as advocated by Smith \etal (2017)~\cite{SmithTHH17}.  

Moreover, we show that we can boost embeddings, showing improved results when aligning various embeddings, and taking simple averages of the embedded words from different data sets.  The results from these boosted embeddings provide the best known results for various analogy and synonym tests.  
More extensive use of ensembles should be possible, and it could be applied to a wider variety of data types where Euclidean feature embeddings are known, such as for graphs~\cite{DeepWalk,node2vec,CZC17,GF17,DCS17}, images~\cite{SIFT,SURF}, and for kernel methods~\cite{RR07,RR08}. 

This alignment can also aid translation, wherein an alignment learned from a small set of words whose translation is known, we can obtain an alignment of a much larger set of words.  We also show how aligning two low-resource languages independently to a well-documented and accurate intermediate language can aid in translation between the first two languages. 

Finally, in the last few years, contextualized embeddings, such as BERT \cite{bert} and ELMo \cite{Peters:2018}, which embed a word differently each time, based on the context it appears in, have become increasingly pervasively used in language processing tasks such as textual entailment and co-reference resolution.  We show that a simple average of the contexts allows our techniques to efficiently extend to modeling a more complex multi way alignment among word representations.


\section{Closed Form Point Set Alignment}
\label{sec:AO}

In many classic computer vision and shape analysis problems, a common problem is the alignment of two (often $3$-dimensional) shapes.  The most clean form of this problem starts with two points sets $A = \{a_1, a_2, \ldots, a_n\}$ and $B = \{b_1, b_2, \ldots, b_n\}$, each of size $n$, where each $a_i$ corresponds with $b_i$ (for all $i \in 1,2,\ldots,n$).  Generically we can say each $a_i, b_i \in \R^d$ (without restricting $d$), but as mentioned the focus of this work was typically restricted to $d=3$ or $d=2$.  Then the standard goal was to find a rigid transformation -- a translation $t \in \R^d$ and rotation $R \in \SO(d)$ -- to minimize the root mean squared error (RMSE).  
An equivalent formulation is to solve for the sum of squared errors as
\begin{equation}\label{eq:opt-R+t}
(R^*, t^*)  = \argmin_{t \in \R^d, R \in \SO(d)} \sum_{i=1}^n \|a_i - (b_i R + t)\|^2.
\end{equation}
For instance, this is one of the two critical steps in the well-known iterative closest point (ICP) algorithm~\cite{BM92,CM92}.  

In the 1980s, several \emph{closed form} solutions to this problem were discovered; their solutions were referred to as solving \emph{absolute orientation}.  
The most famous paper by Horn~\cite{Hor87} uses unit quaternions.  However, this approach seems to have been known earlier~\cite{FH83}, and other techniques using rotation matrices and the SVD~\cite{HN81,AHB87}, rotation matrices and an eigen-decomposition~\cite{SS87,Sch66},  and dual number quaternions~\cite{WSV91}, have also been discovered.  In $2$ or $3$ dimensions, all of these approaches take linear (i.e., $O(n)$) time, and in practice, have roughly the same run time~\cite{ELF97}.

In this document, we focus on the Singular Value Decomposition (SVD)-based approach of Hanson and Norris~\cite{HN81}, since it is clear, has an easy analysis, and unlike the quaternion-based approaches which only work for $d=3$, generalizes to any dimension $d$. A singular value decomposition(SVD) factorizes a matrix of dimensions $m \times n$ to produce two orthonormal matrices ($U$ and $V$) and a diagonal matrix ($S$) to satisfy the linear transformation $x = Ax$. The orthonormal matrices capture the rotation or reflection of the space while the diagonal matrix $S$ captures the singular values which interpret the magnitude of information along each of the respective dimensions.
Hanson and Norris's approach decouple the rotation from the translation and solve for each independently. It further uses the orthonormal matrices produced by SVD to determine the rotation. In particular, this approach first finds the means $\bar a = \frac{1}{n} \sum_{i=1}^n a_i$ and $\bar b = \frac{1}{n} \sum_{i=1}^n b_i$ of each data set.  Then it creates centered versions of those data sets $\hat A \leftarrow (A,\bar a)$ and $\hat B \leftarrow (B, \bar b)$.  
%
Next, we need to compute the RMSE-minimizing rotation (all rotations are then considered around the origin) on centered data sets $\hat A$ and $\hat B$.  First compute the sum of outer products $H = \sum_{i=1}^n \hat b_i^T \hat a_i$, which is a $d \times d$ matrix.  We emphasize $\hat a_i$ and $\hat b_i$ are row vectors, so this is an outer product, not an inner product.  Next take the singular value decomposition of $H$ so $[U, S, V^T] = \mathsf{svd}(H)$, and the ultimate rotation is $R = U V^T$.  We can create the rotated version of $B$ as $\tilde B = \hat B R$ so we rotate each point as $\tilde b_i = \hat b_i R$.  



Within this chapter, we will use this approach, as outlined in Algorithm \ref{alg:AO-rotate}, to align several data sets, each of which have no explicit intrinsic properties tied to their choice of rotation.  
We, in general, do not use the translation step for two reasons.  First, this effectively changes the origin and hence the inner products.  Second, we observe the effect of translation is usually small and typically does not improve performance.  

\begin{algorithm}
	\caption{\label{alg:AO-rotate}$\AOR(A,B)$}
	\begin{algorithmic}
		\STATE Compute the sum of outer products $H = \sum_{i=1}^{n} b_i^T a_i$   
		\STATE Decompose  $[U, S, V^T] = \mathsf{svd}(H)$ 
		\STATE Build rotation $R = U V^T$   
		\STATE \textbf{return} $\tilde B = B R$ so each $\tilde b_i = b_i R$  
	\end{algorithmic}
\end{algorithm}

Technically, this may allow $R$ to include mirror flips, in addition to rotations.  These can be detected (if the last singular value is negative) and factored out by multiplying by a near-identity matrix $R = U I_- V^T$ where $I_-$ is identity, except the last term is changed to $-1$.  We ignore this issue in this chapter, and henceforth consider orthogonal matrices $R \in \OG(d)$ (which includes mirror flips) instead of just rotations $R \in \SO(d)$.  For simpler nomenclature, we still refer to $R$ as a ``rotation.''

We discuss here a few other variants of this algorithm which take into account translation and scaling between $A$ and $B$.
 
\begin{algorithm}
	\caption{\label{alg:AO}\AO$(A,B)$ \cite{HN81}}
	\begin{algorithmic}
		\STATE Compute $\bar a = \frac{1}{n} \sum_{i=1}^n a_i$ and $\bar b = \frac{1}{n} \sum_{i=1}^n b_i$  
		\STATE \underline{Center} $\hat A \leftarrow (A,\bar a)$ so each $\hat a_i = a_i - \bar a$, and similarly $\hat B \leftarrow (B, \bar b)$   
		\STATE Compute the sum of outer products $H = \sum_{i=1}^{n} \hat b_i^T \hat a_i$   
		\STATE Decompose  $[U, S, V^T] = \mathsf{svd}(H)$ 
		\STATE Build rotation $R = U V^T$   
		\STATE \underline{Rotate} $\tilde B = \hat B R$ so each $\tilde b_i = \hat b_i R$  
		\STATE Translate $B^* \leftarrow (\tilde B, - \bar a)$ so each $b^*_i = \tilde b_i + \bar a$  
		\STATE \textbf{return} $B^*$
	\end{algorithmic}
\end{algorithm}

Note that the rotation $R$ and translation $t = - \bar b + \bar a$ derived within this Algorithm \ref{alg:AO} are not exactly the optimal $(R^*, t^*)$ desired in formulation (\ref{eq:opt-R+t}).  This is because the order these are applied, and the point that the data set is rotated around is different.  In formulation (\ref{eq:opt-R+t}) the rotation is about the origin, but the dataset is not centered there, as it is in Algorithm \ref{alg:AO}.  

\noindent \textbullet\ \myParagraph{Translations}
To compare with the use of also optimizing for the choice of translations in the transformation, we formally describe this procedure here.  In particular, we can decouple rotations and translations, so to clarify the discrepancy between Algorithm \ref{alg:AO} and equation (\ref{eq:opt-R+t}), we use a modified version of the above procedure.  In particular, we first center all data sets, $\hat A \leftarrow A$ and $\hat B \leftarrow B$, and henceforth can know that they are already aligned by the optimal translation.  Then, once they are both centered, we can then call $\AOR(\hat A, \hat B)$.  
This is written explicitly and self-contained in Algorithm \ref{alg:AO-cent}.

\begin{algorithm}
	\caption{\label{alg:AO-cent}$\AOc(A,B)$}
	\begin{algorithmic}
		\STATE Compute $\bar a = \frac{1}{n} \sum_{i=1}^n a_i$ and $\bar b = \frac{1}{n} \sum_{i=1}^n b_i$  
		\STATE \underline{Center} $\hat A \leftarrow (A,\bar a)$ so each $\hat a_i = a_i - \bar a$, and similarly $\hat B \leftarrow (B, \bar b)$   
		\STATE Compute the sum of outer products $H = \sum_{i=1}^{n} \hat b_i^T \hat a_i$   
		\STATE Decompose  $[U, S, V^T] = \mathsf{svd}(H)$ 
		\STATE Build rotation $R = U V^T$   
		\STATE \underline{Rotate} $\tilde B = \hat B R$ so each $\tilde b_i = \hat b_i R$  
		\STATE \textbf{return} $\hat A$, $\tilde B$
	\end{algorithmic}
\end{algorithm}

%
%
%

\noindent \textbullet\ \myParagraph{Scaling}
In some settings, it makes sense to align data sets by scaling one of them to fit better with the other, formulated as  
$
(R^*, t^*, s^*)  = \argmin_{s \in \R, R \in \SO(d)} \sum_{i=1}^n \|a_i - s (b_i-t) R \|^2.
$
In addition to the choices of translation and rotation, the optimal choice of scaling can also be decoupled.  

Horn \etal~\cite{Hor87} introduced two mechanisms for solving for a scaling that minimizes RMSE.  Assuming the optimal rotation $R^*$ has already been applied to obtain $\hat B$, then a closed form solution for scaling is 
$
s^* 
= \sum_{i=1}^n \langle \hat a_i, \hat b_i \rangle / \|\hat B\|_F^2.  
$
The sketch for Absolute Orientation with scaling, is in Algorithm \ref{alg:AO+scale}.

\begin{algorithm}
	\caption{\label{alg:AO+scale}$\AOs(A,B)$}
	\begin{algorithmic}
		\STATE $\tilde B \leftarrow \AOR(A,B)$
		\STATE Compute scaling $s = \sum_{i=1}^n \langle a_i, \tilde b_i \rangle / \|\tilde B\|_F^2$
		\STATE \textbf{return} $\breve B$ as $\breve B \leftarrow s \tilde B$ so for each $\breve b_i = s \tilde b_i$.  
	\end{algorithmic}
\end{algorithm}

The steps of rotation, scaling and translation fit together to give us algorithm \ref{alg:AO-scale}. 

\begin{algorithm}
	\caption{\label{alg:AO-scale}$\AOsc(A,B)$}
	\begin{algorithmic}
		\STATE Compute $\bar a = \frac{1}{n} \sum_{i=1}^n a_i$ and $\bar b = \frac{1}{n} \sum_{i=1}^n b_i$  
		\STATE \underline{Center} $\hat A \leftarrow (A,\bar a)$ so each $\hat a_i = a_i - \bar a$, and similarly $\hat B \leftarrow (B, \bar b)$   
		\STATE Compute the sum of outer products $H = \sum_{i=1}^{n} \hat b_i^T \hat a_i$   
		\STATE Decompose  $[U, S, V^T] = \mathsf{svd}(H)$ 
		\STATE Build rotation $R = U V^T$   
		\STATE \underline{Rotate} $\tilde B = \hat B R$ so each $\tilde b_i = \hat b_i R$  
		\STATE Compute scaling $s = \sum_{i=1}^n \langle a_i, b_i \rangle / \|B\|_F^2$
		\STATE \underline{Scale} $\breve B$ as $\breve B \leftarrow s \tilde B$ so for each $\breve b_i = s \tilde b_i$.  
		\STATE \textbf{return} $\tilde A, \breve B$
	\end{algorithmic}
\end{algorithm}


Horn \etal~\cite{Hor87} presented an alternative closed form choice of scaling $s$ which minimizes RMSE, but under a slightly different situation.  In this alternate formulation, $A$ must be scaled by $1/s$ and $B$ by $s$, so the new scaling is somewhere in the (geometric) middle of that for $A$ and $B$.  We found this formulation less intuitive, since the RMSE is dependent on the scale of the data, and in this setting, the new scale is aligned with neither of the data sets.  
%
However, Horn \etal~\cite{Hor87} only showed that the choice of optimal scaling is invariant from the rotation in the second (less intuitive) variant.  We present a proof that this rotation invariance also holds for the first variant.  The proof uses the structure of the SVD-based solution for optimal rotation, with which Horn \etal may not have been familiar.

\begin{lemma}\label{lem:scaling}
Consider two points sets $A$ and $B$ in $\R^d$.  
After the rotation and scaling in Algorithm \ref{alg:AO+scale}, no further rotation about the origin of $\breve B$ can reduce the RMSE.  
\end{lemma}
\vspace{-4mm}
\begin{proof}
	We analyze the SVD-based approach we use to solve for the new optimal rotation.  Since we can change the order of multiplication operations of $s b_i R$, i.e. scale then rotate, we can consider first applying $s^*$ to $B$, and then re-solving for the optimal rotation.  Define $\check B = s^* B$, so each $\check b_i = s^* b_i$.  
Now to complete the proof, we show that the optimal rotation $\check R$ derived from $A$ and $\check B$ is the same as was derived from $A$ and $B$.  
	
Computing the outer product sum
$
	\check H = \sum_{i=1}^n \check b_i^T a_i = \sum_{i=1}^n (s^* b_i)^T a_i = s^* \sum_{i=1}^n b_i^T a_i = s^* H,
$
is just the old outer product sum $H$ scaled by $s^*$.  Then its SVD is 
$
	\mathsf{svd}(\check H) \rightarrow [\check U, \check S, \check V^T] = [U, s^* S, V^T],
$
since all of the scaling is factored into the $S$ matrix.  Then since the two orthogonal matrices $\check U = U$ and $\check V = V$ are unchanged, we have that the resulting rotation 
$
	\check R = \check U \check V^T = U V^T = R
$
is also unchanged.  
\end{proof}

\noindent \textbullet\ \myParagraph{Preserving Inner Products}
%
While Euclidean distance is a natural measure to preserve under a set of transformations, many word vector embeddings are evaluated or accessed by Euclidean inner product operations.  It is natural to ask if our transformations also maximize the sum of inner products of the aligned vectors.  Or does it maximize the sum of cosine similarity: the sum of inner products of \emph{normalized} vectors.  Indeed we observe that $\AOR(A,B)$ results in a rotation 
$
	\tilde R  = \argmax_{R \in \SO(d)} \sum_{i=1}^n \langle a_i, b_i R\rangle.   
$

\begin{lemma}\label{lem:inner}
	$\AOR(A,B)$ rotates $B$ to $\tilde B$ to maximize $\sum_{i=1}^n \langle a_i, \tilde b_i\rangle$. If $a_i \in A$ and $b_i \in B$ are normalized $\|a_i\| = \|b_i\|=1$, then the rotation maximizes the sum of cosine similarities $\sum_{i=1}^n \left\langle \frac{a_i}{\|a_i\|}, \frac{\tilde b_i}{\|\tilde b_i\|}\right \rangle$.  
\end{lemma}
\begin{proof}
	From Hanson and Norris~\cite{HN81} we know $\AOR(B)$ finds a rotation $R^*$ so
	\[ 
	R^*  = \argmin_{ R \in \SO(d)} \sum_{i=1}^n \|a_i - (b_i R )\|^2.
	\]
	Expanding this equation we find 
	\[
	R^*  = \argmin_{ R \in \SO(d)} \left( \sum_{i=1}^n \|a_i\|^2 - \sum_{i=1}^n 2\langle a_i,b_iR\rangle + \sum_{i=1}^n \|b_i R\|^2 \right).  
	\]
Now, the length of a vector does not change upon rotation(R), thus, $\|b_iR\|^2 = \|b_i\|^2$. So, since $\|a_i\|^2$ and  $\|b_i\|^2$ are both lengths of vectors and thus, properties of the dataset, they do not depend on the choice of $R$ and as desired 
	\[
	R^* = \argmax_{ R \in \SO(d)}  \sum_{i=1}^n \langle a_i, b_iR\rangle.
	\]
	
	If all $a_i, b_i$ are normalized, then $R$ does not change the norm $\|\tilde b_i\| = \|b_i R\| = \|b_i\| =1$.  So for $\tilde b_i = b_i R$, each 
	$\langle a_i, \tilde b_i\rangle = \langle \frac{a_i}{\|a_i\|}, \frac{\tilde b_i}{\|\tilde b_i\|} \rangle$ 
	and hence, as desired, 
	\[
	R^* = \argmax_{R \in \SO(d)} \sum_{i=1}^n \left\langle \frac{a_i}{\|a_i\|}, \frac{b_i R}{\|b_i R\|}\right \rangle. \qedhere
	\]
\end{proof}

Several evaluations of word vector embeddings use cosine similarity, so it suggests first normalizing all vectors $a_i \in A$ and $b_i \in B$ before performing $\AOR(A,B)$.  However, we found this does not empirically work as well.  The rationale is that vectors with larger norms tend to have less noise and are supported by more data.  So the unnormalized alignment effectively weights the importance of aligning the inner products of these vectors more in the sum, and this leads to a more stable method.  Hence, in general, we do not recommend this normalization preprocessing.

\subsection{Extension to Contextualized Embeddings}
\label{sec : contextual}
In recent years, contextualized embeddings such as ELMo \cite{Peters:2018} and BERT \cite{bert} have become increasingly popular, because of their ability to express the polysemy of words. A word in these frameworks is not expressed as a single vector, but rather, based on its different meanings or different contexts it has been used in.  That is, each instance of a word is represented by a different vector in the embedding space.  Our method to align individual vectors does not directly apply in this scenario.  

We propose a simple extension to handle this scenario.  Given a word $w_i$ with two separate contextual embeddings, let these embedding vectors be two sets $A_i = \{a_{i,1}, a_{i,2}, \ldots a_{m_{A,i}}\}$ and $B_i = \{b_{i,1}, b_{i,2},  \ldots, b_{m_{B,i}}\}$ of sizes $m_{A,i}$ and $m_{B,i}$, respectively.  Then our method, instead of aligning a single pair of vectors for each word, it aligns \emph{all} vector pairs for each word.  For instance, for finding the optimal rotation, this involves an alignment for $n$ words, each $i$th word $w_i$, then the outer product matrix is defined
\[
H = \sum_{i=1}^n \frac{1}{m_{A,i} m_{B,i}} \sum_{j =1}^{m_{A,i}} \sum_{j'=1}^{m_{B,i}} a_{i,j}^T b_{i,j'},
\]
where each set of all-pairs is weighted equally for each $i$ (this is accomplished by dividing by the number of such pairs $m_{A,i} m_{B,i}$.)

This all-pairs alignment can be computationally expensive as the number of instances of each word $m_{A,i}$ and $m_{B,i}$ increase; even if we only use $10$ instances of each word, in each embedding, this increases the number of alignments by a factor $100$.  
However, we observe, in each step, the set $A_i$ and $B_i$ can be replaced by their averages.
\[
 \bar a_i = \frac{1}{m_{A_i}} \sum_{j=1}^{m_{A,i}} a_{i,j}
 \text{ and }
 \bar b_i = \frac{1}{m_{B_i}} \sum_{j=1}^{m_{B,i}} b_{i,j}.
\]
Then the overall means $\bar b$, $\bar a$, outer product $H$, and scaling $s$ are the same using all instances or the mean instance.  

\begin{lemma}\label{lem:allpair}
The alignments found using all-pair alignment when each word has multiple instances in each embedding is equivalent to that computed by aligning the averages of each set of instances.  
\end{lemma}
\begin{proof}
We need to analyze the $4$ quantities computed as part of any transformation: the two averages, the outer product, and the scale.  
In short, these are all linear vector operations (sum, outer product, inner product), so a vector average can be factored out.  

For each average 
\[
\bar a = \frac{1}{n} \sum_{i=1}^n \frac{1}{m_{A,i}} \sum_{j=1}^{m_{A,i}} a_{i,j} = \frac{1}{n} \sum_{i=1}^n \bar a_i,
\]
and similarly for $\bar b$, the calculations are equivalent.  

For the outer product 
\begin{align*}
H &= \sum_{i=1}^n \frac{1}{m_{A,i} m_{B,i}} \sum_{j =1}^{m_{A,i}} \sum_{j'=1}^{m_{B,i}} a_{i,j}^T b_{i,j'}
\\ &= 
\sum_{i=1}^n  \left( \frac{1}{m_{A,i}} \sum_{j =1}^{m_{A,i}} a_{i,j} \right)^T \left( \frac{1}{ m_{B,i}} \sum_{j'=1}^{m_{B,i}} b_{i,j'} \right)
\\ &= 
\sum_{i=1}^n  \bar a_i^T \bar b_i.
\end{align*}

And finally for the scale
\begin{align*}
s 
&= 
\sum_{i=1}^n \frac{1}{m_{A,i} m_{B,i}} \sum_{j =1}^{m_{A,i}} \sum_{j'=1}^{m_{B,i}} \langle a_{i,j}, b_{i,j'} \rangle / \|B\|_F^2
\\ &=
\sum_{i=1}^n \left \langle \frac{1}{m_{A,i}} \sum_{j =1}^{m_{A,i}} a_{i,j},   \frac{1}{m_{B,i}} \sum_{j'=1}^{m_{B,i}}  b_{i,j'} \right \rangle / \|B\|_F^2
\\ &= 
\sum_{i=1}^n \langle \bar a_i, \bar b_i \rangle / \|B\|_F^2,
\end{align*}
where 
\[
\|B\|_F^2 
= 
\sum_{i=1}^n \left \| \frac{1}{m_{B,i}} \sum_{j=1}^{m_{B,i}} b_{i,j} \right\|^2
=
\sum_{i=1}^n \|\bar b_i\|^2.
\]
Note that the normalization term $\|B\|_F^2$ is defined on the average sum of instances for the all-pairs version, since this is a quadratic operation, and otherwise does not factor out.  
\end{proof}

\subsection{Related Approaches}

\label{sec:related}
As mentioned, Smith \etal (2017)~\cite{SmithTHH17} use Algorithm \ref{alg:AO-rotate} to align \WordToVec word embeddings on English and Italian corpuses, and show that this simple approach is effective in translation. Our work can be seen as building on this, in that we show how to interpret the intrinsic accuracy of such an alignment, how to align word vector corpuses created by different mechanisms, and when to use which variant of the closed form solutions.  Additionally, we confirm some of their language translation results and show that it extends to when the embedding mechanisms for the different language corpuses are not the same (e.g., one by \WordToVec and one by \GloVe), as demonstrated in Section \ref{app : langauges}.  

There are several other methods in the literature that attempt to jointly compute embeddings of datasets so that they are aligned, for instance, in jointly embedding corpuses in multiple languages~\cite{Herm,Mik3}.  The goal of the approaches we study is to circumvent these more complex joint alignments. A couple of very recent papers propose methods to align embeddings after their construction but focus on \emph{affine transformations}, as opposed to the more restrictive but distance preserving rotations of our method.  
Bollegala \etal (2017)~\cite{Art1} uses gradient descent, for parameter $\gamma$, to directly optimize
\begin{equation}
\argmin_{M \in \R^{d \times d}} \sum_{i=1}^n \|a_i - b_i M \|^2 + \gamma \|M\|^2_F.  
\end{equation}
%


Another approach, by Sahin \etal (2017)~\cite{Sahin} uses Low Rank Alignment (LRA), an extension of aligning manifolds from LLE~\cite{Boucher}.  This approach has a 2-step but closed form solution to find an affine transformation applied to both embeddings simultaneously. 
Neither approach directly optimizes for the optimal transformation, and requires regularization parameters; this implies if embeddings start far apart, they remain further apart than if they start closer.  
Both find affine transformations $M$ over $\R^{d \times d}$, not a rotation over the non convex $\OG(d)$ as does our approach.  This changes the Euclidean distance found in the original embedding to a Mahalanobis distance that will change the order of nearest neighbors under Euclidian and cosine distance.  
Finally, the LRA approach requires an eigendecomposition of a $2n \times 2n$ matrix, whereas ours only requires this of a $d \times d$ matrix, so LRA is far less scalable.


		

	

\section{Evaluating Accuracy of Variants}

\label{sec:exp}
We evaluate the effectiveness of our methods on a variety of scenarios, typically using the Root Mean Square Error:  
$ 
\mathsf{RMSE}(A,B) = \sqrt{\frac{1}{|A|}\sum_{i=1}^{|A|} \| a_{i} - b_{i}\|^2}.
$
We fix the embedding dimension of each $A$ (the target) and $B$ (the source) at $300$, and assume $|A| = |B| = n=100{,}000$ or in some cases $n' = |A| = |B| = 10{,}000$.

We consider embeddings with \GloVe~\cite{Art1} (our default), or \WordToVec~\cite{Mik1,wordtovec} with Gensim~\cite{rehurek_lrec}, or occasionally \RAW which is just the $L^1$ normalized word count vectors embedded with SVD~\cite{Lev1}. 
We obtain all these three embeddings for our experiments using our default dataset is the $4.57$ billion token English Wikipedia dump, which we found to be made up of $243K$ vocabulary words (distinct tokens). For \WordToVec, the Gensim~\cite{rehurek_lrec} library provides code for obtaining embeddings of a desired dimensionality, and for \GloVe, the code~\cite{Art1} is provided by the authors themselves. To obtain \RAW embeddings, we run a simple bag of words model which enumerates for each word, how many times it appeared with other words in the vocabulary in a sentence, to give us a vector representation for the word. The \RAW word vectors, thus have the same dimension as that of the vocabulary itself. This when normalized captures the pointwise mutual information and is called the Pointwise Mutual Information (PMI) Matrix.
After embedding using each of these three mechanisms, we select the top 100K most frequent words and their corresponding embeddings for our experiments.

We compare against existing \GloVe embeddings of
Wikipedia + Gigaword (\s{G(WG)}, 6 billion tokens, 400K vocab), 
Common Crawl (\s{G(CC42)}, 42 billion tokens, 1.9M vocab), and
Common Crawl (\s{G(CC840)}, 840 billion tokens, 2.2M vocab), 
and the existing \WordToVec embedding 
of Google News (\s{W(GN)}, 100 billion tokens, 3 million vocab). All of these embeddings are available online (\url{https://nlp.stanford.edu/projects/glove/, https://code.google.com/archive/p/word2vec/}) and were downloaded. 

When aligning \GloVe embeddings to other \GloVe embeddings, we use \AOR.  When aligning embeddings from different sources, we use \AOs. 





\noindent \textbullet\ \myParagraph{Default Data Settings} 
In each embedding, we always consider a consistent vocabulary of $n=100{,}000$ words.  
To decide on this set, we found the $n$ most frequent words used in the default Wikipedia dataset and that were embedded by \GloVe. 
In one case, we compare against smaller datasets and then only work with a small vocabulary of size $n' = 10{,}000$ found the same way.  

For each embedding, we represent each word as a vector of dimension $d=300$.  Note that \RAW originally uses an $n$-dimensional vector.  We reduce this to $d$-dimensions by taking its SVD, as advocated by Levy \etal \cite{Lev1}. They demonstrate how \WordToVec implicitly captures the information as the Shifted Pointwise Mutual Information Matrix (SPMI) in low dimensions. They further demonstrate that computing the SVD of the SPMI matrix maintains the structure captured by the full-dimensional matrix.


%
%
%

\subsection{Calibrating RMSE}

\label{sec:calibrate}
In order to make sense of the meaning of an RMSE score, we calibrate it to the effect of some easier to understand distortions.  
To start, we make a copy of $A$ (the default \s{G(W)} embedding -- we use this notation to signify a \GloVe embedding \s{G($\cdot$)} or the default Wikipedia corpus \s{W})) and apply an arbitrary rotation, translation, and scaling of it to obtain a new embedding $B$.  Invoking $\hat A, \hat B \leftarrow \AOsc(A,B)$, we expect that $\mathsf{RMSE}(\hat A, \hat B) = 0$;  we observe RMSE values on the order of $10^{-14}$, indeed almost $0$ withstanding numerical rounding.  



\noindent \textbullet\ \myParagraph{Gaussian Noise}
Next we add Gaussian noise directly to the embedding.  That is we define an embedding $B$ so that each $b_i = a_i + g_i$ where $g_i \sim N_d(0, \sigma I)$, where $N_d(\mu,\Sigma)$ is a $d$-dimensional Gaussian distribution, and $\sigma$ is a standard deviation parameter.  
Then we measure $\mathsf{RMSE}(\hat A, \hat B)$ from 
$\hat A, \hat B \leftarrow \textsc{\AOR}(A,B)$.
Figure \ref{fig: gaussian} shows the effects for various $\sigma$ values, and also when only added to $10\%$ and $50\%$ of the points.  
We observe the noise is linear, and achieves an RMSE of $2$ to $5$ with $\sigma \in [0.1,0.3]$.  


\noindent \textbullet\ \myParagraph{Noise Before Embedding}
Next, we append noisy, \emph{unstructured} text into the Wikipedia dataset with $1$ billion tokens.  We specifically do this by generating random sequences of $m$ tokens, drawn uniformly from the $n = 10$K most frequent words; we use $m = \{0.01, 0.1, 0.5, 1, 2.5\}$ billion.  
We then extract embeddings for the same vocabulary of $n=100$K words as before, from both datasets, and use \AOR{} to linearly transform the noisy one to the one without noise.  
As observed in Figure \ref{fig:RMSE-noise}, this only changes from about $0.7$ to $1.6$ RMSE.  
The embeddings seem rather resilient to this sort of noise, even when we add more tokens than the original data.  

We perform a similar experiment of adding structured text; we repeat a sequence made of $s = \{100, 1000, 10{,}000 \}$ tokens of medium frequency so the total added is again $m = \{10M, 100M, 500M, 1B,$ $2.5B\}$.  Again in Figure \ref{fig:RMSE-noise}(middle), perhaps surprisingly, this only increases the noise slightly, when compared to the unstructured setting.  
This can be explained since only a small percentage of the vocabulary is affected by this noise, and by comparing to the Gaussian noise, when only added to $10\%$ of the data, it has about a third of the RMSE as when added to all data.

%
%

\noindent \textbullet\ \myParagraph{Incremental Data}
As a model sees more data, it is able to make better predictions and calibrate itself more accurately. This comes at a higher cost of computation and time. If, after a certain point, adding data does not really affect the model, it may be a good trade-off to use a smaller dataset to make an embedding almost equivalent to the one the larger dataset would produce.

We evaluate this relationship using the RMSE values when a \GloVe embedding from a smaller dataset $B$ is incrementally aligned to larger datasets $A$ using \AOR. We do this by starting off with a dataset of the first $1$ million tokens of Wikipedia (1M). We then add data sequentially to it to create datasets of sizes of 100M, 1B, 2.5B or 4.57B tokens.  
For each dataset, we create \GloVe embeddings. Then we align each dataset using $\AOR(A,B)$ where $A$ (the target) is always the larger of the two data set, and $B$ (the source) is rotated and is the smaller of the two.  

Figure \ref{fig:RMSE-noise} (right)  shows the result using a vocabulary of $n=100$K and $n' = 10$K.  The small $n'$ is also used since, for smaller datasets, many of the top $100$K words are not seen.  We observe that even this change in data set size, decreasing from $4.57$B tokens to $2.5$B, still results in substantial RMSE.  However, aligning with fewer but better-represented words starts to show better results, supporting the use of weighted variants.   
\subsection{Changing Datasets and Embeddings}

\label{sec:RMSE-data}

Now with a sense of how to calibrate the meaning of RMSE, we can investigate the effect of changing the dataset entirely or changing the embedding mechanism.

\noindent \textbullet\ \myParagraph{Dependence of Datasets}
Table \ref{tbl:RMSE-Datasets+scale}(top) shows the RMSE when the $4$ \GloVe embeddings are aligned with \AOR, either as a target or source.  
The alignment of \s{G(W)} and \s{G(WG)} has less error than either to \s{G(CC42)} and \s{G(CC840)}, likely because they have substantial overlap in the source data (both draw from Wikipedia).  In all cases, the error is roughly on the scale of adding Gaussian noise with $\sigma \in [0.25, 0.35]$ to the embeddings, or reducing the dataset to $10$M to $100$M tokens. 
This is much more alignment error than in other experiments, indicating that the change in the source data set (and likely its size) has a much larger effect than the embedding mechanism.

\noindent \textbullet\ \myParagraph{Dependence on Embedding Mechanism}
We now fix the data set (the default $4.57$B Wikipedia dataset \s{W}), and observe the effect of changing the embedding mechanism: using \GloVe, \WordToVec, and \RAW.  
We now use \AOs{} instead of \AOR, since the different mechanisms tend to align vectors at drastically different scales.

Table \ref{tbl:RMSE-Datasets+scale}(bottom) shows the RMSE error of the alignments; the columns show the target ($A$) and the rows show the source dataset ($B$).  
This difference in target and source is significant because the scale inherent in these alignments change, and with it, so does the RMSE. Also as shown, the scale parameter $s^*$ from \GloVe to \WordToVec in \AOs {} is approximately $3$ (and non symmetrically about $0.25$ in the other direction from \WordToVec to \GloVe). This means for the same alignment, we expect the RMSE to be between $3$ to $4$ ($\approx 1/0.25$) times larger as well.  

However, with each column, with the same target scale, we can compare alignment RMSE.  We observe that the differences are not too large, all roughly equivalent to Gaussian noise with $\sigma = 0.25$ or using only $1$B to $2.5$B tokens in the dataset.  
Interestingly, this is less error than changing the source dataset; consider the \GloVe column for a fair comparison.  This corroborates that the embeddings find some common structure, capturing the same linear structures, analogies, and similarities.  And changing the datasets is a more significant effect.

%
%
%

\subsection{Similarity and Analogies After Alignment}
\label{sec:similarity}
The \GloVe and \WordToVec embeddings both perform well under different benchmark similarity and analogy tests. These results will be unaffected by rotations or scaling.  
Here we evaluate how these tests transfer under alignment.  Using the default Wikipedia dataset, we use several variants of \AO{} to align \GloVe and \WordToVec embeddings.  Then given a synonym pair $(i,j)$ we check whether $b_j \in B$ (after alignment) is in the neighborhood of $a_i$.

More specifically, we use $4$ common similarity test sets, which we measure with cosine similarity~\cite{Lev2}: 
Rubenstein-Goodenough (\RG, 65 word pairs) \cite{rg}, 
Miller-Charles (\MC, 30 word pairs) \cite{mc}, 
WordSimilarity-353 (\WSim, 353 word pairs) \cite{wsim} and 
SimLex-999 (\Simlex, 999 word pairs) \cite{simlex}.
We use the Spearman correlation coefficient (in $[-1,1]$, larger is better) to aggregate scores on these tests; it compares the ranking of cosine similarity of $a_i$ to the paired aligned word $b_j$, to the rankings from a human-generated similarity score.  

Table \ref{tbl:Sim-Analogy} shows the scores on just the \GloVe and \WordToVec embeddings, and then across these aligned datasets.  
To understand how the variants of \AO{} compare, we compute the scores after each of the various optimal transformation types are applied: rotation, then scaling, then translation, and finally, we consider if we normalize all vectors before alignment to maximize cosine similarities. 
Before transformation (``untransformed'') the across-dataset comparison is very poor, close to $0$; that is, extrinsically, there is very little information carried over.  However, alignment with just \AOR{} achieves scores nearly as good as, and sometimes better than on the original datasets.   \WordToVec scores higher than \GloVe, and the across-dataset scores are typically between these two scores.  Adding scaling with \AOs{} has no effect on the scores on the similarity test because they are measured with cosine similarity.  However, also applying the optimal translation does increase the scores even though it optimizes Euclidean distance and not cosine distance.  Perhaps surprisingly, applying rotation along with translation \emph{and} scaling improves more than just applying rotation and translation.  This method applies scaling after the dataset is centered, so this then alters the inner products, and in a useful way.  

We perform the same experiments on $2$ Google analogy datasets~\cite{Mik1}:
\SEM has $8869$ analogies and
\SYN has $10675$ analogies. 
As discussed in Chapter 1, these are of the form ``\s{A:B::C:D}'' (e.g., ``man : woman :: king : queen''), and we evaluate across data sets by measuring if vector $\ve{D}$ is among the nearest neighbors in data set $A$ of vector $\ve{C} + (\ve{B} - \ve{A})$ in data set $B$.  
The results are similar to the synonym tests, where \AOR{} alignment across-datasets performs similar to within either embedding, and scaling and rotation provided small further improvement.  In this case, performing rotation and scaling improves upon just rotation.  This is because the analogies are accessing something more complicated about the embedding, and so adjusting the scale more aligns the Euclidean distance and hence the vector structure needed to succeed in analogies.  

The right part of the table shows the effect of various weightings.  Normalization makes the similarity and analogy scores worse, but weighting by the norms consistently increases the scores.  Moreover, also scaling and rotating (e.g., as w(r+s+t)) improves the scores further.  

We also align \s{G(W)} to \s{G(CC42)}, to observe the effect of only changing the dataset.   The \s{G(CC42)} dataset performs better itself; it uses more data.  The small similarity tests (\s{RG},\s{MC}) show some extrinsic information is captured without any alignment, but otherwise, across-embedding scores have a similar pattern to across-dataset scores.

Next, in Table \ref{tbl : normalized}, we further investigate the effect of various weighting (or normalizing) before alignment.  In these tests, we show the effect on \AOR{} with three types of weighting.  As before, we simply apply \AOR{} on all 100K words.  But we also find the optimal $R$ on only the most frequent 10K words using \AOR{}, and then again using \AON{} on just these 10K words.  The rotation and evaluation are still on all 100K words needed for the tests.  Surprisingly \AON(10K) performs better than \AOR(10K), and comparably to \AOR(100K).  This indicates that similarity optimization is useful when the words all have sufficient data to embed them properly.  

\subsection{Comparison to Baselines }
Next, we perform similarity tests to compare against alignment implementations of methods by Sahin \etal (2017)~\cite{Sahin} (LRA) and Bollegela \etal (2017)~\cite{Art1} (Affine Transformations).  We reimplemented their algorithms, but did not spend significant time to optimize the parameters; recall our method requires no hyperparameters.  
We only used the top $n' = 10$K words for these transformations because these other methods were much more time and memory intensive.  
We only computed similarities among pairs in the top $10$K words for fairness (about two-thirds of the word pairs evaluated, so the scores do not match other tables), and did not perform analogy tests since fewer than one-third of analogies fully showed up in the top $10$K.  
Table \ref{tbl:baseline} shows results for aligning the \s{G(W)} and \s{G(CC42)} embeddings with these approaches.  
Our \AO-based approach does significantly better than the Affine Transformation based method by Bollegela \etal (2017)~\cite{Art1} and generally better than the method lRA by Sahin \etal (2017) \cite{Sahin}.  Our advantage over LRA increases when aligning all $n=100$K words; by comparison, LRA ran out of memory since it requires an $n \times n$ dense matrix decomposition.

\subsection{Dependence of RMSE Variation With Word Frequency }
\label{sec:word-freq}

Table \ref{tbl:word freq} shows some sampled words of various frequencies in the Wikipedia data set.  A word that is more frequently seen in a corpus is generally seen with a larger proportion of other words and contexts, and thus as observed in the table, has a vector representation that has a larger norm than a word that has a low frequency.  This results in the contribution of high-frequency words in the rotation matrix $H$, computed for minimizing the RMSE, to also be larger.  
This larger frequency, and larger norm, also manifests itself in the error after alignment, as shown in the last two columns of Table \ref{tbl:word freq}, both between data sets and between embedding mechanisms.  The relation in the amount of RMSE between words appears even more correlated when between embedding mechanisms (in this case \WordToVec and \GloVe).  The low-frequency words likely exhibit some baseline noise in the case with different data sets (Wiki and CC(42B)), which obscures this relationship for low-frequency words.

\subsection{Discussion on the Right Variant}
\label{sec:which-AO}


Most of the gain using \AO{} is achieved by just finding the optimal rotation $R$ with \AOR.  
However, consistent improvement can be found by weighting the large points more using \AOW{} and by applying translation or scaling, and slightly more by applying both.  

When different datasets are aligned using the same mechanism (e.g., both with \GloVe or both with \WordToVec), then it is debatable whether scaling and translation is necessary, since scaling does not affect cosine similarity, and translation changes intrinsic inner product properties.  However, using a weighting to put more weight on longer (and implied more robustly embedded) words does not alter any intrinsic properties, and only seems to create better alignments.  

When datasets are embedded with different mechanisms (e.g., one with \WordToVec and one with \GloVe) then they are not scaled properly with respect to each other.  In this case, it is important to find the optimal scaling to put them in a consistent interpretable scale, and to ensure analogy relations are optimized.  So we strongly recommend using scaling in this setting.  

%
%
%


\section{Embedding Alignment: Applications}

\label{app : Applications}

We highlight a few applications which may be served by this alignment and comparison mechanisms that we design and demonstrate their effectiveness.

\subsection{Boosting via Ensembles}

\label{sec:boosting}
A direct application of combining different embeddings can be to increase its robustness.  We show that ensembles of pre computed word embeddings found via different mechanisms and on different datasets can boost the performance on the similarity and analogy tests beyond that of any single mechanism or dataset.  The boosting strategy we use here is just simple averaging of the corresponding words after the embeddings have been aligned.  

Table \ref{tbl:boosting weighted} shows the performance of these combined embedding in three experiments.  
The first set shows the default Wikipedia data set under \GloVe (\s{G(W)}), under \WordToVec (\s{W(W)}), and combined ([\s{G(W)}$\odot$\s{W(W)}]).  
The second set shows \WordToVec embedding of Google News (\s{W(GN)}), and combined ([\s{G(W)}$\odot$\s{W(GN)}]) with \s{G(W)}.  
The third set shows \GloVe embedding trained on text from Common Crawl (840B) (\s{G(CC840)}) and then combined with \s{W(GN)} as [\s{G(CC840)}$\odot$\s{W(GN)}].  
Combining embeddings using \AOWts{} consistently boosts the performance on similarity and analogy tests.  Further, we also see that very similar boosting results occur independent of the precise alignment mechanism (e.g., using \AOsc).  
The best score on each experiment is in bold, and in 5 out of 6 cases, it is from a combined embedding.  Moreover, except for this one case, the combined embedding always performs better on all tests that both of the individual embeddings, and in this one case, \s{G(CC804)$\odot$W(GN)} still outperforms \s{W(GN)} on \SEM analogies.  
For instance, remarkably, \s{G(W)$\odot$W(W)} which only uses the default $4.57B$ token Wikipedia dataset, performs better or nearly as well as \s{W(GN)} which uses $100B$ tokens.  Moreover, in some cases the improvement is significant; on the large similarities test \Simlex, the  [\s{G(CC840)}$\odot$\s{W(GN)}] score is $0.443$ or $0.446$ with weights, whereas the best score without boosting is only $0.408$ using \s{G(CC840)}.

\subsection{Aligning Embeddings Across Languages and Embeddings }
\label{app : langauges}

Word embeddings have been used to place word vectors from multiple languages in the same space~\cite{Herm,Mik3}. These either do not perform that well in monolingual semantic tasks as noted in Luong, Pham and Manning ~\cite{Luong} or use learned affine transformations~\cite{Mik3}, which distort distances and do not have closed-form solutions.  Smith \etal \cite{SmithTHH17} use the equivalent of \AOR{} to translate between word embeddings from different languages that have been extracted using the same method. We extend that here to verify that no matter the embedding mechanism, we can translate using a variant of \AO{}. We use the ability to choose the right variant of Absolute Orientation as per Section \ref{sec:which-AO} to orient different embeddings onto each other coherently.
We use the default English \GloVe embedding from Wikipedia and the FastText \url{https://github.com/facebookresearch/fastText} embedding for Spanish.  FastText is yet another unsupervised learning paradigm for obtaining vector representations for words which uses a lot of concepts from \WordToVec, skipgram models and bag of words. As presented, these two have been derived using different methods and are thus oriented differently in 300 dimensional space. We extract the embeddings for the most frequent $5{,}000$ words from the default English Wikipedia dataset (that have translations in Spanish) and their translations in Spanish and align them using \AOWts. We test before and after alignment, for each of these $10{,}000$ words, if their translation is among their nearest $1$, $5$, and $10$ neighbors.
Before alignment, the fraction of words with its translation among its closest $1$, $5$, and $10$ nearest neighbors is $0.00$, $0.160$, and $0.160$, respectively, while after alignment it is $0.372$, $0.623$, and $0.726$, respectively. 


We perform a cross-validation experiment to see how this alignment applies to new words not explicitly aligned.  On learning the rotation matrix above, we apply it to a set of $1000$ new test set of Spanish words (the translations of the next $1000$ most frequent English words) and bring it into the same space as that of English words as before.  We test these $2000$ new words in the embedded and aligned space of $12{,}000$ words (now $6{,}000$ from each language).  
Before alignment, the fraction of times their translations are among the closest $1$, $5$, and $10$ neighbors are $0.00$, $0.00$, and $0.00$, respectively. After alignment it is $0.311$, $0.689$, and $0.747$,  respectively (comparable to results and setup in Mikolov \etal~\cite{Mik3}, using jointly learned affine transformations). Table \ref{tbl:translation-spanish} gives some examples of translations seen using our method.

We perform a similar experiment between English and French, and see similar results. We first obtain $300$ dimensional embeddings for English Wikipedia dump using \GloVe, and for French words from the FastText embeddings.   
Then, we extract the embeddings for the most frequent $10{,}000$ words from the default Wikipedia dataset (that have translations in French) and their translations in French and align them using \AOWts.  We test before and after alignment, for each of these $10{,}000$ words, if their translation is among their nearest $1$, $5$, and $10$ neighbors.
Before alignment, the fraction of words with its translation among its closest $1$, $5$, and $10$ nearest neighbors is $0.00$, $0.054$, and $0.054$ respectively, while after alignment it is $0.478$, $0.755$, and $0.810$, respectively. 
Table \ref{tbl : translation} lists some examples before and after translation.

We again perform a cross-validation experiment to see how this alignment applies to new words not explicitly aligned.  On learning the rotation matrix above, we apply it to a set of $1000$ new test set of French words (the translations of the next $1000$ most frequent English words in the default dataset) and bring it into the same space as that of English words as before.  We test in this space of $22{,}000$ words now, if their translations are among the closest $1$, $5$, and $10$ nearest neighbors of the $2000$ new words ($1000$ French and their translations in English).  Before alignment, the fraction of times their translations are among the closest $1$, $5$, and $10$ neighbors are $0.00$, $0.00$, and $0.00$, respectively.   After alignment it is $0.307$, $0.513$, and $0.698$,  respectively.

\subsection{Aligning Multiple Languages Onto the Same Space}

As demonstrated in Section \ref{app : langauges}, pairwise alignment of words from different languages needs relatively few points to find the alignment to achieve good accuracy in translation between the two languages for a much larger set of words. This allows us to have a low cost operation to map words of one language to their corresponding translated words in the another language. This additionally leads us to a follow-up application. For many language pairs (say languages $L_1$ and $L_2$), we might not have a known dictionary of corresponding word-pairs. In such cases, finding an alignment for enabling translation can be impeded. However, for each of these languages $L_1$ and $L_2$, if corresponding words to a third language $L_3$ is known, aligning both $L_1$ and $L_2$ onto $L_3$ also brings $L_1$ and $L_2$ into the same space. Thus, translation of words from $L_1$ and $L_2$ is enabled without having a set of corresponding seed words in them by which to define the alignment. 
Aligning multiple languages onto the same space can thus, aid in multi way translation. Further, for low resource languages or pairs of languages for whom, only a very small set of translations, i.e., few corresponding points are known, aligning each of these languages to a more common language with which a larger correspondence is known, can help translation.


To demonstrate this, we pick languages $L_1$ and $L_2$ to be Spanish and French respectively. We also pick the common language $L_3$ to be English, to whose word embedding space we align $L_1$ and $L_2$ to. In Table \ref{tbl : multi}, in the first column, we have Spanish to French translations before alignment. As expected, the top $1$, $5$, and $10$ neighboring word accuracies (as evaluated in Section \ref{app : langauges}) are poor (in fact $0$ accuracy). In the second column, we have accuracies after aligning them onto each other using a pool of 2000 words for which we know translations, i.e., their one-to-one correspondences. Next, in the third column, we align both Spanish and French onto English, using the same set of 2000 words and then compare the accuracies for translations from Spanish to French. We find that the top $1$, $5$, and $10$ accuracies are comparable between columns 2 and 3. 
Thus Spanish-French translation was enabled by knowing Spanish-English and French-English associations. This multi way translation enabled by a third language's association leads us to many possibilities of aligning low-resource languages to each other easily.   




\section{Discussion}
\label{sec:discussion}
We have provided simple, closed-form method to align word embeddings.  
Code can be found on github (\url{https://github.com/sunipa/Abs-Orientation}).  	
It allows for transformations for any subset of translation, rotation, and scaling.  These operations all preserve the intrinsic Euclidean structure, which has been shown to give rise to linear structures which allows for learning tasks like analogies, synonyms, and classification.  All of these operations also preserve the Euclidean distances, so it does not affect the tasks which are measured using this distance; note the scaling also scales this distance, but does not change its order.  Our experiments indicate that the rotation is essential for a good alignment, and the scaling is needed to compare embeddings generated by different mechanisms (e.g., \GloVe and \WordToVec) and while helpful, not necessarily when the data set is changed.  Also, the translation provides minor but consistent improvement.  

We also show how to explicitly optimize cosine similarity by first normalizing all words -- however, this does not perform as well as instead optimizing Euclidean distance.  Rather we propose to weight words in the alignment by their norms, and this further improves the alignment because it emphasizes the words which have more stable embeddings.  

This alignment enables new ways that word embeddings can be compared.  This has the potential to shed light on the differences and similarity between them.  For instance, as observed in other ways, common linear substructures are present in both \GloVe and \WordToVec, and these structures can be aligned and recovered, further indicating that it is a well-supported feature inherent to the underlying language (and dataset).  We also show that changing the embedding mechanism has less of an effect than changing the data set, as long as that data set is meaningful.  Unstructured noise added to the input data set appears not to have much effect, but changing from the $4.57$B token Wikipedia data set to the $840$B token Common Crawl data set has a large effect.  

Additionally, we show that by aligning various embeddings, their characteristics, as measured by standard analogy and synonym tests, can be transferred from one embedding to another.  We also demonstrate that cross-language alignment can aid in word translation even when coming from completely different embedding mechanisms, even in a cross-validation setting.  This cross embedding-mechanism alignment opens the door for many other types of alignment word embeddings with embeddings generated from graphs, images, or any other data set, which has some useful word labels.  

Finally, we showed that we can ``boost'' embeddings without revisiting the (sometimes quite enormous) raw data.  This is surprisingly effective in improving scores on similarity and analogy tests, results in the best-known scores from embeddings on these tests.  For instance, on the \Simlex analogy test, we improve upon the best-known scores by almost $10\%$ in the Spearman correlation coefficient.  
There are many other potential applications of these techniques for aligning high-dimensional data embeddings. We propose some scenarios where they may be used in the following section.

\subsection{Other Applications}
  Here, we enumerate a few applications-- we do not experiment on many of these due to the extreme computational cost of performing an analysis of the effect (i.e., the baseline approaches of not using our techniques can be prohibitively expensive, or too qualitative to effectively evaluate).

\begin{enumerate}
	\item Common Crawl is one of the largest textual data sources available.  Moreover, it consistently gets updated to include the ever-increasing data on the internet. Each of these datasets has over 800B tokens, and extracting embeddings from these can be computationally expensive. However, extracting embeddings from the additional data not included in the previous update of Common Crawl should be significantly less expensive.  Aligning an embedding from just the new data, and performing a weighted average with the older larger one may work as well or better than the embedding made from scratch.  
	
	\item 
	A similar weighted average alignment can help with specialized data. Consider data from scientific journals only, or of domain-specific biomedical terms.  Embeddings from just these data sets would be very specialized and each word would have a specific word sense based on the domain.  Aligning these to a gigantic corpus can enrich the specialized domain-related regions on the larger embedding. 
	
	\item 	Tags and phrases in English can be single words or a string of words. Orienting an embedding of tags/phrases along say Common Crawl using an intersection of the single words in the two datasets can help place multi worded tags or phrases around words related to them. This can help derive meaning from random or unknown phrases. Images also often are annotated with a set of tag words.  So orienting a set of tags can help orient images meaningfully among words.
	
	\item 
	These methods can also be applied to even more heterogeneous embeddings than discussed above.  
	We can orient heterogeneous embeddings derived from a variety of methods, e.g., graph embeddings including node2vec~\cite{node2vec} or DeepWalk~\cite{DeepWalk}, and others \cite{CZC17,GF17,DCS17}, images~\cite{SIFT,SURF}, and for kernel methods~\cite{RR07,RR08}.  
	For instance, RDF data can contain shorthand query phrases like ``president children spouse" which answers the question ``who are the spouses and children of presidents?"  By orienting each word along word embeddings from Common Crawl, this may help answer similar questions even more abstractly.  
	Heterogeneous networks have a mixture of node types. If there is an intersection of some nodes (and node types) between any two embeddings (heterogeneous or homogeneous), we can orient them meaningfully.
	
\item Customer data collected at a company over different years and subsidiaries can be embedded using different features (such as income bracket, credit score, location, etc{.} depending on the company). Using common customers over the year, diverse sources and new users can be added meaningfully to the embedding and inferred about, without embedding all of the data points from scratch. 
Moreover, embedding the same users from different years and aligning them can also help deduce the change in their features over time.  
\end{enumerate}


\begin{figure}[]
	\centering
	\includegraphics[width =.55\linewidth]{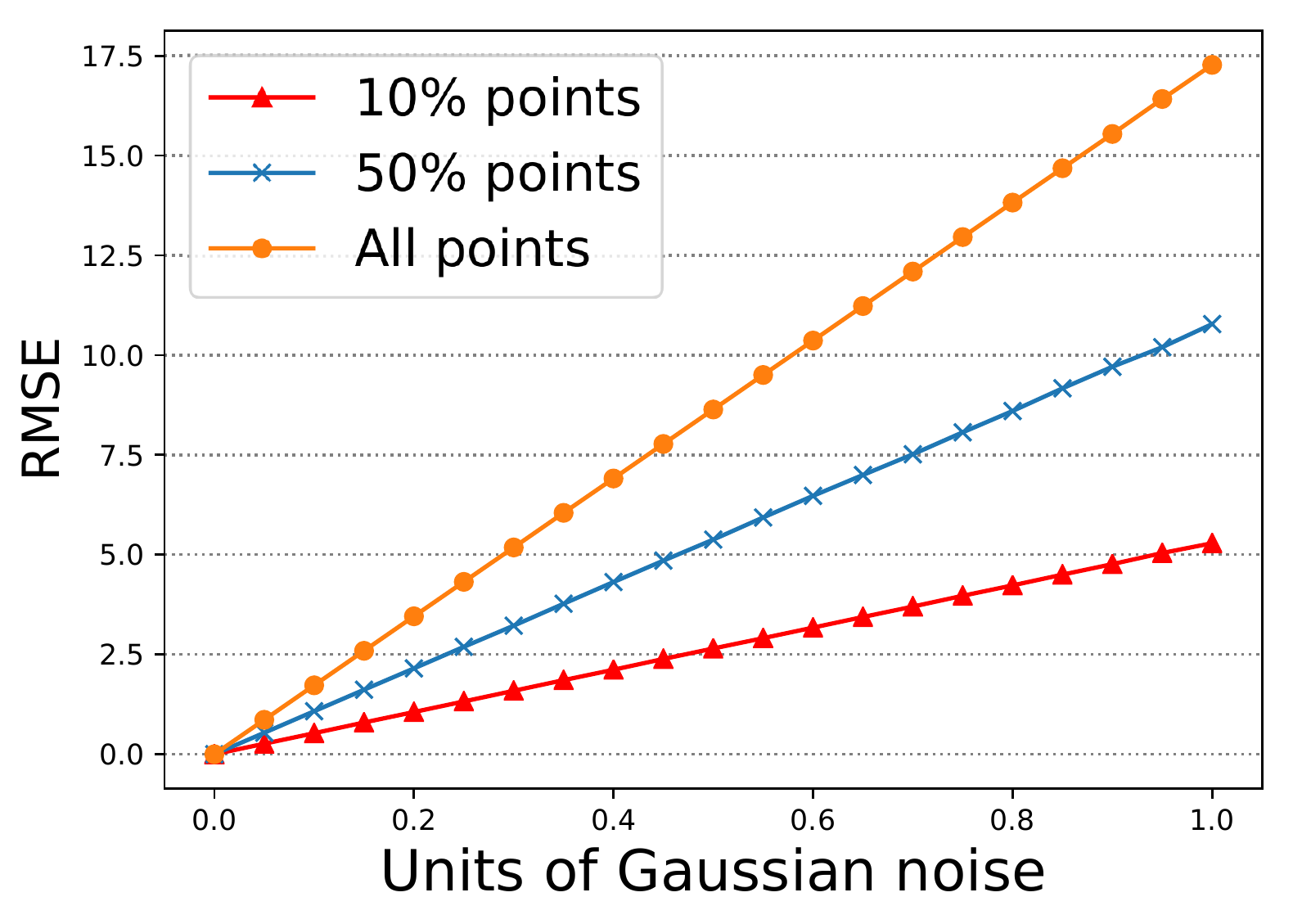} 
		\caption{RMSE Error after noise and \AOR{} alignment adding Gaussian noise to $10\%$, $50\%$, or all points.}
		\label{fig: gaussian}
\end{figure}

	


\begin{figure}
\captionsetup[subfigure]{justification=centering}
    \centering
      \begin{subfigure}{0.45\textwidth}
        \includegraphics[width=\textwidth]{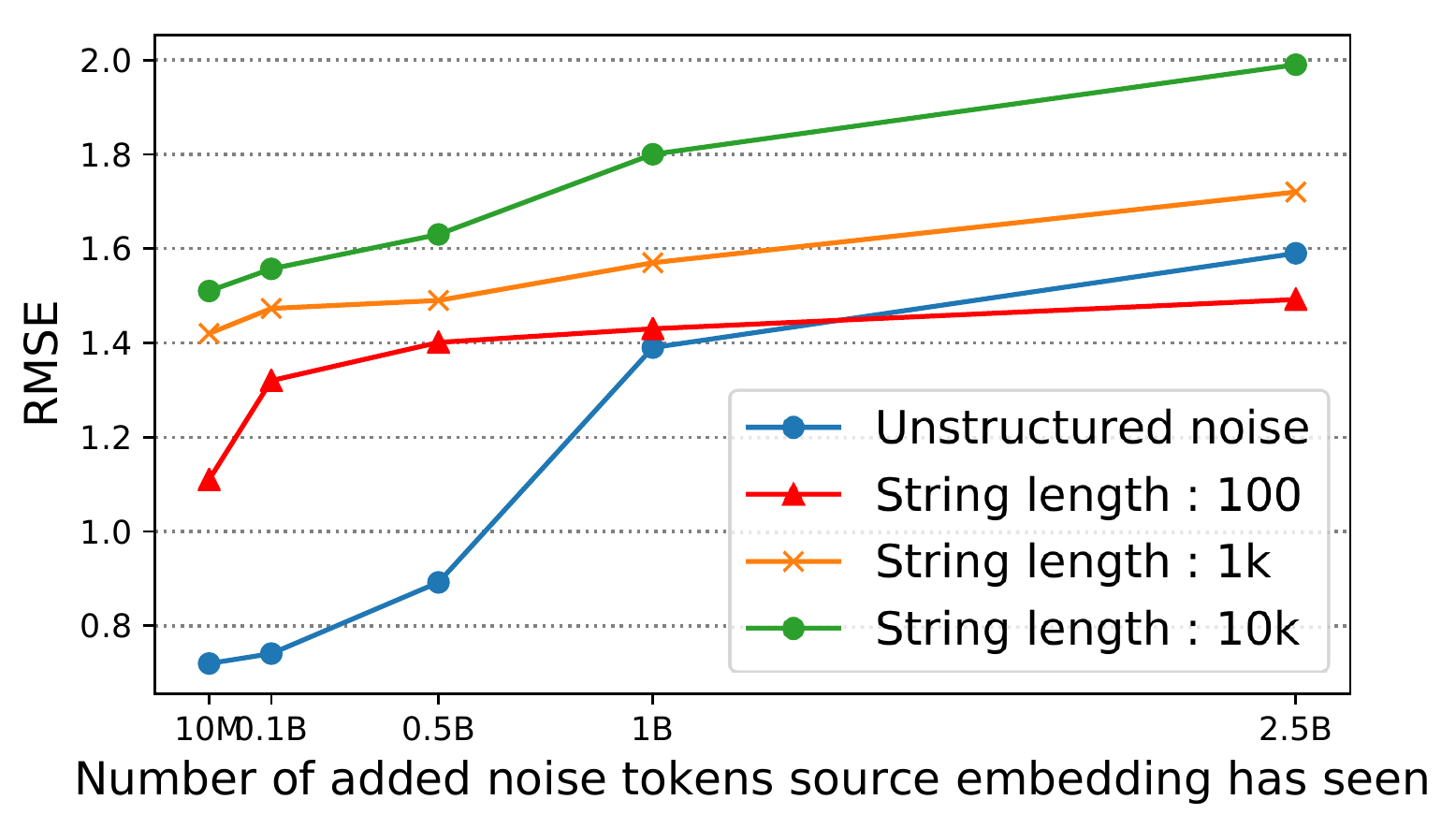}
          \caption{}
      \end{subfigure}
      \begin{subfigure}{0.45\textwidth}
        \includegraphics[width=\textwidth]{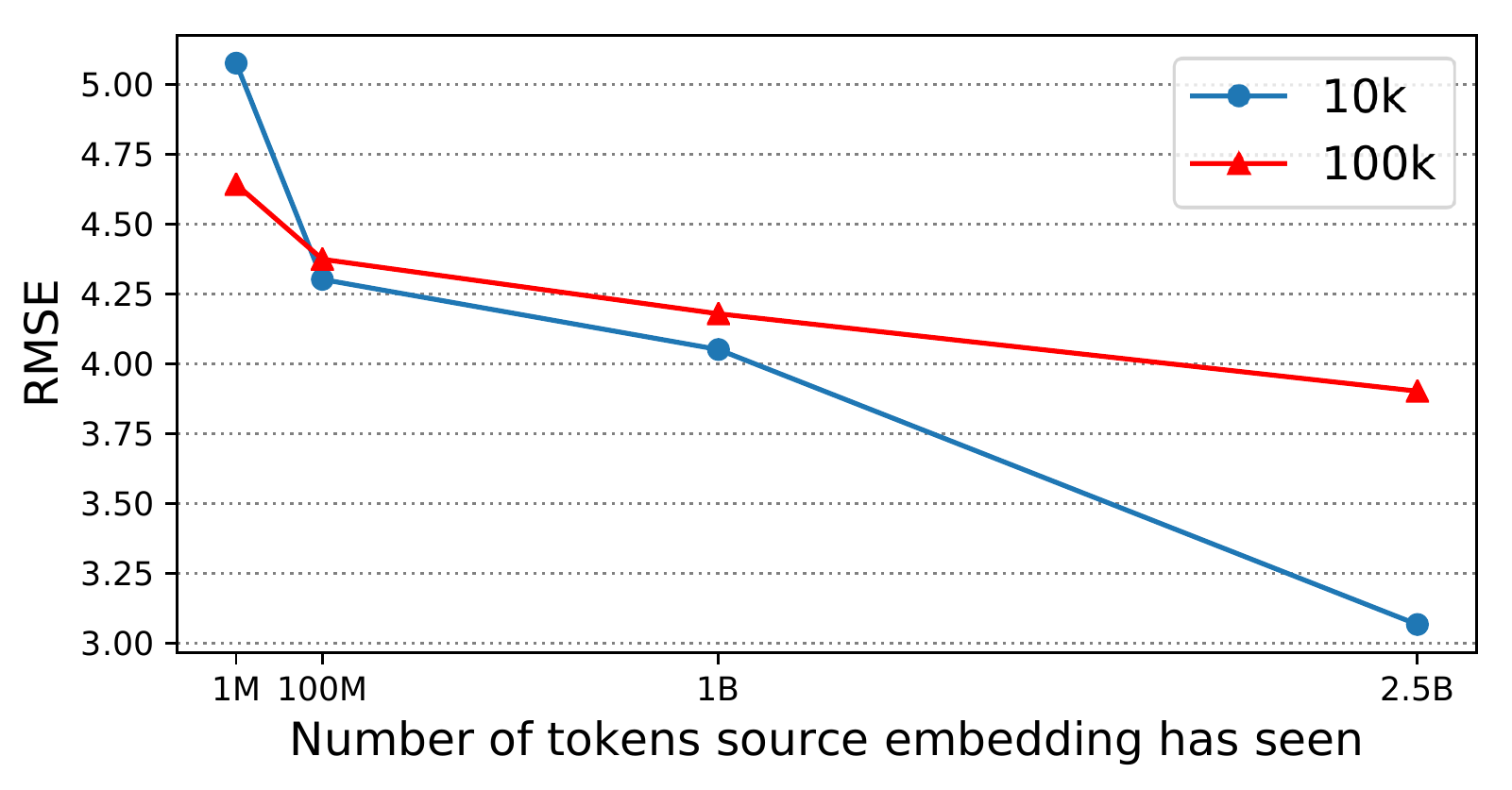}
          \caption{}
      \end{subfigure}
   \caption{RMSE Error after noise and \AOR{} alignment: 
		(a): Adding structured and unstructured noise before embedding. 
		(b): Incrementally added data.}
   \label{fig:RMSE-noise}
\end{figure}

\begin{table}[]

		\caption{\label{tbl:RMSE-Datasets+scale}RMSE after alignment for embeddings.  
		Top: Created from different datasets.  
		Bottom: Created by different embeddings; uses \AOs {} mapping rows onto columns and changing scale.  }
	\centering
	\begin{tabular}{r||cccc}
		\hline
		$\rightarrow$ & \s{G(W)}  & \s{\hspace{-2mm}G(WG)}  & \s{\hspace{-2mm}G(CC42)} & \s{\hspace{-2mm}G(CC840)\hspace{-2mm}}\\ \hline
		\s{G(W)}  & -   & 4.56 & 5.167 & 6.148\\
		
		\s{G(WG)} & 4.56  & - & 5.986  & 6.876\\ 
		\s{\hspace{-1mm}G(CC42)} & 5.167  & 5.986 & -  & 5.507\\ 
		\s{\hspace{-1mm}G(CC840)} & 6.148 & 6.876 & 5.507 & - \\ \hline	
	\end{tabular}		
	\hspace{.25in}
	\centering
	\begin{tabular}{r||ccc}
		\hline
		$\rightarrow$ & \RAW  & \GloVe & \hspace{-1mm}\WordToVec \\ \hline
		\RAW  & -   & 4.12 & 14.73 \\
		
		\GloVe & 0.045 & - & 12.93\\
		\WordToVec &0.043  & 3.68 & -  \\ \hline
		\hspace{-2mm} Scale to \GloVe  & 25 & 1 & 0.25 \\
		\hspace{-2mm} Scale from \GloVe  & 0.011  & 1 & 3 \\ \hline
	\end{tabular}

\end{table}

\begin{table}[p]
		\caption{\label{tbl:Sim-Analogy}Spearman coefficient scores for synonym and analogy tests between the aligned \WordToVec to \GloVe embeddings and between \GloVe embeddings of Wikipedia and CC42 dataset; r,s, and t stand for the functions of optimal rotation, scaling and translation, respectively, and w() is the weighted version of that function. When computing these scores across two embeddings, the best values are printed in bold. }
	\centering
	
	\resizebox{\columnwidth}{!}{\begin{tabular}{r||cccccccccc}
		\hline
		\multirow{2}{*}{Test Sets} & \multirow{2}{*}{\GloVe} & \multirow{2}{*}{\WordToVec } & \multicolumn{8}{c}{\WordToVec to \GloVe} 
		\\
		\cline{4-11}
		& & & untransformed & r & r + s & r + t & r + s + t & normalized & w(r) & w(r+s+t)
		\\
		\hline
		\RG & 0.614 & 0.696 & 0.041 & 0.584 & 0.584 & 0.570 & 0.594 & 0.553 & 0.592 & \textbf{0.597}\\
		\WSim & 0.623 & 0.659 & 0.064 & 0.624 & 0.624 & 0.611 & 0.652 & 0.604 & 0.657 & \textbf{0.664}\\
		\MC &0.669 & 0.818 & 0.013 & 0.868 & 0.868 & 0.843 &0.873 & 0.743 & 0.878 & \textbf{0.882}\\
		\Simlex & 0.296 & 0.342 & 0.012 & 0.278 &0.278 & 0.269 &0.314 & 0.261 & 0.314 & \textbf{0.316}\\ \hline
		\SYN & 0.587 & 0.582 & 0.000 & 0.501 &0.525 & 0.517 & 0.528 &0.493 &0.535 & \textbf{0.539}\\
		\SEM & 0.691 & 0.722 & 0.0001 & 0.624 & 0.656 & 0.633 & 0.697 &0.604 & 0.702 & \textbf{0.712}\\
		\hline
	\end{tabular}}

	\vspace{0.5mm}
	
	\resizebox{\columnwidth}{!}{\begin{tabular}{r||cccccccccc}
		\hline
		\multirow{2}{*}{Test Sets} & \multirow{2}{*}{\s{G(W)}} & \multirow{2}{*}{\s{G(CC42)}} & \multicolumn{8}{c}{\s{G(W)} to \s{G(CC42)}}  
		\\
		\cline{4-11}
		& & & untransformed & r &  r + s & r + t & r + s + t  & normalized & w(r) & w(r+s+t)
		\\
		\hline
		\RG & 0.614 & 0.817 & 0.363 & 0.818 & 0.818 & 0.811 &0.821 & 0.815 &0.818 & \textbf{0.825}\\
		\WSim & 0.623 & 0.63 & 0.017 & 0.618 & 0.618 & 0.615 &0.618& 0.601 &0.616 & \textbf{0.637}\\
		\MC & 0.669 & 0.786 & 0.259 & 0.766 & 0.766 & 0.732 &0.768 & 0.705 &0.771 & \textbf{0.774}\\
		\Simlex & 0.296 & 0.372 & 0.035 & 0.343 & 0.343 & 0.339 & \textbf{0.346} & 0.296 & \textbf{0.346} & \textbf{0.346}\\ \hline
		\SYN & 0.587  & 0.625 &0.00 & 0.566  &  \textbf{0.576} & 0.572 & \textbf{0.576} & 0.502 &\textbf{0.576} & \textbf{0.576}\\
		\SEM & 0.691  &  0.741 &0.00& 0.676  &  0.684 & 0.676 & 0.688 & 0.565 & 0.690 & \textbf{0.695}\\
		\hline
	\end{tabular}}

\end{table}

\begin{table}[]
		\caption{Synonym and analogy scores from rotations (applied to all words) learned on 100K, top 10K words, and top 10K words normalized.}
	\label{tbl : normalized}
	\centering

	\begin{tabular}{c||ccc}
		
		\hline
		&  \textsc{AO+R} & \textsc{AO+R} & \AON \\ 
		&  100K & 10K & 10K \\ \hline
		RG & 0.584 & 0.576 & 0.588 \\
		WSIM & 0.624 & 0.612 & 0.643\\
		MC & 0.868 & 0.817 & 0.851\\
		SIMLEX & 0.278 & 0.292 & 0.308\\ \hline
		SYN & 0.501 & 0.505 & 0.511\\
		SEM & 0.624 & 0.616 & 0.616\\
		\hline
	\end{tabular}
\end{table}

\begin{table}[]
		\caption{\label{tbl:baseline}Modified similarity tests (on only top 10K words) after alignment by Affine Transformation (AffTrans), LRA, and \AOR {} of Wikipedia and CC42 \GloVe embeddings. }
	\centering
	\begin{tabular}{r||ccc|c}
		\hline
		& LRA & AffTrans & \textsc{AO+R} & \textsc{AO+R} \\
		Test Sets & 10K & 10K & 10K & 100K \\
		\hline
		\RG & 0.701 & 0.301 & 0.728 & 0.818\\
		\WSim & 0.616 & 0.269 & 0.612 &0.618\\
		\MC & 0.719 & 0.412 & 0.722 & 0.766\\
		\Simlex & 0.327 & 0.126 & 0.340 & 0.343\\
		\hline 
	\end{tabular}

\end{table}

\begin{table}[]
		\caption{RMSE variation with word frequency in (a) \GloVe Wiki to \GloVe Common Crawl and (b) \WordToVec to \GloVe evaluated for Wiki dataset. All words were used for tests in lower case as listed in the table.}
	\centering
	\label{tbl:word freq}

	\resizebox{\columnwidth}{!}{\begin{tabular}{rr||ccc}
		\hline
		Word & Frequency in Wiki  & Norm in (\GloVe) Wiki & Wiki to CC (42B)  & \WordToVec to \GloVe
		\\ \hline 
		talk & 187513532 & 9.681& 8.608 & 12.462 \\
		november &2340726 & 7.847 & 5.26 &  8.614 \\
		man & 1035008 & 8.648 & 4.25 & 5.161\\
		statistical & 83531 & 5.891 & 4.63& 5.097\\
		bubbles & 11200 & 5.455 & 4.66 & 3.768 \\
		skateboard & 3804 & 5.670 & 5.714& 3.891 \\
		emoji & 1761 & 6.090 & 6.781 & 2.402\\
		haymaker & 705 & 4.108 & 5.951  & 1.573\\
		\hline
	\end{tabular}}

\end{table}

\begin{table}[]

	\caption{\label{tbl:boosting weighted}Similarity and analogy tests before and after alignment and combining embeddings derived from different techniques and datasets by AO + CENTERED + WEIGHTED.  Best scores in bold.}
	
		\centering
	\resizebox{\columnwidth}{!}{\begin{tabular}{r|ccccccc}
		\hline
		TestSets & \s{G(W)}  & \hspace{-1mm}\s{W(W)} & \hspace{-2mm}[\s{G(W)}$\odot$\s{W(W)}]\hspace{-1mm} & \s{W(GN)} & \hspace{-2mm}[\s{G(W)}$\odot$\s{W(GN)}]\hspace{-1mm} & \hspace{-1mm}\s{G(CC840)} & \hspace{-2mm}[\s{G(CC840)}$\odot$\s{W(GN)}]\hspace{-2mm} \\
		\hline
		\RG & 0.614 & 0.696 & 0.716 & 0.760 &   {\bf 0.836} & 0.768 & 0.810 \\
		\WSim & 0.623 & 0.659 & 0.695 & 0.678 & 0.708 & 0.722 &  {\bf 0.740} \\
		\MC & 0.669 & 0.818 & {\bf 0.869} & 0.800 & 0.811 & 0.798 & 0.847 \\
		\Simlex & 0.296 & 0.342 & 0.368 & 0.367 & 0.394 & 0.408 &  {\bf 0.446} \\
		\hline 
		\SYN & 0.587& 0.582  & 0.592 & 0.595 & 0.607 &  {\bf 0.618} & 0.609 \\
		\SEM & 0.691 & 0.722 &  {\bf 0.759} &  0.713 & 0.733 & 0.729 & 0.733 \\
		\hline
	\end{tabular}}
\end{table}

\begin{table}[]

		\caption{\label{tbl:translation-spanish}
		The 5 closest neighbors of a word before and after alignment by AO + ROTATION (between English - Spanish). Target word (translation) in bold.}
	\centering 

	\resizebox{\columnwidth}{!}{\begin{tabular}{r||cc}
		\hline
		Word  & Neighbors before alignment & Neighbors after alignment  \\
		\hline
		woman & her, young, man, girl, mother & her, girl, \textbf{mujer}, mother, man\\
		week & month, day, year, monday, time &  days, \textbf{semana}, year, day, month\\
		casa & apartamento, casas, palacio, residencia, habitaci & casas, \textbf{home}, homes, habitaci, apartamento\\
		caballo & caballos, caballer, jinete, jinetes, equitaci & \textbf{horse}, horses, caballos, jinete\\
		sol & sombra,luna,solar,amanecer,ciello & \textbf{sun}, moon, luna, solar, sombra \\
		\hline
	\end{tabular}}

\end{table}

\begin{table}[]

		\caption{The 5 closest neighbors of a word before and after alignment by AO (between English - French). Target word (translation) in bold.}
	\centering

	\begin{tabular}{r||cc}
		\hline
		Word  & Neighbors before alignment & Neighbors after alignment  \\
		\hline
		woman & her, young, man, girl, mother & her, young, man, \textbf{femme}, la\\
		week & month, day, year, monday, time & month, day, year, \textbf{semaine}, start\\
		heureux & amoureux, plaisir, rire, gens, vivre & \textbf{happy}, plaisir, loving, amoureux, rire\\
		cheval & chein, petit, bateau, pied, jeu & \textbf{horse}, dog, chien, red, petit\\
		daughter & father, mother, son, her, husband & mother, \textbf{fille}, husband,  mere, her \\
		\hline
	\end{tabular}

	\label{tbl : translation}
\end{table}

\begin{table}[]
        \caption{Translating Spanish to French by aligning directly as compared to aligning both to English.}
    \centering
    \begin{tabular}{c||ccc}
        \hline
         Top n accuracy & Unaligned & Pairwise Aligned & Indirectly Aligned  \\
         \hline 
         n = 1 & 0.0 & 0.312 & 0.277\\
         n = 5 & 0.0 & 0.635 & 0.601\\
         n = 10 & 0.0 & 0.723 & 0.707\\
        \hline
    \end{tabular}

\label{tbl : multi}
\end{table}

%% file: chap3.tex

\chapter{Attenuating Bias in Word Vectors}
\label{chap: bias paper 1}
   Word vector representations are well developed tools and are known to retain significant semantic and syntactic structure of languages. But they are also prone to carrying and amplifying societal biases~\cite{debias} which can perpetrate discrimination in various applications. We explore in this chapter, how and where biases get captured in word representations. We also explore new, simple ways to detect the most stereotypically gendered words in an embedding and remove the bias from them.  Further, we verify how names are masked carriers of gender bias and then use that as a tool to attenuate bias in embeddings. Finally, we extend this property of names to show how names can be used to detect other types of bias in the embeddings such as bias based on race, ethnicity, and age.

\section{Bias in Word Vectors}
\label{sec : intro}

As we have noted, word embeddings are an increasingly popular application that drives NLP and many tasks in Machine Learning as well.  
However, it has been observed that word embeddings are prone to express the societal biases inherent in the data it is extracted from \cite{debias,debias2,Caliskan183}. Further, Zhao \etal (2017) \cite{ZhaoWYOC17} and Hendricks \etal (2018) \cite{Burns2018WomenAS} show that machine learning algorithms and their output show more bias than the data they are generated from.

Word vector embeddings as used in machine learning towards applications that significantly affect people's lives, such as to assess credit~\cite{credit}, predict crime~\cite{crime}, and other emerging domains such as judging loan applications and resumes for jobs or college applications.  So it is paramount that efforts are made to identify and if possible to remove bias inherent in them.  Or at least, we should attempt to minimize the propagation of bias within them.  For instance, in using existing word embeddings, Bolukbasi \etal (2016) \cite{debias} demonstrated that women and men are associated with different professions, with men associated with leadership roles and professions like doctor, programmer and women closer to professions like receptionist or nurse.  Caliskan \etal (2017) ~\cite{Caliskan183} similarly noted how word embeddings show that women are more closely associated with arts than math while it is the opposite for men. They also showed how positive and negative connotations are associated with European-American versus African-American names.

This chapter simplifies, quantifies, and fine-tunes these approaches: we show that a simple linear projection of all words based on vectors captured by common names is an effective and general way to significantly reduce bias in word embeddings.  
More specifically:

\begin{itemize}
\item[1a.]
We demonstrate that simple linear projection of all word vectors along a bias direction is more effective than the Hard Debiasing of Bolukbasi \etal (2016) \cite{debias} which is more complex and also partially relies on crowdsourcing. 

\item[1b.]
We show that these results can be slightly improved by dampening the projection of words that are far from the projection distance.  

\item[2.]
We examine the bias inherent in the standard word pairs used for debiasing based on gender by randomly flipping or swapping these words in the raw text before creating the embeddings.  We show that this alone does not eliminate bias in word embeddings, corroborating that simple language modification is not as effective as repairing the word embeddings themselves.  

\item[3a.]
We show that common names with gender association (e.g., \word{john}, \word{amy}) often provides a more effective gender subspace to debias along than using gendered words (e.g., \word{he}, \word{she}).  

\item[3b.] 
We demonstrate that names carry other inherent, and sometimes unfavorable, biases associated with race, nationality, and age, which also corresponds with bias subspaces in word embeddings.  And that it is effective to use common names to establish these directions of bias and remove this bias from word embeddings.  
\end{itemize}



\section{Data}
\label{sec:data}
We set as default the text corpus of a Wikipedia dump (\url{dumps.wikimedia.org/enwiki/latest/enwiki-latest-pages-articles.xml.bz2}) with 4.57 billion tokens and we extract a $\GloVe$ embedding from it in $D=300$ dimensions per word. We restrict the word vocabulary to the most frequent $100{,}000$ words.  We also modify the text corpus and extract embeddings from it as described later. 

So, for each word in the Vocabulary $W$, we represent the word by the vector $w_i \in \b{R}^D$ in the embedding.   The bias (e.g., gender) subspace is denoted by a set of vector $B$.  It is typically considered in this work to be a single unit vector, $v_B$ (explained in detail later).  As we will revisit, a single vector is typically sufficient and will simplify descriptions.  However, these approaches can be generalized to a set of vectors defining a multi-dimensional subspace.

\section{How to Attenuate Bias}
\label{sec : methods to attenuate bias}

Given a word embedding, debiasing typically takes as input a set $\c{E} = \{E_1, E_2, \ldots, E_m\}$ of equality sets.  An equality set $E_j$ for instance can be a single pair (e.g., \{\word{man}, \word{woman}\}), but could be more words (e.g., \{\word{latina}, \word{latino}, \word{latinx}\}) that if the bias connotation (e.g., gender) is removed, then it would objectively make sense for all of them to be equal.  Our data sets will only use word pairs (as a default the ones in Table \ref{tbl:word-pairs}), and we will describe them as such hereafter for simpler descriptions.  In particular, we will represent each $E_j$ as a set of two vectors $e_i^+, e_i^- \in \b{R}^D$.  

Given such a set $\c{E}$ of equality sets, the bias vector $v_B$ can be formed as follows~\cite{debias}.  For each $E_j = \{e_j^+, e_j^-\}$ create a vector $\vec{e}_i = e_i^+ - e_i^-$ between the pairs.  Stack these to form a matrix $Q = [\vec{e}_1\; \vec{e}_2\; \ldots\; \vec{e}_m]$, and let $v_B$ be the top singular vector of $Q$.  We revisit how to create such a bias direction in Section \ref{sec:gender-subspace}.  

Now given a word vector $w \in W$, we can project it to its component along this bias direction $v_B$ as 
\begin{equation}
\pi_B(w) = \langle w, v_B \rangle v_B.
\end{equation}

\subsection{Existing Method: Hard Debiasing}

The most notable advance towards debiasing embeddings along the gender direction has been by Bolukbasi \etal (2016) \cite{debias} in their algorithm called Hard Debiasing ($HD$). 
%
It takes a set of words desired to be neutralized, $\{w_1, w_2, \ldots, w_n\} = W_N \subset W$, a unit bias subspace vector $v_B$, and a set of equality sets $E_1, E_2, \ldots, E_m$. 

First, words $\{w_1, w_2, \ldots, w_n\} \in W_N$ are projected orthogonal to the bias direction and normalized  
\begin{equation}
w_i' = \frac{w_i - w_B}{||w_i - w_B||}.  
\end{equation}


Second, it corrects the locations of the vectors in the equality sets.  Let $\mu_j = \frac{1}{|E|} \sum_{e \in E_j} e$ be the mean of an equality set, and $\mu = \frac{1}{m} \sum_{j=1}^m \mu_j$ be the mean of of equality set means.  Let $\nu_j = \mu - \mu_j$ be the offset of a particular equality set from the mean.  
Now each $e \in E_j$ in each equality set $E_j$ is first centered using their average and then neutralized as
\begin{equation}
e' = \nu_j + \sqrt{1-\|\nu_j\|^2}\frac{\pi_B(e) - v_B}{\|\pi_B(e)- v_B\|}.  
\end{equation}
Intuitively $\nu_j$ quantifies the amount words in each equality set $E_j$ differ from each other in directions apart from the gender direction. This is used to center the words in each of these sets.

This renders word pairs such as \word{man} and \word{woman} as equidistant from the neutral words $w_i'$ with each word of the pair being centralized and moved to a position opposite the other in the space. This can filter out properties either word gained by being used in some other context, like \word{mankind} or \word{humans} for the word \word{man}.

The word set $W_N = \{w_1,w _2, \ldots, w_n\} \subset W$ which is debiased is obtained in two steps.  
First, it seeds some words as definitionally gendered via crowdsourcing and using dictionary definitions; the complement -- ones not selected in this step -- are set as neutral. 
Next, using this seeding an SVM is trained and used to predict among all $W$ the set of other biased $W_B$ or neutral words $W_N$. 
This set $W_N$ is taken as desired to be neutral and is debiased.  
Thus not all words $W$ in the vocabulary are debiased in this procedure, only a select set, chosen via crowd-sourcing and definitions, and its extrapolation.  
Also, the word vectors in the equality sets are also handled separately.  
This makes this approach not a fully automatic way to debias the vector embedding.

\subsection{Alternate and Simple Methods}
\label{sec : projection}

We present some simple alternatives to HD which are simple and fully automatic.  
These all assume a bias direction $v_B$.  

\noindent \textbullet\ \myParagraph{Subtraction}
As a simple baseline, for \emph{all} word vectors $w$  subtract the gender direction $v_B$ from $w$:
\begin{equation}
w' = w - v_B. 
\end{equation}

\noindent \textbullet\ \myParagraph{Linear Projection}
A better baseline is to project \emph{all} words $w \in W$ orthogonally to the bias vector $v_B$.  
\begin{equation}
w' = w - \pi_B(w) = w - \langle w, v_B \rangle v_B.  
\end{equation}
This enforces that the updated set $W' = \{w' \mid w \in W\}$ has no component along $v_B$, and hence the resulting span is only $D-1$ dimensions.  Reducing the total dimension from say $300$ to $299$ should have minimal effects of expressiveness or generalizability of the word vector embeddings. 

Bolukbasi \etal \cite{debias} apply this same step to a dictionary definition based extrapolation and crowd-source-chosen set of word pairs $W_N \subset W$.  We quantify in Section \ref{sec : tests} that this single universal projection step debiases better than HD.

For example, consider the bias as gender, and the equality set with words \word{man} and \word{woman}.  Linear projection will subtract from their word embeddings the proportion that was along the gender direction $v_B$ learned from a larger set of equality pairs.   It will make them close-by but not exactly equal.   
The word \word{man} is used in many extra senses than the word \word{woman}; it is used to refer to humankind, to a person in general, and in expressions like ``oh man.''  
In contrast, a simpler word pair with fewer word senses, like he-she and him-her, we can expect them to be almost at identical positions in the vector space after debiasing, implying their synonymity. 

Thus, this approach uniformly reduces the component of the word along the bias direction without compromising on the differences that words (and word pairs) have.  

%

\subsection{Partial Projection}
\label{sec : variants}
A potential issue with the simple approaches is that they can significantly change some embedded words which are definitionally biased (e.g., the neutral words $W_B$ described by Bolukbasi \etal~\cite{debias}).  (We note that this may not *actually* be a problem (see Section \ref{sec : tests}); the change may only be associated with the bias, so removing it would then not change the meaning of those words in any way except the ones we want to avoid.) However, these intuitively should be words which have a correlation with the bias vector, but also are far in the orthogonal direction.  In this section, we explore how to automatically attenuate the effect of the projection on these words.   

This stems from the observation that given a bias direction, the words which are most extreme in this direction (have the largest dot product) sometimes have a reasonably biased context, but some do not.  These ``false positives'' may be large normed vectors which also happen to have a component in the bias direction.  

We start with a bias direction $v_B$ and mean $\mu$ derived from equality pairs (defined the same way as in the context of HD).  
Now given a word vector $w$ we decompose it into key values along two components, illustrated in Figure \ref{fig:eta-beta}.  
First, we write its bias component as 
\begin{equation}
\beta(w) = \langle w, v_B \rangle - \langle \mu, v_B \rangle.
\end{equation}
This is the difference of $w$ from $\mu$ when both are projected onto the bias direction $v_B$.  

Second, we write a (residual) orthogonal component 
\begin{equation}
r(w) = w - \langle w, v_B \rangle v_B.  
\end{equation}
Let $\eta(w) = \|r(w)\|$ be its value. 
It is the orthogonal distance from the bias vector $v_B$; recall we chose $v_B$ to pass through the origin, so the choice of $\mu$ does not affect this distance.  

Now we will maintain the orthogonal component ($r(w)$, which is in a subspace spanned by $D-1$ out of $D$ dimensions) but adjust the bias component $\beta(w)$ to make it closer to $\mu$.  But the adjustment will depend on the magnitude $\eta(w)$.  
As a default we set 
\[
w' = \mu + r(w)
\]
so all word vectors retain their orthogonal component but have a fixed and constant bias term.  
This is functionally equivalent to the Linear Projection approach; the only difference is that instead of having a $0$ magnitude along $v_B$ (and the orthogonal part unchanged), it instead has a magnitude of constant $\mu$ along $v_B$ (and the orthogonal part still unchanged).  This adds a constant to every inner product and a constant offset to any linear projection or classifier.  If we are required to work with normalized vectors (we do not recommend this as the vector length captures veracity information~
about its embedding), we can simple set $w' = r(w)/\|r(w)\|$.

Given this set-up, we now propose three modifications.  In each set 
\begin{equation}
w' = \mu + r(w) + \beta \cdot f_i(\eta(w)) \cdot v_B
\end{equation}
were $f_i$ for $i = \{1,2,3\}$ is a function of only the orthogonal value $\eta(w)$.  For the default case $f(\eta) = 0$
\begin{align*}
f_1(\eta) &= \sigma^2/(\eta+1)^2
\\
f_2(\eta) &= \exp(-\eta^2/\sigma^2)
\\
f_3(\eta) &= \max(0,\sigma/2\eta) 
\end{align*}
Here $\sigma$ is a hyperparameter that controls the importance of $\eta$; in Section \ref{app:sigma} we show that we can just set $\sigma=1$.  x

In Figure \ref{fig : variations} we see the regions of the $(\eta,\beta)$-space that the functions $f$, $f_1$ and $f_2$ consider gendered. $f$ projects all points onto the $y = \mu$ line. But variants $f_1$, $f_2$, and $f_3$ are represented by curves that dampen the bias reduction to different degrees as $\eta$ increases.  Points P1 and P2 have the same dot products with the bias direction but different dot products along the other $D-1$ dimensions. We can observe the effects of each dampening function as $\eta$ increases from P1 to P2.

\subsection{Setting $\sigma=1$}
\label{app:sigma}
To complete the damping functions $f_1$, $f_2$, and $f_3$, we need a value $\sigma$.  If $\sigma$ is larger, then more word vectors have bias completely removed; if $\sigma$ is smaller than more words are unaffected by the projection.  
The goal is that words $S$ which are likely to carry a bias connotation to have little damping (small $f_i$ values) and words $T$ which are unlikely to carry a bias connotation to have more damping (large $f_i$ values -- they are not moved much).  

Given sets $S$ and $T$, we can define a gain function
\begin{equation}
\gamma_{i,\rho}(\sigma) =  \sum_{s \in S} \beta(s) (1-f_i(\eta(s))) -  \rho \sum_{t \in T} \beta(t) (1-f_i(\eta(t))),
\end{equation}
with a regularization term $\rho$. 
The gain $\gamma$ is large when most bias words in $S$ have very little damping (small $f_i$, large $1-f_i$), and the opposite is true for the neutral words in $T$.  We want the neutral words to have large $f_i$ and hence small $1-f_i$, so they do not change much.

To define the gain function, we need sets $S$ and $T$; we do so with the bias of interest as gender.  The biased set $S$ is chosen among a set of $1000$ popular names in $W$ which (based on babynamewizard.com and SSN databases~\cite{nameWizard,ssn}) are strongly associated with gender.  The neutral set $T$ is chosen as the most frequent $1000$ words from $W$, after filtering out obviously gendered words like names \word{man} and \word{he}.  We also omit occupation words like \word{doctor} and others which may carry unintentional gender bias (these are words we would like to automatically de-bias).  The neutral set may not be perfectly non gendered, but it provides a reasonable approximation of all non gendered words.

We find for an array of choices for $\rho$ (we tried $\rho =1$, $\rho=10$, and $\rho=100$), the value $\sigma=1$ approximately maximizes the gain function $\gamma_{i,\rho}(\sigma)$ for each $i \in \{1,2,3\}$.  So for hereafter we fix $\sigma=1$.  

Although these sets $S$ and $T$ play a role somewhat similar to the crowd-sourced sets $W_B$ and $W_N$ from HD that we hoped to avoid, the role here is much reduced.  This is just to verify that a choice of $\sigma=1$ is reasonable, and otherwise, they are not used.   

\subsection{Flipping the Raw Text}
Since the embeddings preserve inner products of the data from which it is drawn, we explore if we can make the data itself gender unbiased and then observe how that change shows up in the embedding.  Unbiasing a textual corpus completely can be very intricate and complicated since there are many (sometimes implicit) gender indicators in text. 
Nonetheless, we propose a simple way of neutralizing bias in textual data by using word pairs $E_1, E_2, \ldots E_m$; in particular, when we observe in the raw text a word of a word pair, we randomly flip it to the other word in the pair.  For instance for gendered word pairs (e.g., he-she) in a string ``\word{he} was a doctor'' we may flip to ``\word{she} was a doctor.''

We implement this procedure over the entire input raw text and try various probabilities of flipping each observed word, focusing on probabilities $0.5$, $0.75$, and $1.00$.  The first $0.5$-flip probability makes each element of a word pair equally likely.  The last $1.00$-flip probability reverses the roles of those word pairs, and $0.75$-flip probability does something in between.  
We perform this set of experiments on the default Wikipedia data set and switch between word pairs (say \word{man} $\rightarrow$ \word{woman}, \word{she} $\rightarrow$ \word{he}, etc.), from a large list of 75 word pairs (see Appendix \ref{app : flipping}).  

We observe how the proportion along the principal component changes with this flipping in Figure \ref{fig:flipping}.  
We see that flipping with $0.5$ somewhat dampens the difference between the different principal components.  
On the other hand, flipping with a probability of $1.0$ (and to a lesser extent $0.75$) exacerbates the gender components rather than dampening it. Now there are two components significantly larger than the others.  This indicates this flipping is only addressing part of the explicit bias, but missing some implicit bias, and these effects are now muddled.  

We list some gender biased analogies in the default embedding and how they change with each of the methods described in this section in Table \ref{tbl : analogies flipping}.

\section{The Bias Subspace}
\label{sec:gender-subspace}

%

We explore ways of detecting and defining the bias subspace $v_B$ and recovering the most gendered words in the embedding.  Recall as default, we use $v_B$ as the top singular vector of the matrix defined by stacking vectors $\vec{e}_i = e_i^+ - e_i^-$ of biased word pairs.  
We primarily focus on gendered bias, using words in Table \ref{tbl:word-pairs}, and show later how to effectively extend to other biases.

\subsection{Detecting the Gender Subspace}
\label{app : gender dir}
First, we take a set of gendered word pairs as listed in Table \ref{tbl : Gendered Words}.
From our default Wikipedia dataset, using the embedded vectors for these word pairs (i.e., woman-man, she-he, etc.), we create a basis for the subspace $F$, of dimension $10$.   We then try to understand the distribution of variance in this subspace.  To do so, we project the entire dataset onto this subspace $F$, and take the SVD.  The top chart in Figure \ref{fig:SVDcharts} shows the singular values of the entire data in this subspace $F$.  We observe that there is a dominant first singular vector/value which is almost twice the size of the second value.  After this drop, the decay is significantly more gradual.  This suggests using only the top singular vector of $F$ as the gender subspace, not $2$ or more of these vectors.   

To grasp how much of this variation is from the correlation along the gender direction, and how much is just random variation, we repeat this experiment, again in Figure \ref{fig:SVDcharts}, with different ways of creating the subspace $F$.  
First, in chart (b), we generate $10$ vectors, with one word chosen randomly, and one chosen from the gendered set (e.g., chair-woman).  
Second, in chart (c), we generate $10$ vectors between two random words from our set of the $100{,}000$ most frequent words; these are averaged over $100$ random iterations due to higher variance in the plots.  
Finally in chart (d), we generate $10$ random unit vectors in $\b{R}^{300}$.    
We observe that the pairs with one gendered vector in each pair still exhibit a significant drop in singular values, but not as drastic as with both pairs. The other two approaches have no significant drop since they do not in general contain a gendered word with interesting subspace. All remaining singular values and their decay appears similar to the non leading ones from the charts (a) and (b). This further indicates that there is roughly one important gender direction, and any related subspace is not significantly different than a random one in the word-vector embedding.

%
%

Now, for any word $w$ in vocabulary $W$ of the embedding, we can define $w_B$ as the part of $w$ along the gender direction. 


Based on the experiments shown in Figure \ref{fig:SVDcharts}, it is justified to take the gender direction as the (normalized) first right singular vector, $v_B$, or the full data set data projected onto the subspace $F$.  Then, the component of a word vector $w$ along $v_B$ is simply $\langle w, v_B\rangle v_B$. 

Calculating this component when the gender subspace is defined by two or more of the top right singular vectors of $V$ can be done similarly.

We should note here that the gender subspace defined here passes through the origin. Centering the data and using PCA to define the gender subspace lets the gender subspace not pass through the origin. We see a comparison in the two methods in Section \ref{sec : tests} as $HD$ uses PCA and we use SVD to define the gender direction.

\subsection{Most Gendered Words}
The dot product, $\langle v_B, w \rangle$ of the word vectors $w$ with the gender subspace  $v_B$ is a good indicator of how gendered a word is. 
The magnitude of the dot product tells us of the length along the gender subspace and the sign tells us whether it is more female or male.   
Some of the words denoted as most gendered are listed in Table \ref{tbl : Gendered Words}.

\subsection{Bias Direction Using Names}
\label{detecting bias with names}

When listing gendered words by $|\langle v_B, w \rangle|$, we observe that many gendered words are names. This indicates the potential to use names as an alternative (and potentially in a more general way) to bootstrap finding the gender direction. 

From the top 100K words, we extract the 10 most common male $\{m_1, m_2, \ldots, m_{10}\}$ and female $\{s_1,s_2,\ldots, s_{10}\}$ names which are not used in ambiguous ways (e.g., not the name \word{hope} which could also refer to the sentiment). 
We pair these 10 names from each category (male, female) randomly and compute the SVD as before.  
We observe in Figure \ref{fig : names} that the fractional singular values show a similar pattern as with the list of correctly gendered word pairs like man-woman, he-she, etc.  
But this way of pairing names is quite imprecise. These names are not ``opposites" of each other in the sense that word pairs are. So, we modify how we compute $v_B$ now so that we can better use names to detect the bias in the embedding. The following method gives us this advantage where we do not necessarily need word pairs or equality sets as in Bolukbasi \etal \cite{debias}.

Our gender direction is calculated as,
\begin{equation}
v_{B,\textsf{names}} = \frac{s - m}{\|s - m\|},
\end{equation}
where $s = \frac{1}{10}\sum_{i} s_i$ and 
$m = \frac{1}{10}\sum_{i} m_i$.  


Using the default Wikipedia dataset, we found that this is a good approximator of the gender subspace defined by the first right singular vector calculated using gendered words from Table \ref{tbl:word-pairs}; there dot product is $0.809$.
We find similar large dot product scores for other datasets too. 

Here too we collect all the most gendered words as per the gender direction $v_{B,\textsf{names}}$ determined by these names. Most gendered words returned are similar to using the default $v_B$, like occupational words, adjectives, and synonyms for each gender. We find names to express similar classification of words along male - female vectors with \word{homemaker} more female and \word{policeman} being more male. We illustrate this in more detail in Table \ref{tbl : occ names list}. 

Using that direction, we debias by linear projection. There is a similar shift in analogy results. We see a few examples in Table \ref{tbl : analogies}.

\section{Quantifying Bias}
\label{sec : tests}
In this section, we develop new measures to quantify how much bias has been removed from an embedding and evaluate the various techniques we have developed for doing so.

As one measure, we use the Word Embedding Association Test (WEAT) test developed by Caliskan \etal (2017) \cite{Caliskan183} as analogous to the IAT tests to evaluate the association of male and female gendered words with two categories of target words: career oriented words versus family oriented words.  We detail WEAT and list the exact words used (as in \cite{Caliskan183}) in the Appendix \ref{app: words weat}; smaller values are better.

Bolukbasi \etal \cite{debias} evaluated embedding bias use a crowdsourced judgment of whether an analogy produced by an embedding is biased or not.  Our goal was to avoid crowdsourcing, so we propose two more automatic tests to qualitatively and uniformly evaluate an embedding for the presence of gender bias.


\subsection{Embedding Coherence Test (ECT)}

A way to evaluate how the neutralization technique affects the embedding is to evaluate how the nearest neighbors change for (a) gendered pairs of words $\c{E}$ and (b) indirect-bias-affected words such as those associated with sports or occupational words (e.g., \word{football}, \word{captain}, \word{doctor}).  
We use the gendered word pairs in Table \ref{tbl:word-pairs} for $\c{E}$ and the professions list $P = \{p_1, p_2, \ldots, p_k\}$ as proposed and used by Bolukbasi \etal \url{https://github.com/tolga-b/debiaswe} (see also Appendix \ref{sec : occupation words}) to represent (b).

\begin{itemize}
\item[S1:] 
For all word pair $\{e_j^+, e_j^-\} = E_j \in \c{E}$ we compute two means $m = \frac{1}{|\c{E}|} \sum_{E_j \in \c{E}} e_j^+$ and $s = \frac{1}{|\c{E}|} \sum_{E_j \in \c{E}} e_j^-$.   We find the cosine similarity of both $m$ and $s$ to all words $p_i \in P$.  This creates two vectors $u_m, u_s \in \b{R}^k$.  

\item[S2:] 
We transform these similarity vectors to replace each coordinate by its rank order and compute the Spearman Coefficient (in $[-1,1]$, larger is better) between the rank order of the similarities to words in $P$. 
\end{itemize}

Thus, here, we care about the order in which the words in $P$ occur as neighbors to each word pair rather than the exact distance. The exact distance between each word pair would depend on the usage of each word and thus on all the different dimensions other than the gender subspace too. But the order being relatively the same, as determined using Spearman Coefficient would indicate the dampening of bias in the gender direction (i.e., if \word{doctor} by profession is the 2nd closest of all professions to both \word{man} and \word{woman}, then the embedding has a dampened bias for the word \word{doctor} in the gender direction). 
Neutralization should ideally bring the Spearman coefficient towards 1.

\subsection{Embedding Quality Test (EQT)  }
 
The demonstration by Bolukbasi \etal \cite{debias} about the skewed gender roles in embeddings using analogies is what we try to quantify in this test. We attempt to quantify the improvement in analogies with respect to bias in the embeddings. 

We use the same sets $\c{E}$ and $P$ as in the ECT test.  
However, for each profession $p_i \in P$ we create a list $S_i$ of their plurals and synonyms from WordNet on NLTK~\cite{nltk}.

\begin{itemize}
\item[S1:] 
For each word pair $\{e_j^+, e_j^-\} = E_j \in \c{E}$, and each occupation word $p_i \in P$, we test if the analogy  
$e_j^+ : e_j^- :: p_i$ 
returns a word from $S_i$. If yes, we set $Q(E_j,p_i) = 1$, and $Q(E_j,p_i) = 0$ otherwise.  

\item[S2:]
Return the average value across all combinations
$\frac{1}{|\c{E}|}\frac{1}{k} \sum_{E_j \in \c{E}} \sum_{p_i \in P} Q(E_j, p_i)$.  
\end{itemize}

The scores for EQT are typically much smaller than for ECT.  We explain two reasons for this.  
First, EQT does not check for if the analogy makes relative sense, biased or otherwise. So, ``\word{man} : \word{woman} :: \word{doctor} : \word{nurse}'' 
is as wrong as 
``\word{man} : \word{woman} :: \word{doctor} : \word{chair}." 
This pushes the score down.

Second, synonyms in each set $s_i$ as returned by WordNet \cite{WN} on the Natural Language Toolkit, NLTK \cite{nltk} do not always contain all possible variants of the word. For example, the words \word{psychiatrist} and \word{psychologist} can be seen as analogous for our purposes here but linguistically are removed enough that WordNet does not put them as synonyms together.  Hence, even after debiasing, if the analogy returns 
``\word{man} : \word{woman} :: \word{psychiatrist} : \word{psychologist}`` 
S1 returns 0. Further, since the data also has several misspelt words, \word{archeologist} is not recognized as a synonym or alternative for the word \word{archaeologist}. 
For this too, S1 returns a 0.

The first caveat can be side-stepped by restricting the pool of words we search over for the analogous word to be from list $P$. But it is debatable if an embedding should be penalized equally for returning both nurse or chair for the analogy 
``\word{man} : \word{woman} :: \word{doctor} : ?''. 

This measures the quality of analogies, with better quality having a score closer to $1$.

\subsection{Evaluating Embeddings}
We mainly run $4$ methods to evaluate our methods WEAT, EQT, and two variants of ECT:
 ECT (word pairs) uses $\c{E}$ defined by words in Table \ref{tbl:word-pairs} and 
 ECT (names) which uses vectors $m$ and $s$ derived by gendered names. 
 
We observe in Table \ref{tbl : Test 2 - analogies} that the ECT score increases for all methods in comparison to the non-debiased (the original) word embedding; the exception is flipping with $1.0$ probability score for ECT (word pairs) and all flipping variants for ECT (names).  Flipping does nothing to affect the names, so it is not surprising that it does not improve this score; further indicating that it is challenging to directly fix the bias in the raw text before creating embeddings.  
Moreover, HD has the lowest score (of $0.917$) whereas projection obtains scores of $0.996$ (with $v_B$) and $0.943$ (with $v_{B,\textsf{names}}$).  

EQT is a more challenging test, and the original embedding only achieves a score of $0.128$, and HD only obtains $0.145$ (that is $12-15\%$ of occupation words have their related word as one of its nearest neighbors).  On the other hand, projection increases this percentage to $28.3\%$ (using $v_B$) and $29.1\%$ (using $v_{B,\textsf{names}}$).  Even subtraction does nearly as well at between $23-27\%$.  Generally, the subtraction always performs slightly worse than projection.  

For the WEAT test, the original data has a score of $1.623$, and this is decreased the most by all forms of flipping, down to about $1.1$.  HD and projection do about the same with HD obtaining a score of $1.221$ and projection obtaining $1.219$ (with $v_{B,\textsf{names}}$) and $1.234$ (with $v_B$); values closer to 0 are better (See Section \ref{sec: weat} in Chapter \ref{chap: background} ).    

In the bottom of Table \ref{tbl : Test 2 - analogies} we also run these approaches on standard similarity and analogy tests for evaluating the quality of embeddings.  We use cosine similarity~\cite{Lev2} on 
WordSimilarity-353 (WSim, 353 word pairs) \cite{wsim} and 
SimLex-999 (Simlex, 999 word pairs) \cite{simlex}, 
each of which evaluates a Spearman coefficient (larger is better).  We also use   
the Google Analogy Dataset using the function 3COSADD \cite{Lev1} which takes in three words that form a part of the analogy and returns the 4th word which fits the analogy the best. 

%

We observe (as expected) that all that debiasing approaches reduce these scores.  The largest decrease in scores (between $1\%$ and $10\%$) is almost always from HD.  Flipping at $0.5$ rate is comparable to HD.  And simple linear projection decreases the least (usually only about $1\%$, except on analogies where it is $7\%$ (with $v_B$) or $5\%$ (with $v_{B,\textsf{names}}$).

In Table \ref{tbl:damping} we also evaluate the damping mechanisms defined by $f_1$, $f_2$, and $f_3$, using $v_B$.  These are very comparable to simple linear projection (represented by $f$).  The scores for ECT, EQT, and WEAT are all about the same as simple linear projection, usually slightly worse.  

While ECT, EQT and WEAT scores are in a similar range for all of $f$, $f_1$, $f_2$, and $f_3$; the dampened approaches $f_1$, $f_2$, and $f_3$ performs better on the Google Analogy test.  This test set is devoid of bias and is made up of syntactic and semantic analogies. So, a score closer to that of the original, biased embedding, tells us that more structure has been retained by $f_1$, $f_2$, and $f_3$.  Overall, any of these approaches could be used if a user wants to debias while retaining as much structure as possible, but otherwise linear projection (or $f$) is roughly as good as these dampened approaches.  

\section{Detecting Other Biases Using Names}
We saw so far how projection combined with finding the gender direction using names works well and works as well as projection combined with finding the gender direction using word pairs. 

We explore here a way of extending this approach to detect other kinds of bias where we cannot necessarily find good word pairs to indicate a direction, like Table \ref{tbl:word-pairs} for gender, but where names are known to belong to certain protected demographic groups. For example, there is a divide between names that different racial groups tend to use more. Caliskan \etal \cite{Caliskan183} uses a list of names that are more African-American (AA) versus names that are more European-American (EA) for their analysis of bias. There are similar lists of names that are distinctly and commonly used by different ethnic, racial (e.g.,  Asian, African-American) and even religious (for e.g., Islamic) groups.  


We first try this with two common demographic group divides: Hispanic / European-American and African-American / European-American. 

\begin{itemize}

\item \myParagraph{Hispanic and European-American Names} 
Even though we begin with most commonly used Hispanic (H) names (Appendix \ref{app: names race}), this is tricky as not all names occur as much as European American names and are thus not as well embedded. We use the frequencies from the dataset to guide us in selecting commonly used names that are also most frequent in the Wikipedia dataset.  Using the same method as Section \ref{detecting bias with names}, we determine the direction, $v_{B,\textsf{names}}$, which encodes this racial difference and find the words most commonly aligned with it.  Other Hispanic and European-American names are the closest words. But other words like, \word{latino} or \word{hispanic} also appear to be close, which affirms that we are capturing the right subspace.

\item \myParagraph{African-American and European-American Names} 
We see a similar trend when we use African-American names and European-American names (Figure \ref{fig : hispanic}). We use the African-American names used by Caliskan \etal (2017) \cite{Caliskan183}. We determine the bias direction by using the method in Section \ref{detecting bias with names}. 

We plot in Figure \ref{fig : hispanic} a few occupation words along the axes defined by H-EA and AA-EA bias directions, and compare them with those along the male-female axis.  The embedding is different among the groups, and likely still generally more subordinate-biased towards Hispanic and African-American names as it was for female.  Although \word{footballer} is more Hispanic than European-American, while \word{maid} is more neutral in the racial bias setting than the gender setting.  
We see this pattern repeated across embeddings and datasets. 


When we switch the type of bias, we also end up finding different patterns in the embeddings. In the case of both of these racial directions, there is a split in not just occupation words but other words that are detected as highly associated with the bias subspace. It shows up foremost among the closest words of the subspace of the bias. Here, we find words like \word{drugs} and \word{illegal} close to the H-EA direction while, close to the AA-EA direction, we retrieve several slang words used to refer to African-Americans.  
These word associations with each racial group can be detected by the WEAT tests (lower means less bias) using positive and negative words as demonstrated by Caliskan \etal (2017) \cite{Caliskan183}.  
We evaluate using the WEAT test before and after linear projection debiasing in Table \ref{tbl : weat other bias}.  For each of these tests, we use half of the names in each category for finding the bias direction and the other half for WEAT testing. This selection is done arbitrarily and the scores are averaged over 3 such selections.

More qualitatively, as a result of the dampening of bias, we see that biased words like other names belonging to these specific demographic groups, slang words, colloquial terms like \word{latinos} are removed from the closest 10\% words. This is beneficial since the distinguishability of demographic characteristics based on names is what shows up in these different ways like in occupational or financial bias. 


\item \myParagraph{Age Associated Names} 
We observed that names can be masked carriers of age too. Using the database for names through time~\cite{ssn} and extracting the most common names from the early 1900s as compared to the late 1900s and early 2000s, we find a correlation between these names (see Appendix \ref{app: age} for names) and age related words. In Figure \ref{fig : age}, we see
a clear correlation between age and names. Bias in this case does not show up in professions as clearly as in gender but in terms of association with positive and negative words \cite{Caliskan183}. We again evaluate using a WEAT test in Table \ref{tbl : weat other bias}, the bias before and after debiasing the embedding. 
\end{itemize}

\section{Discussion}

Different types of bias exist in textual data. Some are easier to detect and evaluate. Some are harder to find suitable and frequent indicators for and thus, to dampen. Gendered word pairs and gendered names are frequent enough in textual data to allow us to successfully measure it in different ways and project the word embeddings away from the subspace occupied by gender. Other types of bias don't always have a list of word pairs to fall back on to do the same. But using names, as we see here, we can measure and detect the different biases anyway and then project the embedding away from. In this work, we also see how a weighted variant of projection removes bias while retaining best, the inherent structure of the word embedding. 


\newpage

\begin{table}[]
    \caption{Gendered word pairs.}
    \centering
    \begin{tabular}{c}
\hline
    \{\word{man},\word{woman}\}, \{\word{son},\word{daughter}\}, \{\word{he},\word{she}\}, \{\word{his},\word{her}\},\\
    \{\word{male},\word{female}\},
   \{\word{boy},\word{girl}\}, \{\word{himself},\word{herself}\}, \\ \{\word{guy},\word{gal}\},
   \{\word{father},\word{mother}\}, \{\word{john},\word{mary}\} \\ \hline
    \end{tabular}

    \label{tbl:word-pairs}
\end{table}

\begin{figure}
\begin{center}
\vspace{1cm}
\includegraphics[width=8cm]{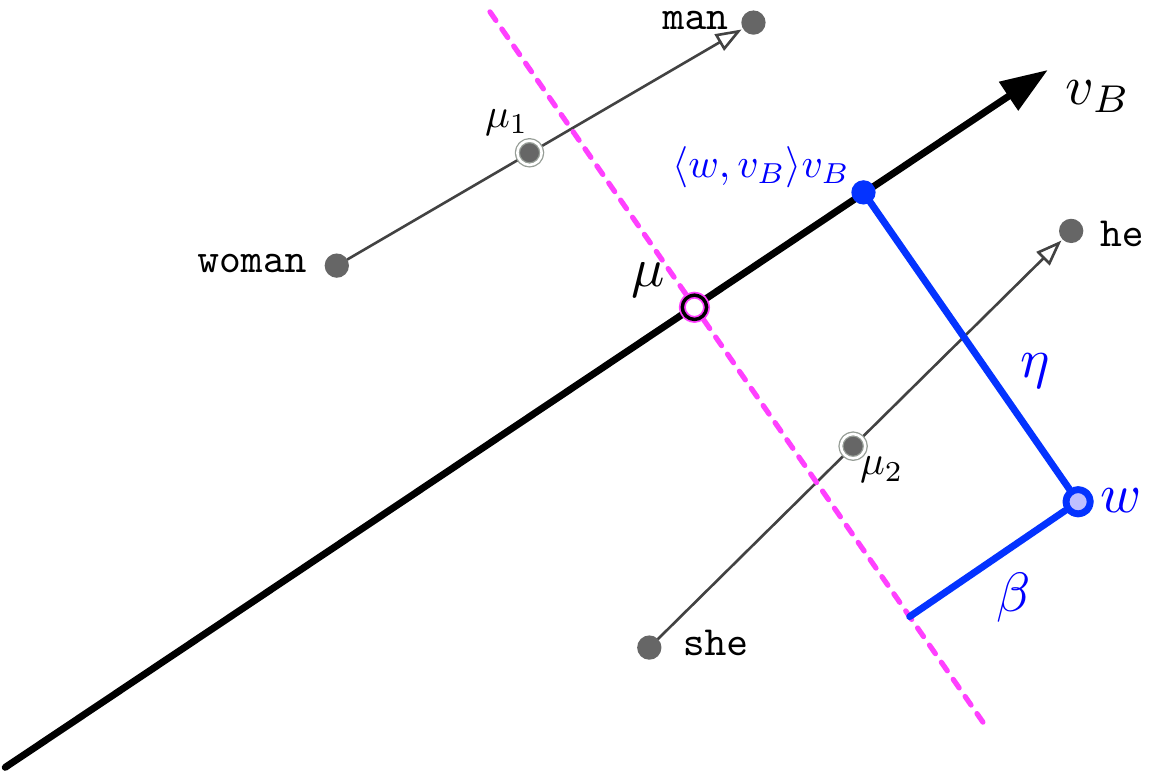}
\end{center}

\vspace{-3mm}
\caption{\label{fig:eta-beta}Illustration of $\eta$ and $\beta$ for word vector $w$.}
\end{figure}

\begin{figure}
\centering
\vspace{1cm}
	\includegraphics[width = 0.55 \textwidth]{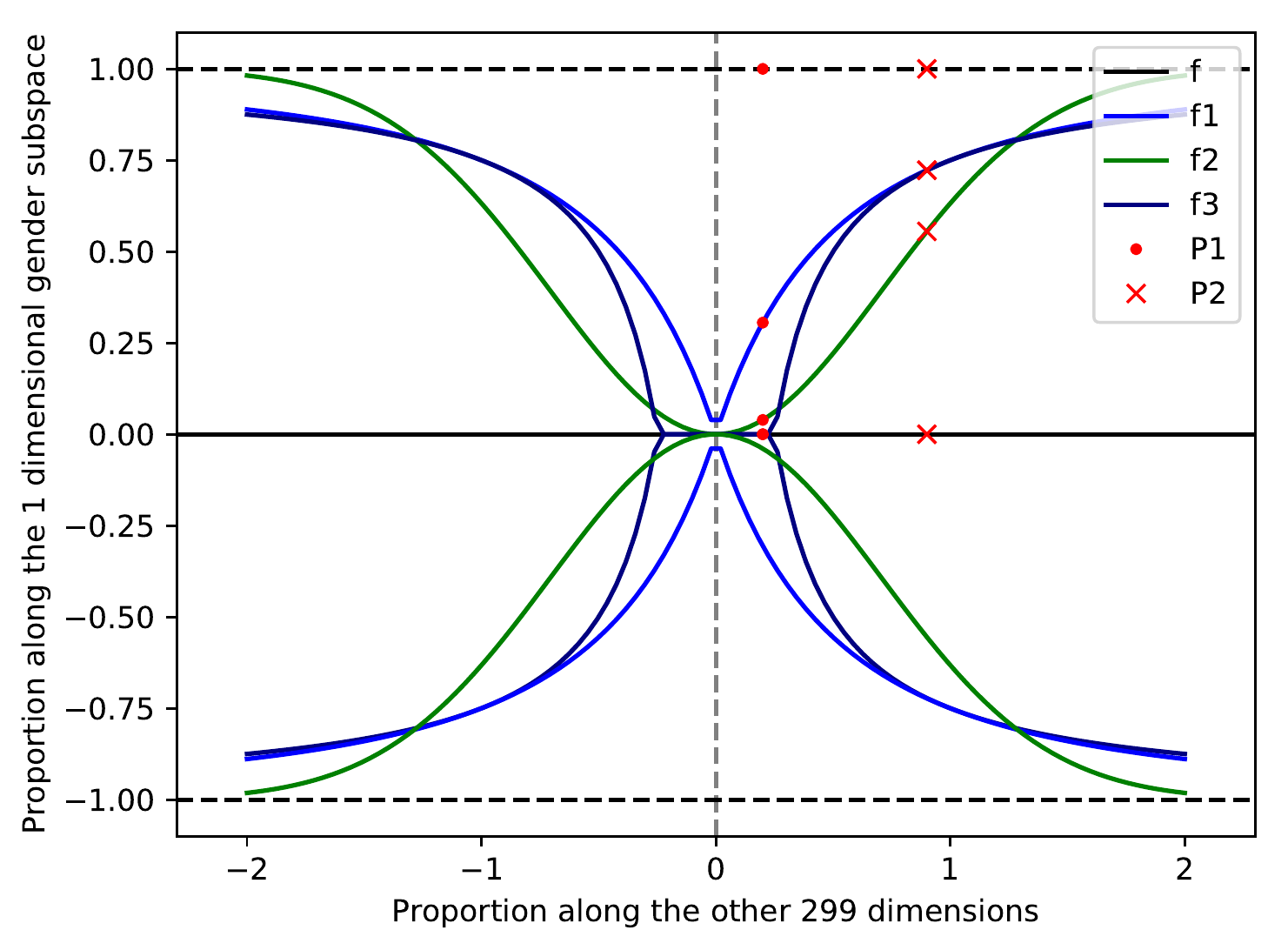}
	\small
	\caption{The gendered region as per the three variations of projection. Both points P1 and P2 have a dot product of 1.0 initially with the gender subspace. But their orthogonal distance to it differs, as expressed by their dot product with the other 299 dimensions. \label{fig : variations}}
	\normalsize
\end{figure}

	

  \begin{figure}
\captionsetup[subfigure]{justification=centering}
    \centering
      \begin{subfigure}{0.24\textwidth}
        \includegraphics[width=\textwidth]{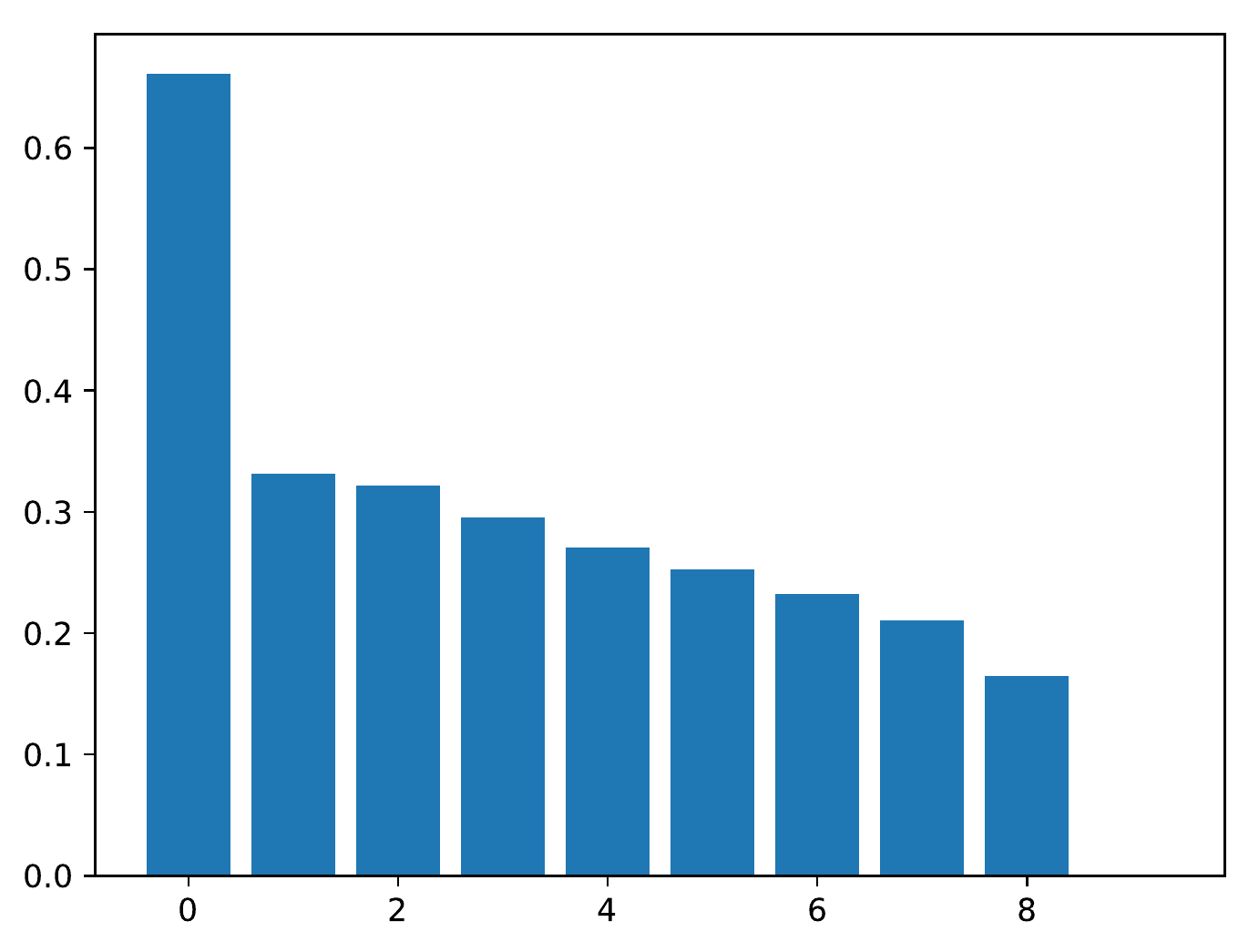}
          \caption{ 0.0}
      \end{subfigure}
      \begin{subfigure}{0.24\textwidth}
        \includegraphics[width=\textwidth]{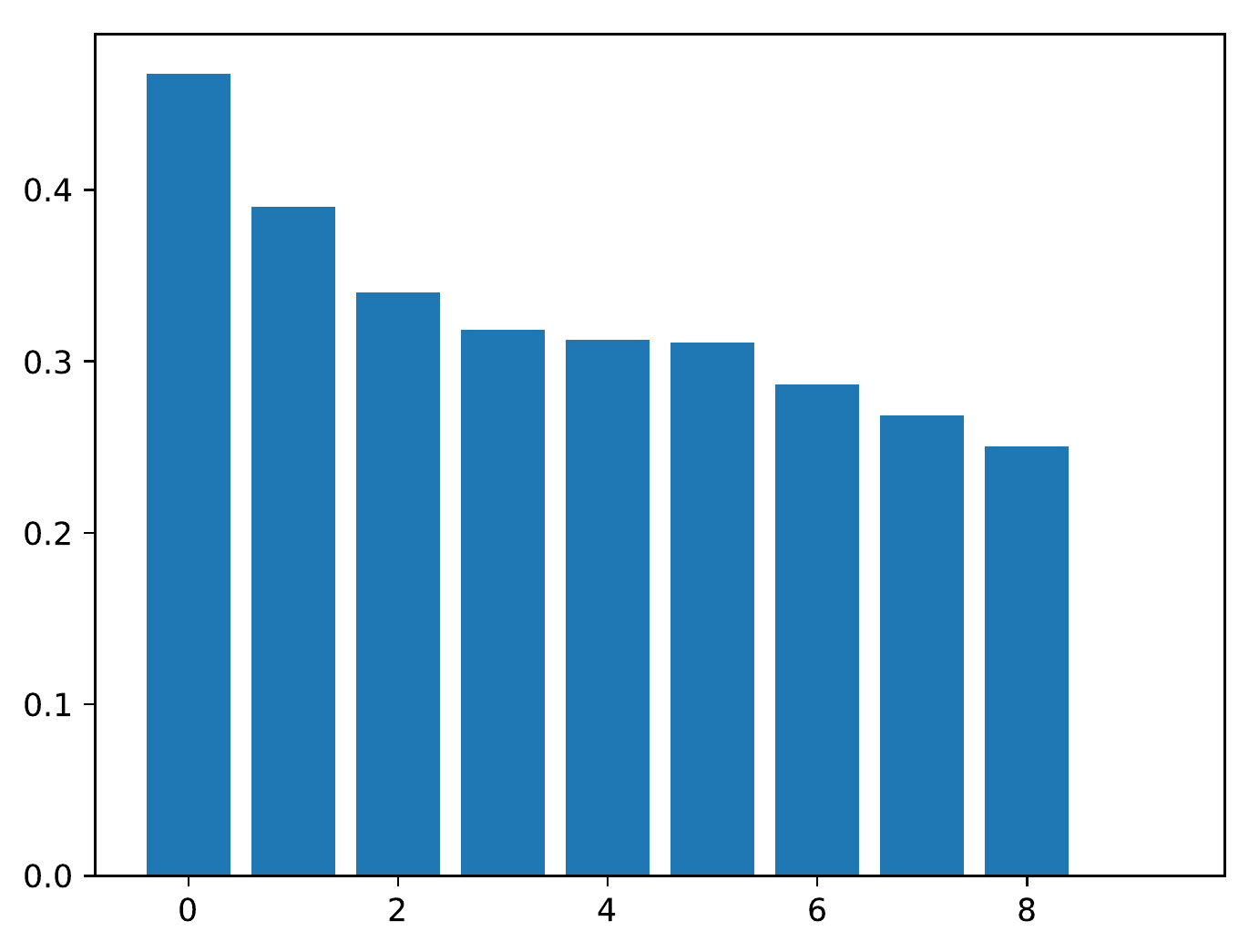}
          \caption{ 0.5}
      \end{subfigure}
      \begin{subfigure}{0.24\textwidth}
        \includegraphics[width=\textwidth]{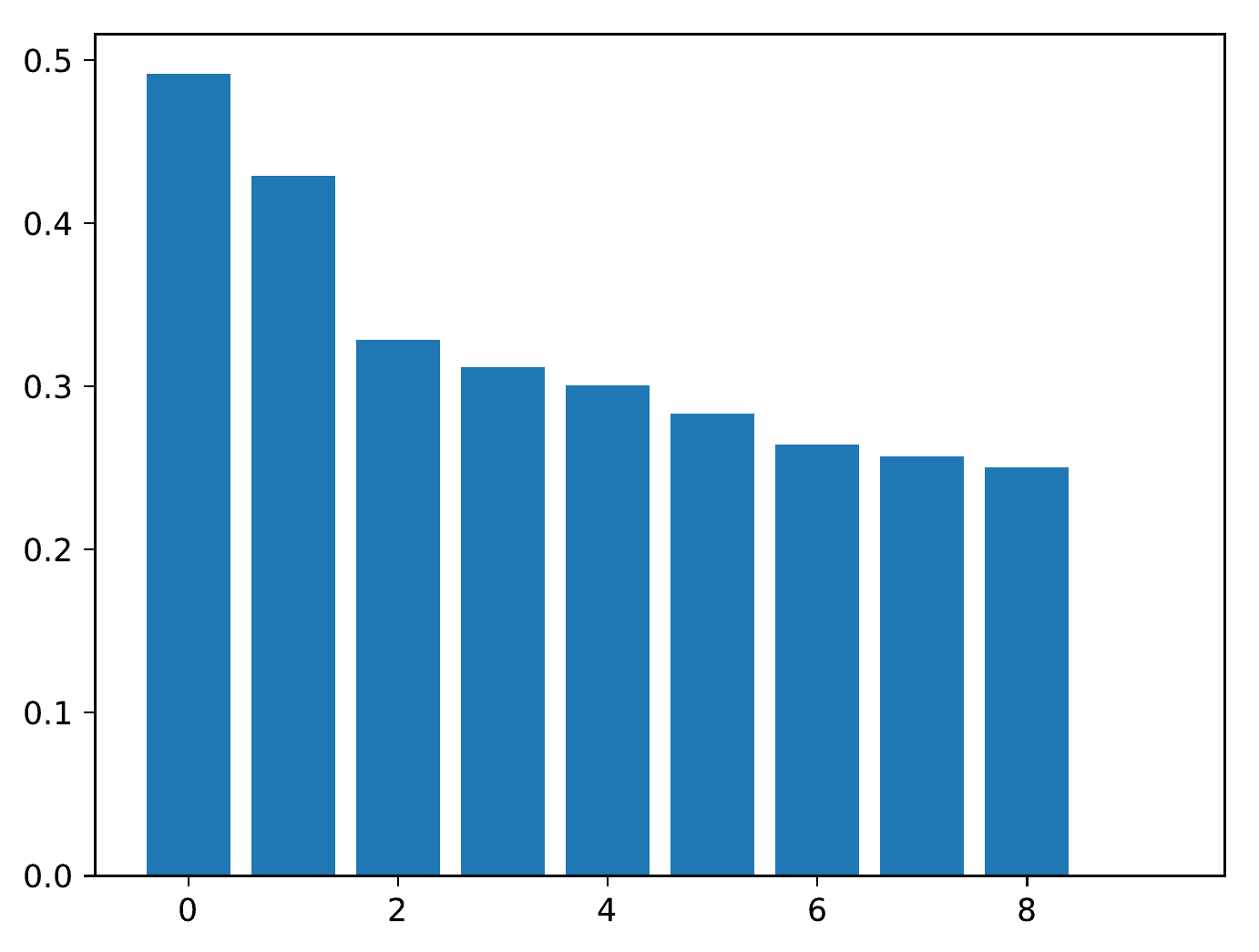}
          \caption{ 0.75}
      \end{subfigure}
      \begin{subfigure}{0.24\textwidth}
        \includegraphics[width=\textwidth]{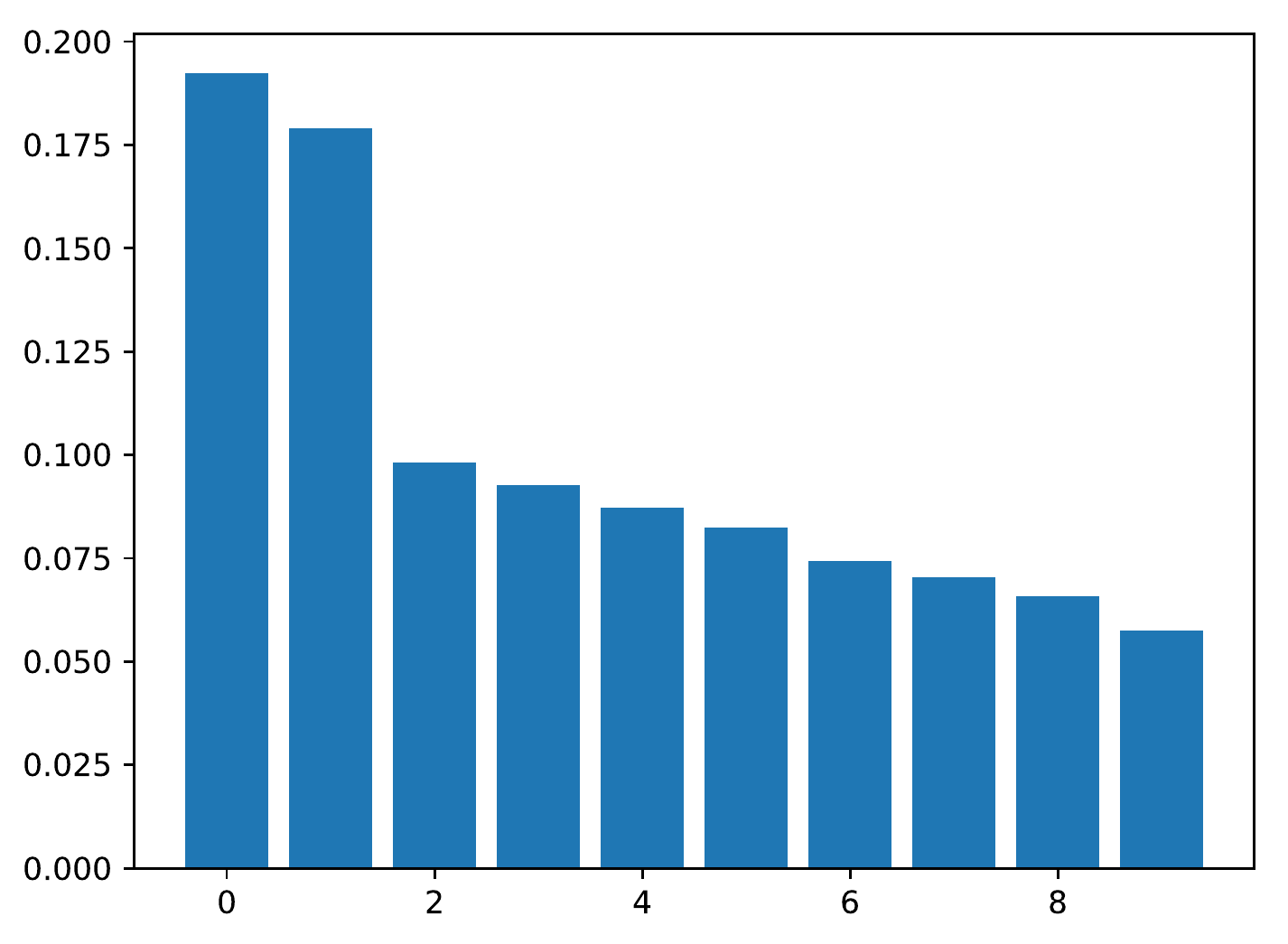}
          \caption{ 1.0}
      \end{subfigure}
   \caption{ Fractional singular values for avg male - female words (as per Table \ref{tbl:word-pairs}) after flipping with probability (a) $0.0$ (the original data set), (b) $0.5$, (c) $0.75$,  and (d) $1.0$.}
   \label{fig:flipping}
\end{figure}


\begin{table}
\vspace{1cm}	
\caption{What analogies look like before and after damping gender by different methods discussed: Hard debiasing, flipping words in text corpus, subtraction and projection.}
	\centering
	\footnotesize

	\label{tbl : analogies flipping}
	
	\resizebox{\columnwidth}{!}{\begin{tabular}{c|c|c|ccc|c|c}
		\hline
		analogy head & original & HD & \multicolumn{3}{c|}{flipping} & subtraction  & projection \\
		\cline{4-6}
		& & & 0.5 & 0.75 & 1.0 & & \\\hline
		\word{man} : \word{woman} :: \word{doctor} : &  \word{nurse} & \word{surgeon}  & \word{dr} & \word{dr} &  \word{medicine} & \word{physician} & \word{physician}\\
		\word{man} : \word{woman} :: \word{footballer} :  & \word{politician} & \word{striker} & \word{midfielder} & \word{goalkeeper} & \word{striker} & \word{politician} & \word{midfielder}\\
		
		\word{he} : \word{she} :: \word{strong} : & \word{weak} & \word{stronger} & \word{weak} & \word{strongly} &  \word{many} & \word{well} & \word{stronger}\\
		\word{he} : \word{she} :: \word{captain} : & \word{mrs} & \word{lieutenant} &  \word{lieutenant} &  \word{colonel} & \word{colonel} & \word{lieutenant} & \word{lieutenant} \\
		\word{john} : \word{mary} :: \word{doctor} : & \word{nurse} & \word{physician} & \word{medicine} & \word{surgeon} & \word{nurse} & \word{father} & \word{physician}\\ \hline
	\end{tabular}}

\end{table}

\begin{table}[h]
\vspace{1cm}
    \caption{Some of the most gendered words in default embedding; and most gendered adjectives and occupation words.} 
	\centering
\label{tbl : Gendered Words}
    \begin{tabular}{c c c c }
         \multicolumn{4}{c}{Gendered Words}\\ \hline 
         \word{miss} & \word{herself} & \word{forefather} & \word{himself}\\
         \word{maid} & \word{heroine} & \word{nephew} & \word{congressman} \\
         \word{motherhood} & \word{jessica} & \word{zahir} & \word{suceeded}\\
         \word{adriana} & \word{seductive} & \word{him} & \word{sir}\\
         
    \end{tabular}

    \begin{tabular}{c c}
	Female Adjectives & Male Adjectives  \\ \hline 
	\word{glamorous} & \word{strong}\\
	\word{diva} & \word{muscular}\\
	\word{shimmery} & \word{powerful} \\
	\word{beautiful} & \word{fast}\\
\end{tabular}

\begin{tabular}{c c}
	Female Occupations & Male Occupations  \\ \hline 
	\word{nurse} & \word{soldier}\\
	\word{maid} & \word{captain}\\
	\word{housewife} & \word{officer}\\
	\word{prostitute} & \word{footballer}\\
\end{tabular}

\end{table}

  \begin{figure}
\captionsetup[subfigure]{justification=centering}
    \centering
      \begin{subfigure}{0.24\textwidth}
        \includegraphics[width=\textwidth]{gender1.pdf}
          \caption{}
      \end{subfigure}
      \begin{subfigure}{0.24\textwidth}
        \includegraphics[width=\textwidth]{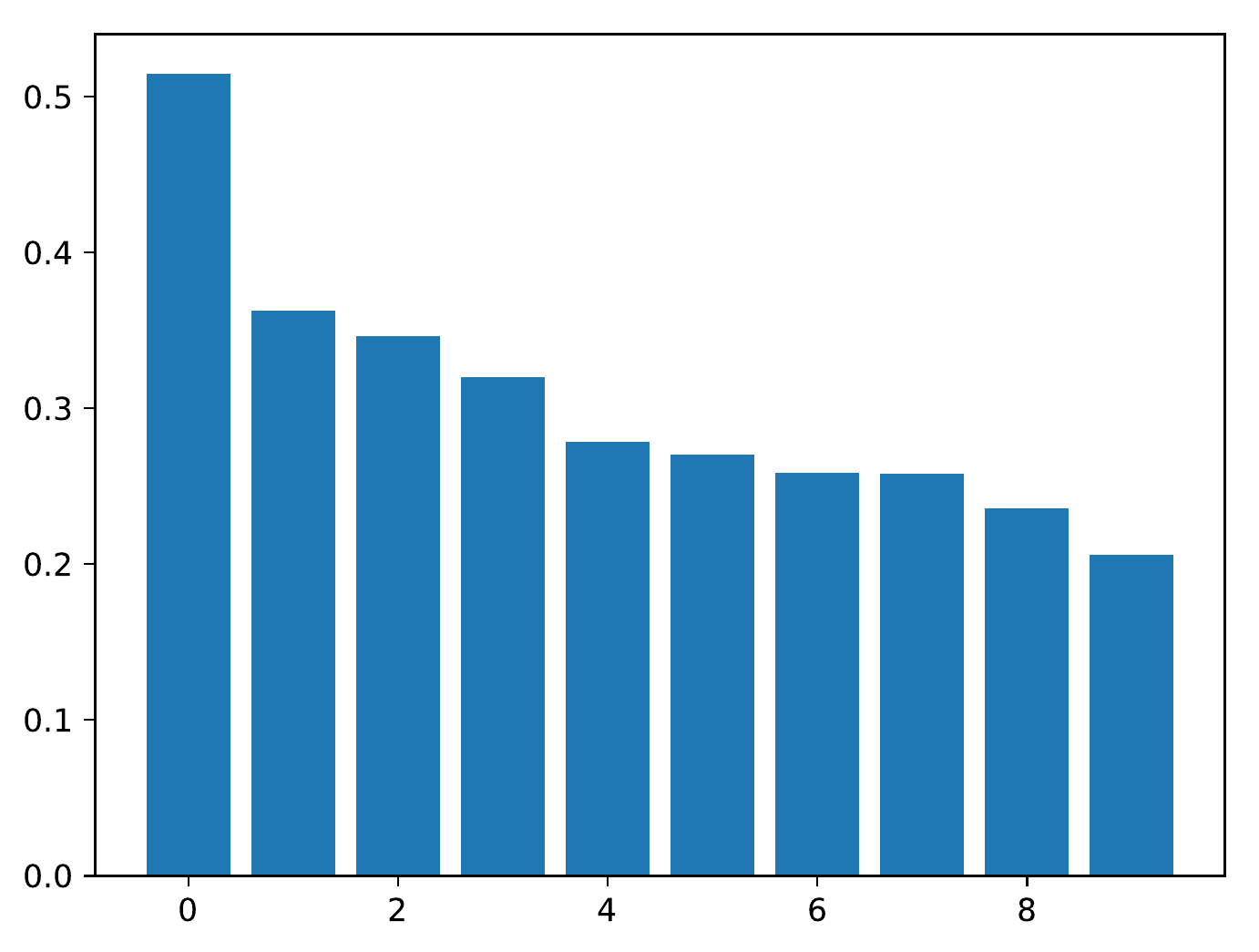}
          \caption{}
      \end{subfigure}
      \begin{subfigure}{0.24\textwidth}
        \includegraphics[width=\textwidth]{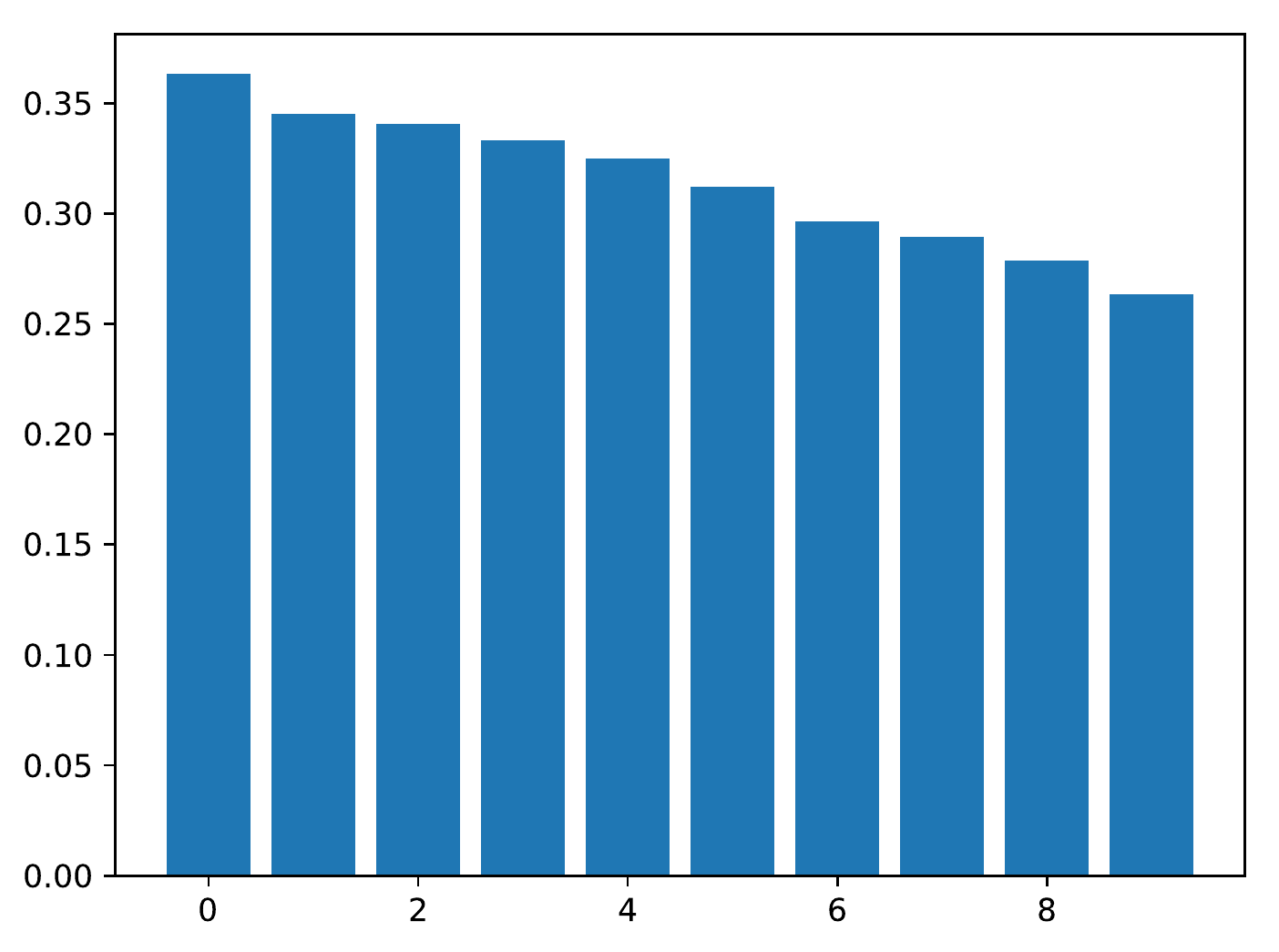}
          \caption{}
      \end{subfigure}
      \begin{subfigure}{0.24\textwidth}
        \includegraphics[width=\textwidth]{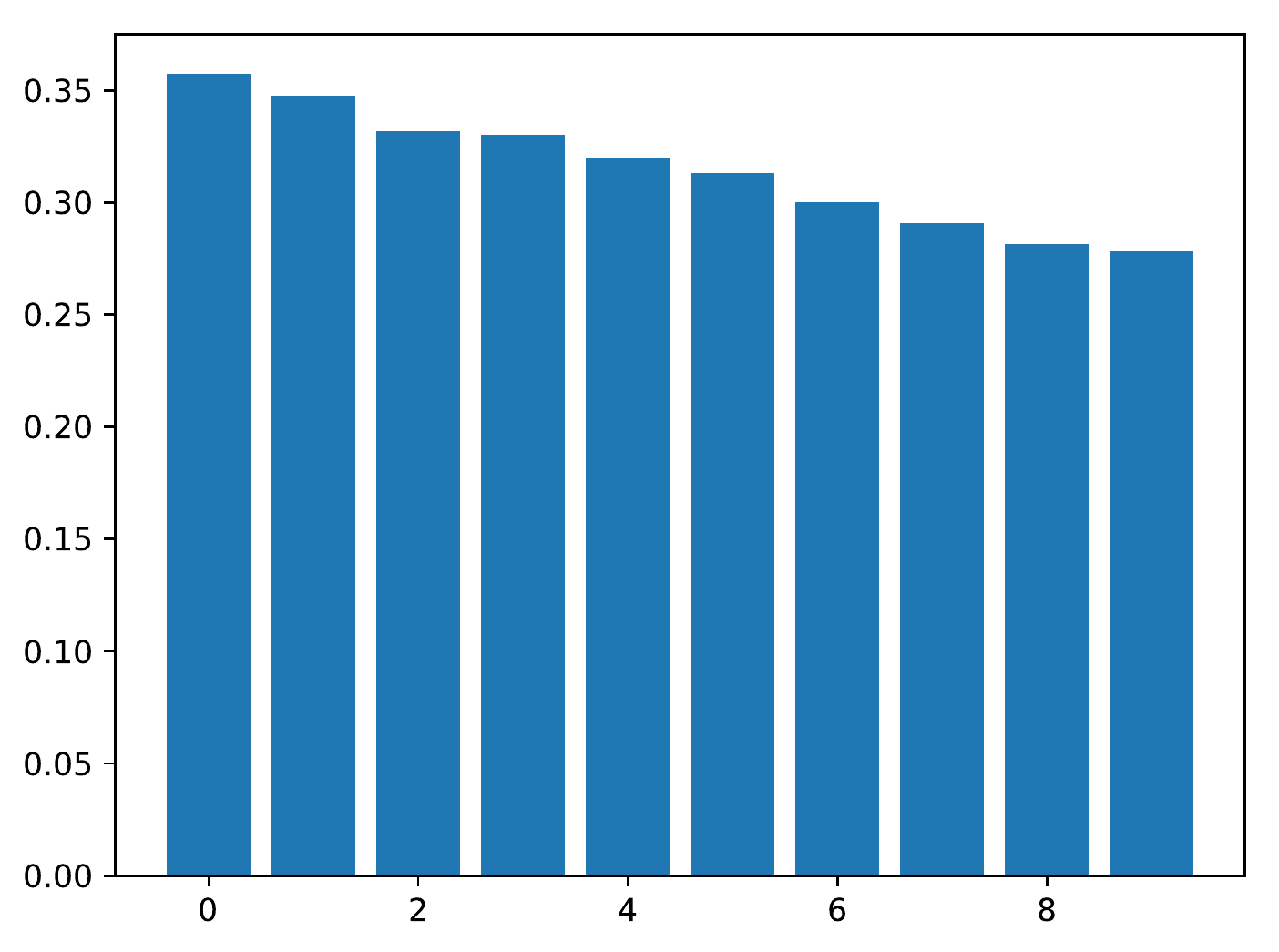}
          \caption{}
      \end{subfigure}
   \caption{Fractional singular values for (a) male-female word pairs, (b) one gendered word - one random word, (c) random word pair, and (d) random unit vectors.}
   \label{fig:SVDcharts}
\end{figure}

	


\begin{figure}
\vspace{1cm}
\captionsetup[subfigure]{justification=centering}
    \centering
      \begin{subfigure}{0.37\textwidth}
        \includegraphics[width=\textwidth]{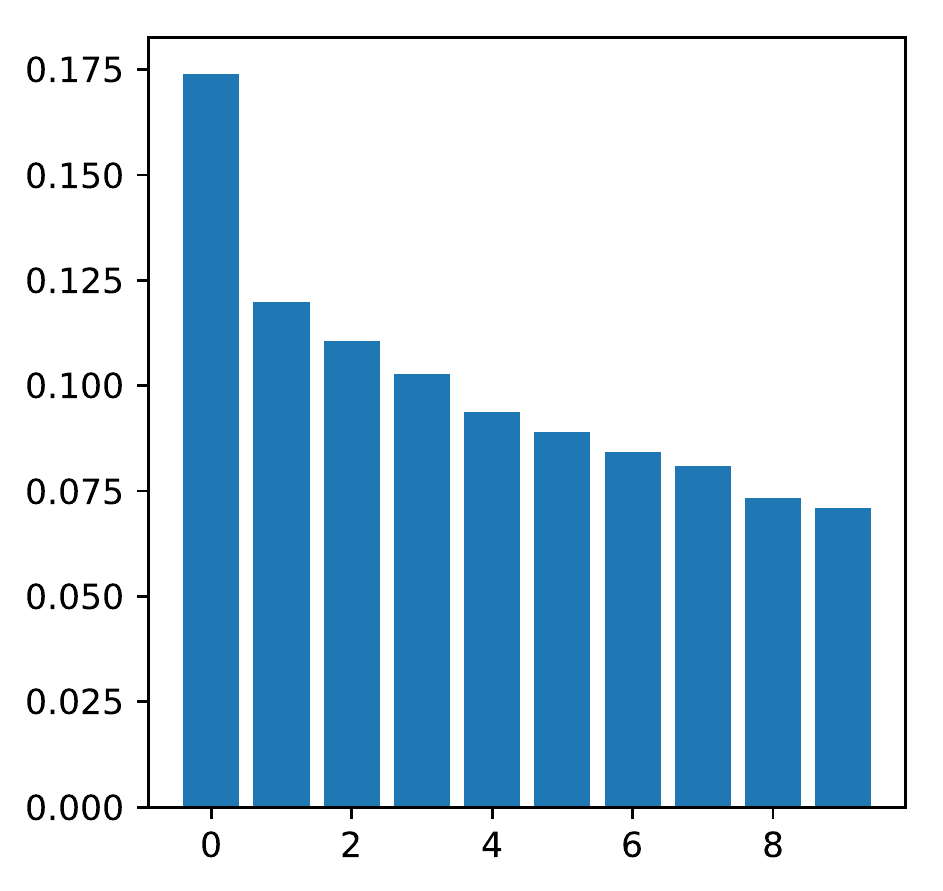}
          \caption{}
      \end{subfigure}
      \begin{subfigure}{0.45\textwidth}
        \includegraphics[width=\textwidth]{gender1.pdf}
          \caption{}
      \end{subfigure}
   \caption{Proportion of singular values along principal directions (a) using names as indicators and (b) using word pairs from Table \ref{tbl:word-pairs} as indicators.}
   \label{fig : names}
\end{figure}

\begin{table}[]
\vspace{1cm}
		\caption{Gendered occupations as observed in word embeddings using names as the gender direction indicator.}
\centering
	\label{tbl : occ names list}
	\begin{tabular}{c c | c c}
		Female Occ & Male Occ & Female* Occ & Male* Occ \\ \hline 
		\word{nurse} & \word{captain} & \word{policeman} & \word{policeman}\\
		\word{maid} & \word{cop} & \word{detective} & \word{cop}\\
		\word{actress} & \word{boss} & \word{character} & \word{character}\\
		\word{housewife} & \word{officer} & \word{cop} & \word{assassin}\\
		\word{dancer} & \word{actor} & \word{assassin} & \word{bodyguard}\\
		\word{nun} & \word{scientist} & \word{actor} & \word{waiter}\\
		\word{waitress} & \word{gangster} & \word{waiter} & \word{actor}\\
		\word{scientist} & \word{trucker} & \word{butler} & \word{detective}\\
	\end{tabular}

\end{table}

\begin{table}[]
\vspace{1cm}
		\caption{What analogies look like before and after removing the gender direction using names.}
	\centering
	
	\begin{tabular}{c|ccc}
	\hline
		analogy head & original & subtraction & projection \\ \hline
		\word{man} : \word{woman} :: \word{doctor} :  &  \word{nurse} & \word{physician} & \word{physician} \\
		\word{man} : \word{woman} :: \word{footballer} : & \word{politician} & \word{politician}  & \word{midfielder}\\
		
		\word{he} : \word{she} :: \word{strong} : &  \word{weak} & \word{very} & \word{stronger} \\
		\word{he} : \word{she} :: \word{captain} : &  \word{mrs} & \word{lieutenant} & \word{lieutenant}\\
		\word{john} : \word{mary} :: \word{doctor} : &  \word{nurse} & \word{dr} &  \word{dr} \\ \hline
	\end{tabular}
	\label{tbl : analogies}

\end{table}

 \begin{table}[]
 \vspace{1cm}
    \caption{Performance on ECT, EQT, and WEAT by the different debiasing methods; and performance on standard similarity and analogy tests.}

    \centering

    \resizebox{\columnwidth}{!}{\begin{tabular}{c|c|c|ccc|cc|cc}
    	\hline
        analogy head & original & HD & \multicolumn{3}{c|}{flipping} & \multicolumn{2}{c|}{subtraction}  & \multicolumn{2}{c}{projection} \\
       \cline{4-10}
         & & & 0.5 & 0.75 & 1.0 & word pairs & names & word pairs & names\\
           \hline 
           ECT (word pairs) & 0.798 & 0.917  & 0.983  & 0.984 & 0.683   & 0.963  & 0.936  & 0.996 & 0.943 \\
           ECT (names) & 0.832 & 0.968  & 0.714  & 0.662 & 0.587  & 0.923  & 0.966  & 0.935 & 0.999 \\
           EQT  & 0.128 & 0.145  & 0.131  &0.098 & 0.085  & 0.268 & 0.236  & 0.283 & 0.291\\
           WEAT & 1.623 & 1.221 & 1.164 & 1.09 & 1.03  & 1.427  & 1.440 & 1.233 & 1.219\\
         \hline
     		WSim & 0.637 & 0.537 & 0.567 & 0.537 & 0.536 & 0.627 & 0.636 & 0.627 & 0.629\\
		Simlex & 0.324 & 0.314 & 0.317 & 0.314 & 0.264 & 0.302 & 0.312 & 0.321 & 0.321\\
		Google Analogy & 0.623 & 0.561 & 0.565 & 0.561 & 0.321 & 0.538 & 0.565 & 0.565 & 0.584 \\ \hline 
    \end{tabular}}

    \label{tbl : Test 2 - analogies}

\end{table}

\begin{table}[]
\vspace{1cm}
		\caption{\label{tbl:damping} Performance of damped linear projection using word pairs.
}
	\centering

	\begin{tabular}{c|ccccc}
		\hline

		Tests & f & $f_1$ & $f_2$ & $f_3$\\ \hline
		ECT & 0.996 &0.994 & 0.995 & 0.997\\
		EQT & 0.283& 0.280& 0.292 & 0.287\\
		WEAT & 1.233 & 1.253 & 1.245 & 1.241\\ \hline
		WSim & 0.627 & 0.628 & 0.627 & 0.627\\
		Simlex & 0.321 & 0.324 & 0.324 & 0.324\\
		Google Analogy & 0.565 & 0.571 & 0.569 & 0.569\\				
		\hline
	\end{tabular}

\end{table}

\begin{figure}[h]
\includegraphics[width=.98\linewidth]{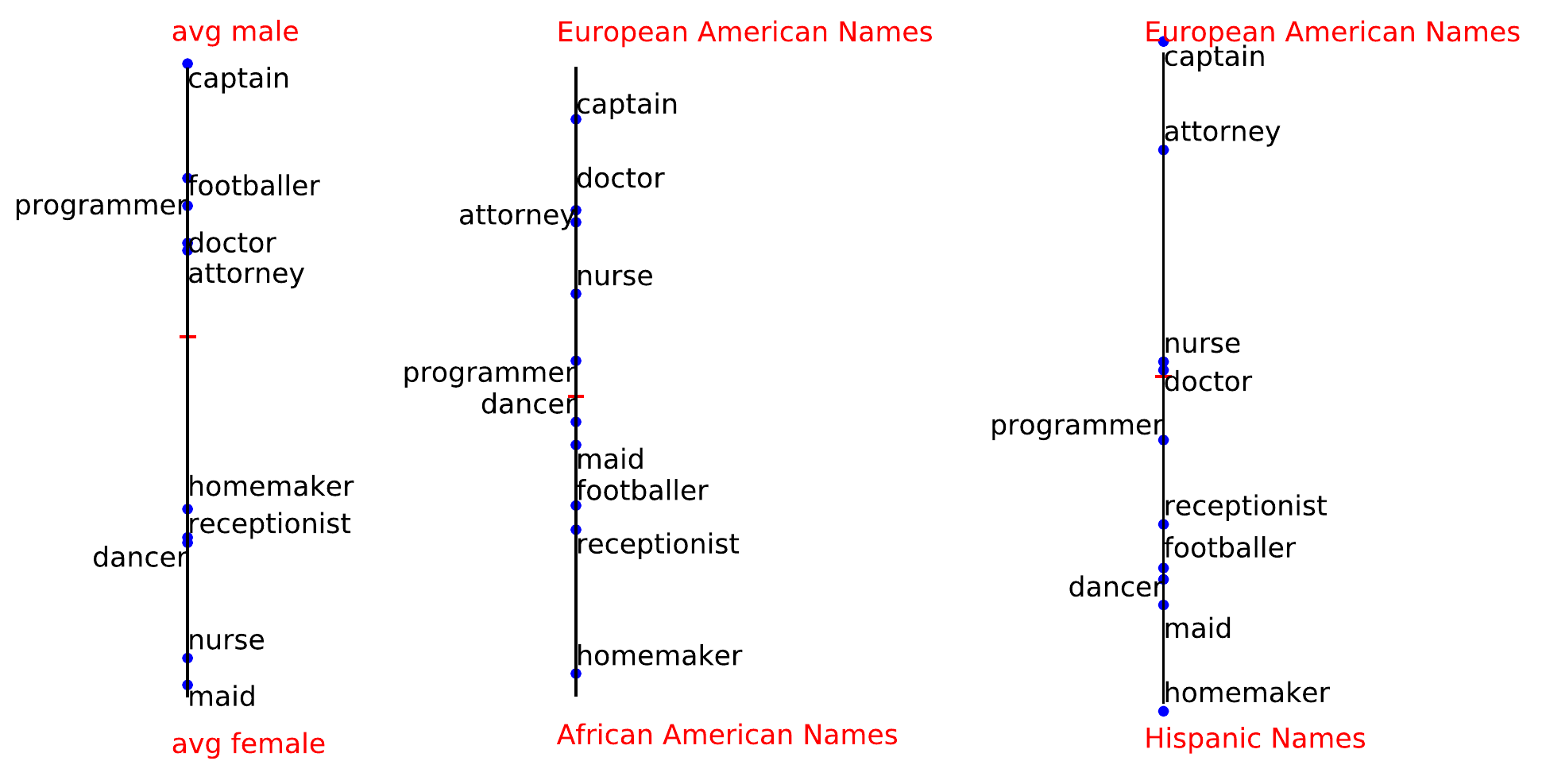}

\caption{\label{fig : hispanic} Gender and racial bias in the embedding.}
\end{figure}

\begin{table}[h]
\vspace{2.5cm}
     \caption{ WEAT positive-negative test scores before and after debiasing.}
    \centering

    \label{tbl : weat other bias}
    \begin{tabular}{c|cc}
    	\hline
        & Before Debiasing & After Debiasing  \\
        \hline
        EA-AA & 1.803 & 0.425 \\
        EA-H & 1.461& 0.480\\ 
		Youth-Aged & 0.915 & 0.704\\ \hline
    \end{tabular}

\end{table}

\begin{figure}[]
	\centering
	
	\includegraphics[width=  \textwidth]{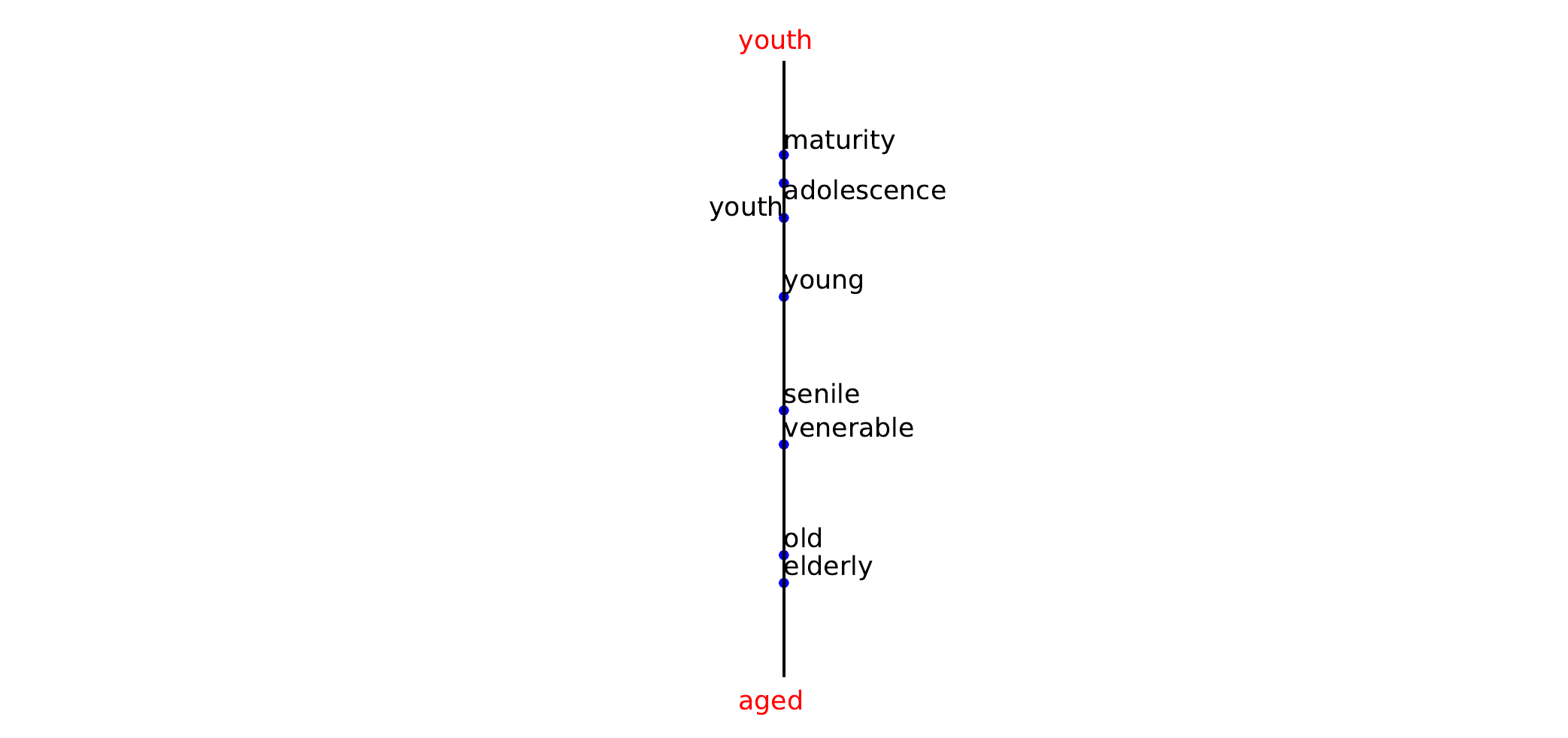}

	\caption{Detecting age with names : a plot of age related terms along names from different centuries.}
	\label{fig : age}
	
\end{figure}

%% file: chap-visa.tex
    
\chapter{Subspace Isolation in Noisy Embeddings}
\label{chap: visa}

Raw data in many real world settings is extremely noisy and has few markers and identifiers. For instance, if we look at web interaction based user representations, user identifiers don't hold as much significance as word meanings do in the case of text representations. The categorization of such users is thus unsupervised. In such settings, isolation of subspaces poses several challenges specific to the domain the data is from.

In this chapter\footnote{\normalsize{work done as part of summer internship at Visa Research}}, we explore one such setting where, in a representation based on noisy data, we identify and locate concept subspaces.

\section{Transaction-Based User Representations}
Extensive amounts of user and merchant data and transaction records are available to financial organizations. To utilize them for building applications such as recommendation tools for users or the detection of fraudulent activity, it is important for them to understand user behavior and preferences.

This data collected is both enormous and rich in features such as the place of purchase, the amount spent in each transaction, the average amount spent at a given merchant, merchant type (gas station, restaurant, etc.) merchant name, etc. Several hundreds of thousands of such transactions are recorded per day. Apart from just records of the transaction, this data is rich with implicit features like user preferences or merchant clientele categories. That is because patterns about merchants and users and their behavior can be observed in this rich dataset. For instance, for some merchant M1, say a grocery store, the people visiting it would most likely belong to locations within 5-10 miles of the merchant itself. For a different merchant type M2, a restaurant, there would be specific demographics of people conducting transactions at it based on how expensive the restaurant is and the cuisine it serves. Commutatively, for users such patterns can be seen based on the merchants they visit. Users tend to visit grocery stores close to their place of residence and visit coffee shops and restaurants close to their office. Users also have preferences with respect to store or cuisine type and price range. 

Understanding these patterns of preferences can be useful for organizations with access to this data. For one, user behavior understanding helps avoid fraudulent transactions being processed on credit or debit cards. Unusual behavior in terms of excess spending or expenses at a random location can be flagged easily for monitoring. Further, on the determination of types of clientele a certain merchant primarily caters to, or conversely, if we understand the merchant preferences of a user, we can build personalized recommendations for the user. These recommendations will not be generic in terms of popularity amongst users in a given area, but in addition to it, they will be able to predict if the merchant attributes (distance to the user, price, merchant type, etc.) fit the user's preferences. A recommendation system with these features will be better suited to the specific user and will be more powerful.

In this chapter, we focus on the task of personalization of the recommendation system for a user. Since the transaction records are rich with a lot of different merchant types and also have a lot of noise and missing or mixed attributes, building a generic recommendation system over all merchant types is an extremely challenging task. We thus restrict ourselves over the set of merchants that are restaurants. Restaurant recommendation is also a highly used tool, so there is significant demand and motivation for building a better suited and personalized restaurant recommendation tool. Most restaurant recommendation tools such as Yelp, evaluate preferences based mostly on popularity or distance to the person. While there are filters available, there is no way to personalize further. The benefit of having transaction data for personalization is two-fold. Firstly, when building the recommender system, if the exact set of merchants where purchases are made by a user is known, then understanding the implicit features of the merchant set helps predict better.  Also, there can be reinforcement after deployment of this system, where a purchase at the recommended merchants following the recommendation is positive feedback. This feedback can be used as a label for future recommendations.

The first challenge in building this system is to represent this data in a coherent manner such that it captures all its entities - the merchants and users - as well as the implicit features as discussed above. One scalable, efficient method of doing so is to create distributed embeddings. To do so, we use a standard word embedding technique, which we discuss in Section \ref{sec: making embeddings}. The resultant embeddings however are significantly influenced by specific dominant meta-information associated with them like their location or their popularity (expressed as their frequency in the transaction data) or the average amount spent at them each time. The factors of frequency and location in fact dominate over other features such as cuisine type. So, the similarity between any two restaurants is dominated by location and frequency more than the cuisine. This makes building tasks such as recommendation systems for restaurants for users hard and inaccurate, with the system suggesting restaurants with similar location and popularity rather than the cuisine preference of the user. This makes it essential to tease out the different subspaces within the embedding space that represent different features such as location or volume or price. In this project, we focus on just removing the location information from the embedding space, thus merging different restaurants from different locations based on other features like volume or cuisine.
We do this by using post-processing techniques as our merchant embeddings are built on extremely large data and retraining the embeddings is an arduous and time consuming task. Further, the use of linear methods ensures higher interpretability in terms of what is changed in the embedding to achieve the removal of location from any vector in the embedding. As only linear transformations are applied, the other relationships between any two points in the embedding are also retained. While the method used to achieve this is an extension from Chapter \ref{chap: bias paper 1}, the main takeaway in this chapter is how do we interpret the embeddings which do not have coherent `word meanings' associated with them. 

\section{Understanding and Representing the Transaction Data}
About 6 months of transaction data from a large financial organization is curated for this task. The users are anonymized and unidentifiable from the data but merchant ids can be traced to merchant names and locations. Each merchant data point has the following attributes:

\noindent \{\textit{Merchant ID, Merchant Name, City, State, Country, Zipcode, Merchant Category, Latitude, Longitude, Number/Frequency of Transactions, Average Transaction Amount, Average Tip Amount, Amount Variance}\}

Merchant ID is a unique identifier for each merchant. So, two different stores of the same merchant chain would have distinct IDs.

Alongside, we also have text files wherein we have a sentence denoted by a single user's transaction sequence. Each sentence here would be a sequence of merchant ids such as ``$M_1M_7M_9M_1...$". This is described in detail in Section \ref{sec: making embeddings}

Similarly, there is data for users where each user has tags and features attached to them. We restrict ourselves to the use of only merchant information as it has less sensitive and thus more public information. 

\subsection{User and Merchant Embeddings}
\label{sec: making embeddings}
Being able to understand patterns requires representing these merchants and users in a meaningful and scalable manner.

Billions of transaction records are collected every day. On filtering and restructuring the data, we can have for each merchant, a set of all users that visited it and for each user, a set of all merchants visited by it. Further, we can structure in a manner that we preserve the order in which users visited a given merchant, or for each user, the order in which he visited merchants. This gives us meaningful sequences of merchants and users. 

To get a bigger picture however, we want to embed each of these merchants as a point in some high-dimensional space. This is because, to recommend merchants, it is important that we understand what merchants are considered similar as per user behavior. So, from here on we will consider only those sequences created per user. Each sequence looks like a string of merchant ids, each representing some distinct merchant. For instance, a sequence for a user could like ``$M_1M_{10}M_{55}M_1..."$. Thus, for each user, we have an ordered sequence of merchants visited. A coffee shop purchase of a certain amount at regular intervals, on a daily basis, can be seen for users based on their sequences. Or, the repetition of merchants of a specific cuisine at a higher frequency than others can also be observed in these sequences. Lunch followed by dessert in the form of two purchases of significantly different amounts, consecutively can also be seen in these sequences. Thus, patterns for users can be observed by looking at these merchants.

Now, if we consider for each user, the sequence of merchants visited to be a ``sentence", then the list of such user ``sentences" gives us a large text file. Each word here is a merchant id of the form ``$M_i$" and just as in a regular text file, tokens in the sequence don't need to be unique. On this text file of transaction sequences, we can run any word embedding algorithm like word2vec~\cite{wordtovec} or GloVe~\cite{glove}. This allows us to can extract embeddings for each merchant id in a way similar to word vectors. Each point in the embedding space now corresponds to a unique merchant id and for the set of merchant ids, we now have a fixed lookup table.

User embeddings can be created in a similar manner when we consider each merchant to be a sentence consisting of the ordered sequence of users that visited it, for e.g., ``$U_1U_5U_{53}...$". This gives us a vector space where the points denote a distinct user rather than a merchant. Since we characterize and recommend merchants here, we will restrict ourselves to the case of studying merchant embeddings.

\subsection{Embedding Attributes}
Using the trick of treating consecutive merchants visited as the context for a user, we embed merchants like words of a language. This allows us to represent these merchants in some ~$100$ dimensional space which captures a lot of implicit attributes. 

For any user, the merchants visited are influenced by several factors. Firstly, the physical location of the merchant has to be close to that of the user. They need to be in the same city as the user for the user to frequent it. On a finer scale, distance to different merchants also impacts the preference of users. Next, the price or average amount of money spent at a merchant is factored into the user’s choices of places to visit. The average popularity of a place is also reflected in the dataset as a whole by the frequency of a merchant id in the cumulative text file. This is captured and expressed by the merchant embeddings.

The extent to which each of these different attributes is expressed and affect the vector's position in the space varies. The location of the merchant has been observed to be a major influencing factor and dominates the vector values. In Figure \ref{fig : sf-man}, we see how merchants in San Francisco (in blue) are distinctly located in a different cluster from merchants in Manhattan (in red). 

Other factors such as price or cuisine type are expressed by these merchant vectors but are overpowered by the more dominant factors. For instance, cuisine can be seen to be similar when examining the nearest neighbors of a chosen merchant vector but only within a location and volume cluster.

When recommending merchants in a personalized manner to users, some features may be more important than others. For example, say we want to be able to recommend based primarily on cuisine type, followed by other factors such as price and popularity. The sorting then can be done by distance to the user. Additionally, in most cases, a user requires recommendations when traveling, so the location information has changed. But in our embeddings, location overpowers all other attributes and thus, it is not possible to find nearest neighbors based on cuisine type across the location of origin and the new location. In this case, we want the location information contained by the representations to be reduced. This in principle is a similar problem to what we described with respect to text representations in Chapter \ref{chap: bias paper 1}, where a specific concept subspace needs to be determined and isolated. In the next section, we try to extend our methods to identify the location difference subspace and reduce its influence on the user embedding space and help identify the recessive factor of the cuisine type better.

\section{Finding the Location Subspace}
We extend our method of locating a subspace as defined in Chapter \ref{chap: bias paper 1} here. The aim is to be able to determine the subspace that captures the difference in location between two groups of merchants.
\subsection{The Location Bias Vector}
For a vector space V, we can find the direction of ``bias"
or the direction of difference between two subsets $V_1$ and $V_2$ by the 2means method described in Chapter \ref{chap: bias paper 1}.

There are a few different methods for finding this direction, one of the more common ones being calculating the PCA of the stacked difference vectors \cite{debias} between corresponding vectors between two groups (like man - woman, he - she, in the case of gender). With merchants in two cities however, there are no corresponding merchants. What this means is there is no merchant in city A that is the exact same as some other merchant in city B. Hence, we used the 2means method \cite{Bias1} which does not need corresponding points and aggregates over groups of points to find the target difference vector.

So, for two cities, say Manhattan and San Francisco, we create two can subsets $V_1$ and $V_2$ of merchants belonging to Manhattan and San Francisco respectively. Applying 2means between these two groups should give us the direction between the two cities in the merchant embedding.

In Bolukbasi’s  PCA based method, we would need to say a merchant in San Francisco which was the exact counterpart for some other merchant in Manhattan, which is rare if at all. Even if we consider chain restaurants like say, Starbucks, there is no one particular Starbucks in San Francisco which is the exact counterpart of some Starbucks in Manhattan.

It is pivotal to note the importance of the merchant sets $V_1$ and $V_2$ from Manhattan and San Francisco respectively. They serve as anchor points for finding the bias vector. The question then is how do we select the appropriate set such that we are able to accurately isolate the location vector as the difference vector $v'$?

The first approach to answer this is to select random points from each city. The hope is that this unbiased selection will represent each city fairly well and will thus give us the location difference in the vector $v'$. However, as explained in Section \ref{sec: eval}, we experimentally see that this does not give us the accurate location vector. 

Why does this act like a random vector? Let's again go back to the method of random selection on merchants from a city. Each merchant from the two groups $V_1$ and $V_2$ do differ in location. But additionally, each merchant within $V_1$ and $V_2$ and across $V_1$ and $V_2$ also differ in other parameters such as cuisine type, price range and volume. All of these also influence the vector representations of the merchants in $V_1$ and $V_2$ and act as noise in the difference vector $v'$ which should have been just the difference in location but now contains this difference in other parameters too. Since we randomly selected these merchants, the parameters they have are also random for each location. The aggregation thus is unstable, depends heavily on the random set selected and is close to that of a random vector. 

Thus, there is a need to actually refine this 2means method to appropriately find the location difference vector. In language, when trying to find gender difference direction, we have gendered words to guide us and act as anchor points ($V_1$ and $V_2$) but that is not the case in other non language embeddings such as merchant embeddings. The existence of word meanings in a language is a big advantage that is lacking in other distributed embeddings such as these merchant embeddings. In the following section, we try to find a more robust way of finding this bias vector in such noisy and unguided settings. 


\subsection{Refining the Bias Vector}
In the random selection of subsets $V_1$ and $V_2$, there were too many variable parameters that crept into the $v'$ vector calculated. So, we try different ways to fix these parameters across both groups $V_1$ and $V_2$. The reason is that once the only parameter varying amongst all merchants in $V_1$ $\cup$ $V_2$ is location, the difference in aggregates over $V_1$ and $V_2$ should capture the location difference. 



\begin{itemize}

  \item Quad Trees for Price and Volume Refinement: The volume or frequency variable of a merchant denotes the number of times it has been visited by users. This increases its frequency of occurrence in the text file used by an algorithm like word2vec or GloVe to generate the merchant embeddings. Since the algorithm sees more instances of some merchants, it also updates it more. This in turn makes the embedding of these merchants more resilient and stable. Along with location, this parameter largely dominates a merchant's vector embedding.

The price is another factor with significant influence. It also varies across different merchants but it is more of a categorical change in that, there are different blocks of values or price ranges that different merchants fall into. 

Since there is no direct correlation between the price and volume of different merchants, we require to fix both of these across the merchants in our groups $V_1$ and $V_2$. We find the intersection of restaurants which are visited with similar frequencies and are of a similar price range, we build quad-trees on the volume - price range. We ideally would want the frequency in both groups selected to be moderately high. This ensures that the merchants thus selected have stable embeddings. We don't select very high frequency as that gives us only a very small number of merchants to work with. Similarly, we would want to find an average price range that would have a large number of merchants. 
But we want this selected range of volume and price to be the same in both the cities. For that, we need to understand the spending patterns in the two cities and know which ranges have similar high merchants.

We visualize this using the k-d trees based quad-tree approach~\cite{quad-tree}. 
This approach divides the two-dimensional space of volume and price into smaller and smaller blocks iteratively such that each 2-D block roughly contains the same number of merchants in them. This helps ensure that the range picked eventually does not have too many or too few merchants. This partitioning trick enables us to pick corresponding price and volume ranges from the two cities/locations in question from which we get the merchants for subspace determination.

Figure \ref{fig:quad-tree} shows the spending patterns of Manhattan and San Francisco using this method. Some important things to note here are :
(i) Both the cities have very similar divisions in the volume-price space. The blocks in both look very similar. (ii) Blocks with extremely high volume and/or price are relatively very low populated. (iii) 0-10000 volume and 10-20 USD price values are where most merchants are.
This aligns well with what we expected. Using boxes that look most similar and have high frequency, we can now choose the same block across both cities and thus select the merchants within that block. This effectively fixes the merchant volume and price values across sets $V_1$ and $V_2$.

\item Cuisine Space:
Since we want the cuisine of a restaurant to be expressed well by the embedding, it is vital that we not lose any information about it contained in the embeddings. When we randomly select merchants for sets $V_1$ and $V_2$, merchants with differing cuisines can get selected. This then includes components of cuisine in the difference vector $v'$. We can ensure this does not happen by fixing the cuisine type across $V_1$ and $V_2$. Now, the cuisine label of a merchant is categorical and has a lot of missing values. Most merchants recorded in the dataset do not come with the cuisine label. The low fraction of labeled data also prevents automating the prediction of the label with high accuracy.
This obstructs the fixing of a cuisine type. 

However, the assumption that all merchants within a chain restaurant all serve the same cuisine is fair. For example, all Taco Bells serve Mexican fast food and all McDonalds serve American fast food. So, we choose only those restaurants in $V_1$ and $V_2$ that belong to the same chain such as Taco Bell, McDonald's, or Starbucks, the cuisine type of which remains constant across all stores in the United States.
\end{itemize}
\subsection{Removing the Location Subspace}

To remove the location difference subspace, we use the method of linear projection, applied on the whole set of points as introduced in Chapter \ref{chap: bias paper 1}. So, for any merchant represented by vector $v$, with the location vector l, we update the $v$ as
\[
 v = v - \langle v,l \rangle v 
\]
This brings all vectors into a lower dimensional space uniformly. It preserves all information contained in the embedding apart from the specific subspace removed, in this case, the location difference subspace.

\section{Evaluation}
Word embeddings have a variety of metrics that help indicate the quality of the embeddings (similarity or analogy tasks as well as downstream task performance such as at textual entailment) \cite{wsim,rg,snli}. There are also a large number of tests that detect the presence or absence of certain social biases like gender in them and also quantify the amount of such biases present \cite{Caliskan183,Bias1,dev2019measuring}. These tasks cumulatively help assess the accuracy and performance of the embeddings. Here, we try to capture the location bias captured in embeddings by the task of classification.


\subsection{Classification}
\label{sec: eval}
The classification of merchants by location is a good quantitative indicator of how well separated merchants in two locations are. For merchants in two locations, whose location difference is indistinguishable, the classification accuracy should be equivalent to flipping a coin to determine its location.


To verify if location indistinguishability is achieved, we compare the classification accuracy before and after we have removed the location subspace. If the accuracy decreases, we have succeeded in dampening the influence of the location subspace on the merchant vectors. We train an SVM (Support Vector Machine) for this classification over a set of approximately 30,000 merchants from the two locations. We then predict the accuracy of this binary classification over a held out set of 5,000 merchants. We keep the hyperparameters such as the validation split constant for both classification models (before and after location vector removal).
In Figure \ref{fig : sf-man merge}, we see how the San Francisco and Manhattan points merge in the space after a projection step using the 2means method. The classification boundary between the two cities is now more obfuscated than with Figure \ref{fig : sf-man}. This is quantified in Table \ref{tab: eval} which records the classification accuracy based on location. In the first column, we have the classification accuracy in the initial embedding which is high at about 99\% each. In the next column of random projection, the vector of location was determined used random merchants from either city/state. We see that this performs similarly to the initial embedding. The final column is where the vector was determined using a refined approach of the selection of merchants using the quad-tree method as described. Here, a significant decrease in accuracy is observed. The same is observed for merchants of California and Texas. Figure \ref{fig: texas}(a) represents this visually. Merchants from California (in red) and Texas (in blue) are in separable clusters. We further see how each of the differently colored clusters is further divided into three sub clusters each. The three sub clusters are based on the difference in volume or the net transactions seen by the merchants over an identical period of time (about 6 months). Volume, thus, is another dominant influencing factor affecting the vector position of a given merchant. Figure \ref{fig: texas}(b) is after projection when the merchants now share the same space irrespective of the location.

\section{Discussion}
 Transaction data is very rich and vast and trying to find patterns in it with respect to specific features is hard when there are clearly dominant features that influence their embeddings more. This desired disentanglement is generally very specific in that only the feature that is harmful or unwanted in the task is disentangled and removed. It is very vital that other features are not lost in the process. The method described here ensures that since the subspace captured is constrained appropriately. This transparency is not there in many other methods that automate the collapsing of two desired groups onto one another, such as the use of GANs \cite{wang2019feature}. 

We should note that even though we use location as the subspace to be untangled and removed from the merchant embeddings, it can be extended similarly to other feature dimensions such as price ranges wherein we want to mimic user behaviors irrespective of their usual price range and thus, across different price ranges. Further, we can use this technique on user embeddings too, where, for instance, we could remove location information of users to help merchants target different user groups irrespective of their location. This recommendation would be useful to merchants that conduct business on online forums. 

\begin{figure}[]
    \centering
    \includegraphics[height=0.2\textheight]{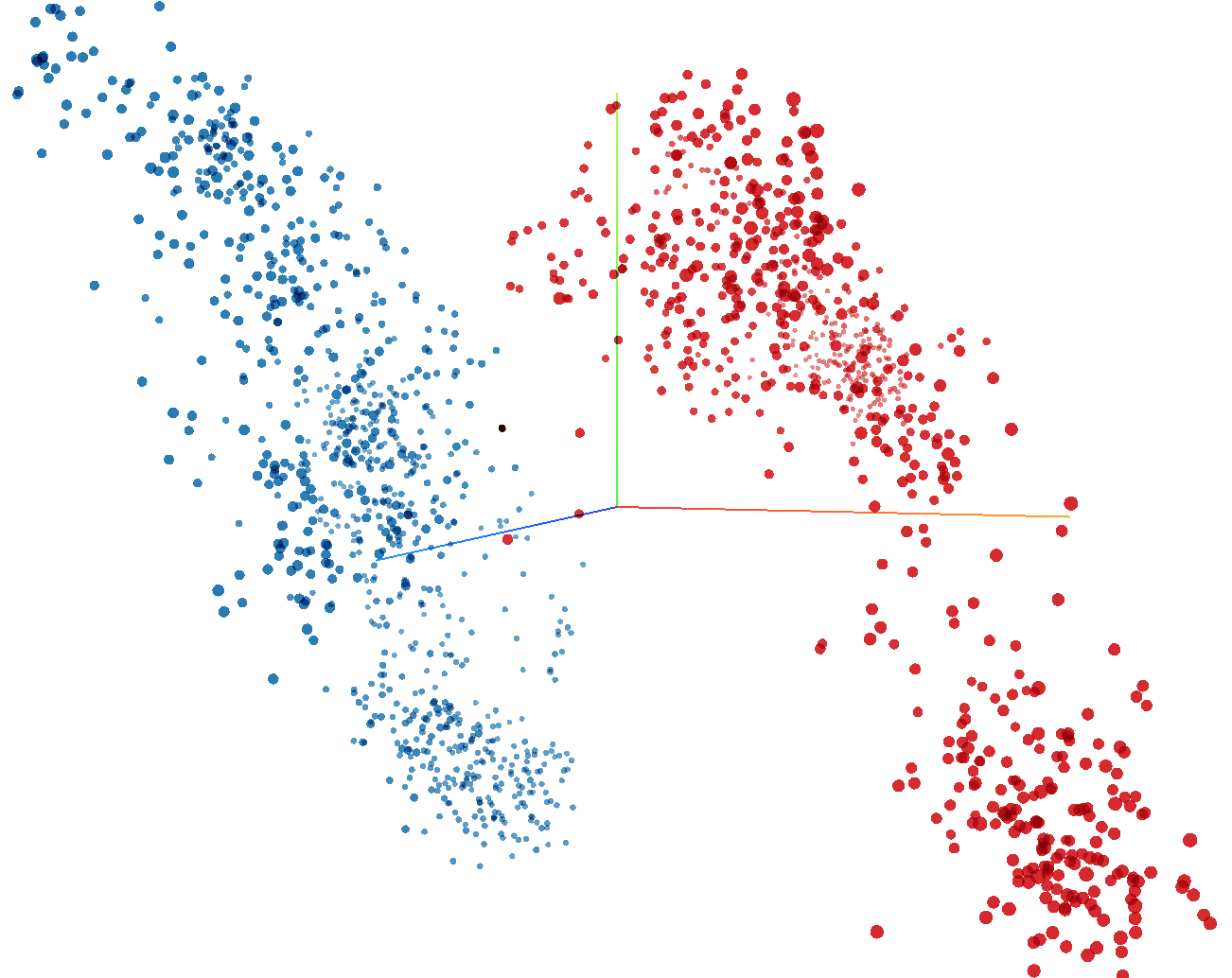}
    \caption{Merchant embeddings of San Francisco (blue points) and Manhattan (red points).}
    \label{fig : sf-man}
\end{figure}


\begin{figure}
\vspace{0.5cm}
\captionsetup[subfigure]{justification=centering}
    \centering
      \begin{subfigure}{0.45\textwidth}
        \includegraphics[width=\textwidth]{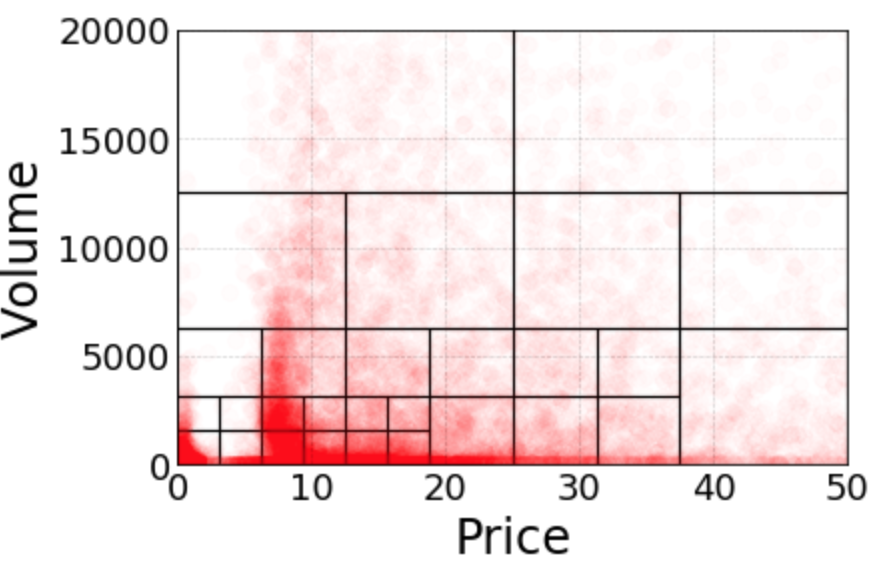}
          \caption{}
      \end{subfigure}
      \begin{subfigure}{0.45\textwidth}
        \includegraphics[width=\textwidth]{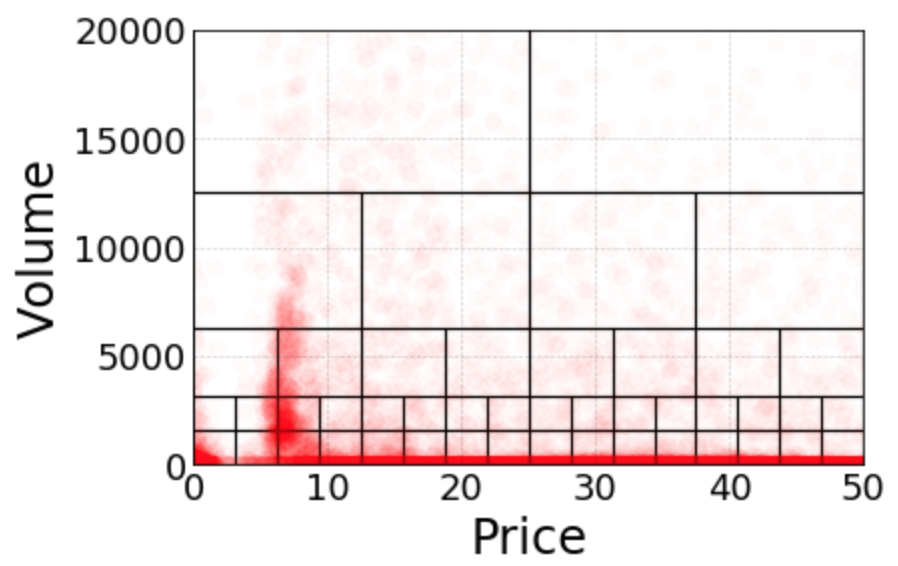}
          \caption{}
      \end{subfigure}
   \caption{Breaking down the price-volume space by the quad-tree approach of (a) Manhattan and (b) San Francisco.}
   \label{fig:quad-tree}
\end{figure}


\begin{figure}[]
\vspace{0.5cm}
    \centering
    \includegraphics[height = 0.3\textheight]{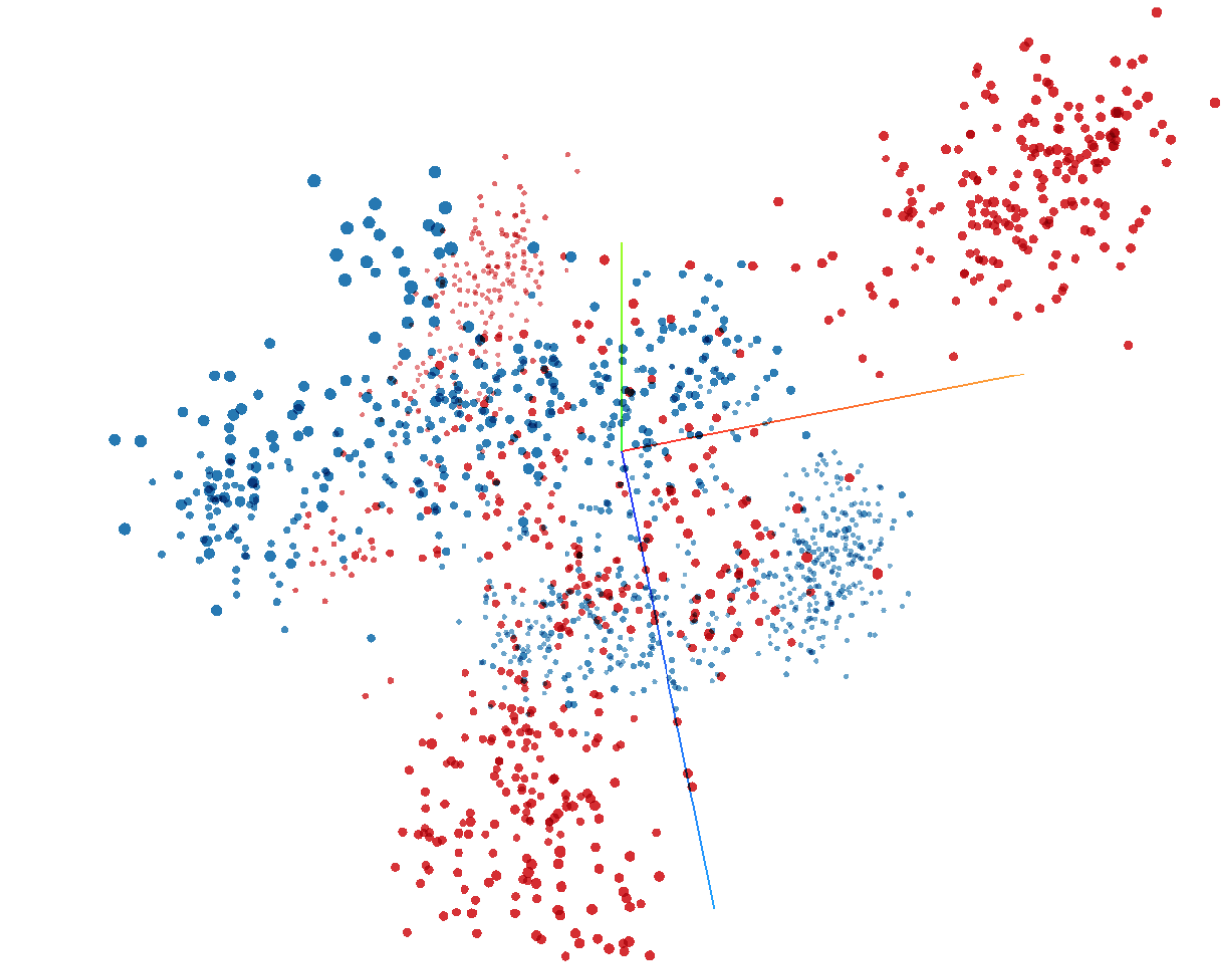}
    \caption{San Francisco (blue points) and Manhattan (red points) merchant embeddings post location subspace removal.}
    \label{fig : sf-man merge}
\end{figure}

\begin{table}[]

    \caption{Classification of merchants by city before projection, with projection with vector derived from random merchants from cities and projection with vector based on merchants picked by quad-tree method.}
    \centering
    \begin{tabular}{c|c|cc}
    \hline
    Locations & Initial Embedding & Random Projection & Refined Projection \\
    \hline
      San Francisco - Manhattan & 99.6  & 99.25 & 83.4  \\
     California - Texas &  98.7 & 99.03  & 86.9 \\
     \hline
    \end{tabular}

    \label{tab: eval}
\end{table}



\begin{figure}
\captionsetup[subfigure]{justification=centering}
    \centering
      \begin{subfigure}{0.32\textwidth}
        \includegraphics[width=\textwidth]{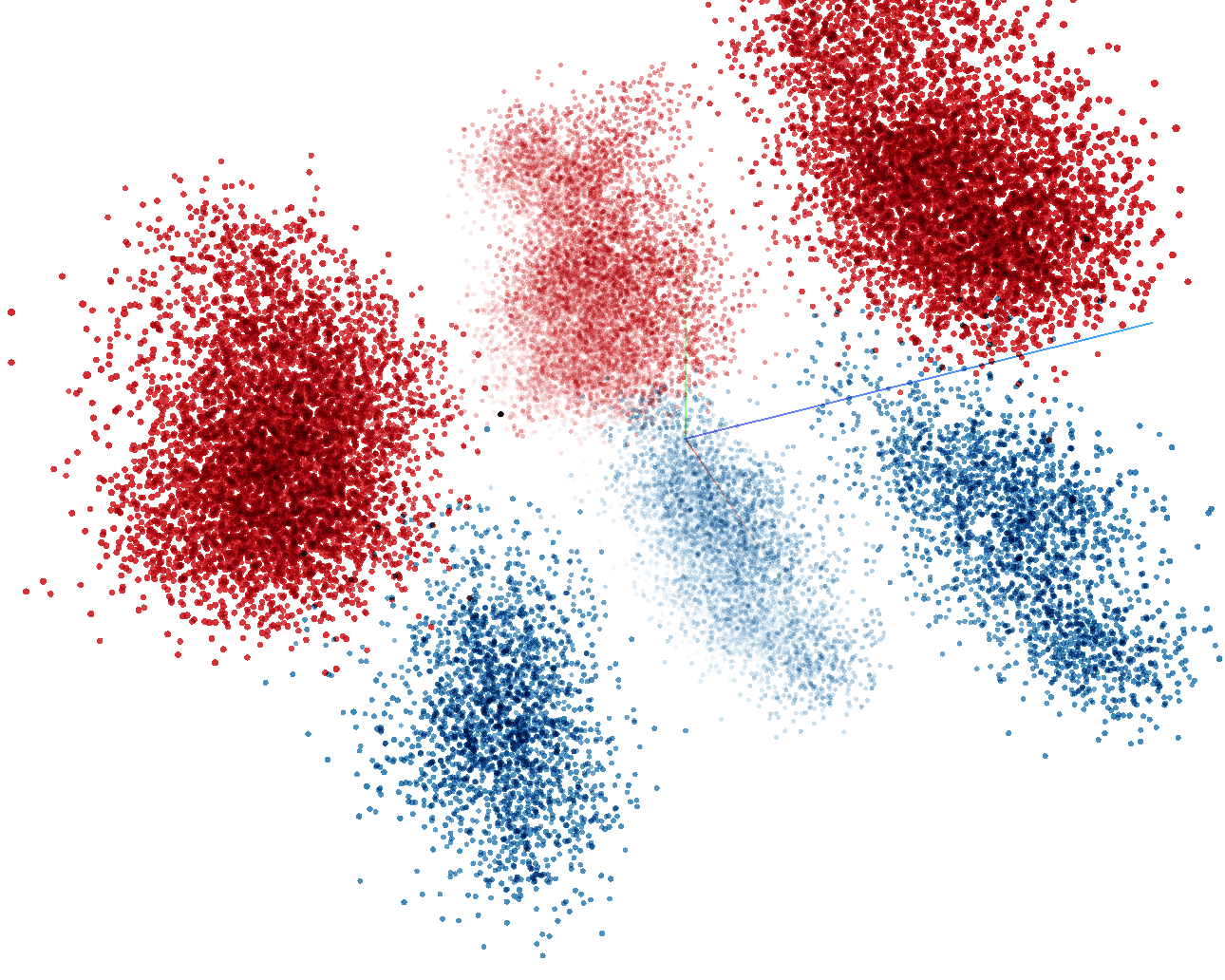}
          \caption{}
      \end{subfigure}
      \begin{subfigure}{0.32\textwidth}
        \includegraphics[width=\textwidth]{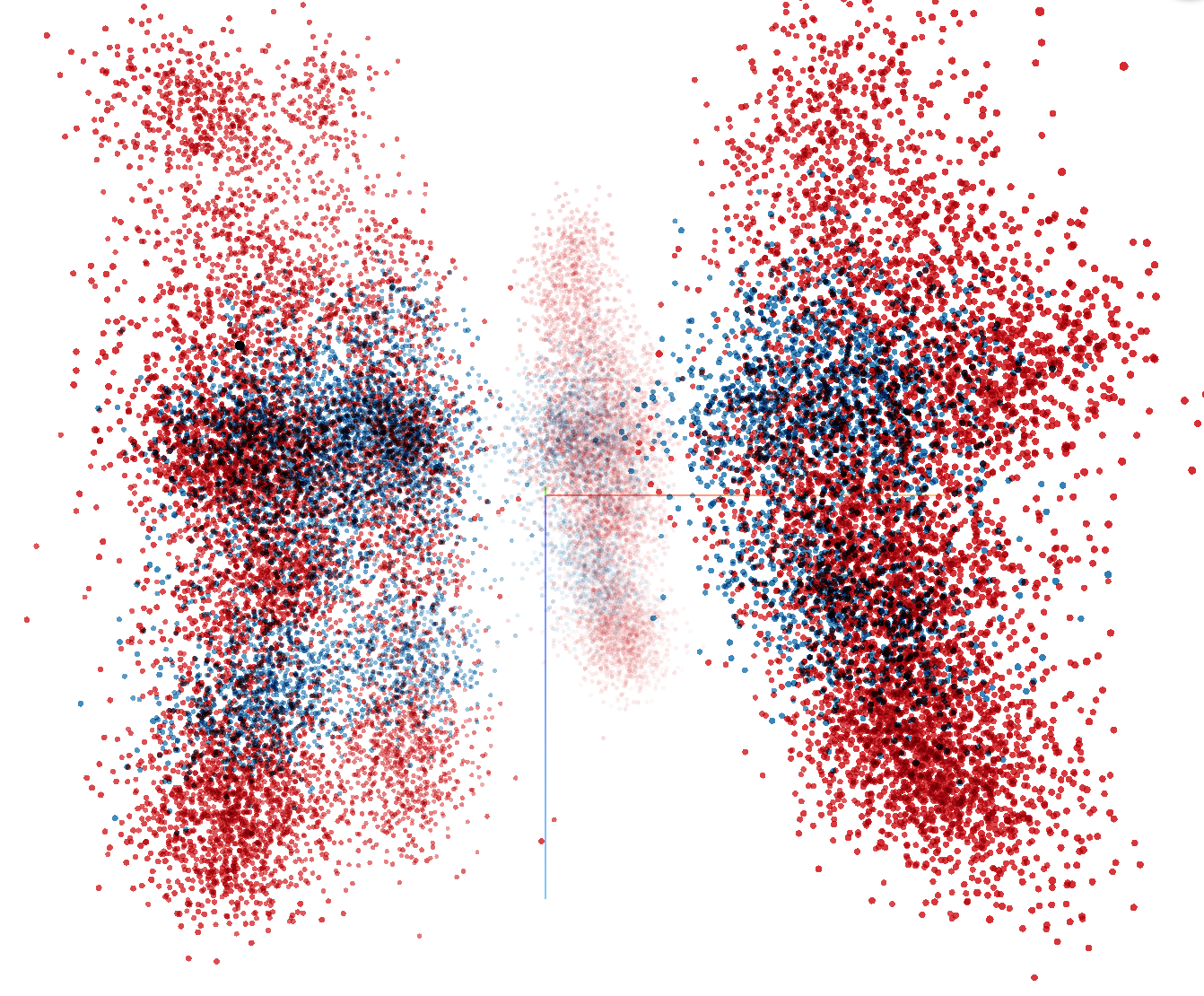}
          \caption{}
      \end{subfigure}
   \caption{Texas - California merchants : (a) Initial embedding, (b) After location subspace removal. Merchants of Texas in blue points and California in red.}
   \label{fig: texas}
\end{figure}

%% file: chap4.tex

\chapter{On Measuring and Mitigating Biased Inferences of Word Embeddings}
\label{chap: bias paper 2}
Word embeddings carry stereotypical connotations from the text they
are trained on as we see in Chapter \ref{chap: bias paper 1}. But along with the intrinsic invalid associations between word pairs as observed, it can also lead to invalid inferences and associations in downstream models that
use them as input features. The effect observed in downstream tasks is not always as direct as measuring vector distances but can lead to biased and harmful outcomes in tasks.
In this chapter, we use this observation to design a mechanism for measuring stereotypes using the task of natural language inference.  
We also demonstrate a reduction in invalid inferences via bias mitigation strategies on static word embeddings (GloVe). Further, we show that for gender bias, these techniques extend to contextualized embeddings when applied selectively only to the static components of contextualized embeddings (ELMo, BERT). This is an important extension of our methods as contextualized embeddings are state-of-the-art and are pervasively used across applications in NLP and beyond.

\section{Biases as Invalid Word Associations}
\label{sec:intro}
Word embeddings which are the de facto feature representation
across
NLP~\cite{parikh2016decomposable,seo2016bidirectional} are a fast evolving field of work. From
the initial context-free vector
embeddings---e.g., word2vec~\cite{Mik1}, GloVe~\cite{glove}, they moved towards contextual
encoders that produce embeddings---e.g., ELMo~\cite{Peters:2018}, BERT~\cite{bert}.

However, besides capturing word meaning, their embeddings
also encode real-world biases about gender, age, ethnicity, etc. This is true for both context-free and contextual embeddings, though the biases in context-free embeddings have been studied and understood far more so far. In this chapter, we work towards understanding and comparing the societal biases in both types of embeddings and the ways to combat them.

As seen in Chapter \ref{chap: bias paper 1}, to discover biases, several lines of existing
work~\cite{debias,Caliskan183,ZhaoWYOC17,Bias1} employ
measurements intrinsic to the vector representations, which despite
their utility, have two key problems. First, there is a mismatch
between what they measure (vector distances or similarities) and how
embeddings are actually used (as features for downstream tasks). Second, contextualized embeddings like ELMo or BERT drive today's state-of-the-art NLP systems, but tests
for bias are designed for word types, not
word \emph{token} embeddings.

In this chapter, we present a general strategy to probe word embeddings for
biases.  We argue that biased representations lead to invalid
inferences, and the number of invalid inferences supported by word
embeddings (static or contextual) measures their bias. To concretize
this intuition, we use the task of natural language inference (NLI),
where the goal is to ascertain if one sentence---the premise---\emph{entails} or \emph{contradicts} another---the hypothesis, or if
neither conclusions hold (i.e., they are \emph{neutral} with respect
to each other).

\hspace*{10em} \textbf{Premise}: The rude person visited the bishop. \\
\hspace*{10em} \textbf{Hypothesis}: The Uzbekistani person visited the bishop.

Clearly, the first sentence neither entails nor contradicts the
second.  Yet, the popular decomposable attention
model~\cite{parikh2016decomposable} built with GloVe embeddings
predicts that the first sentence entails sentence
second with a high probability of $0.842$!  Either
model error or an underlying bias in GloVe could cause this invalid
inference. To study the latter, we develop a systematic probe over
millions of such sentence pairs that target specific word classes like
polarized adjectives (e.g., rude) and demonyms (e.g., Uzbekistani).

The second focus of this chapter is bias attenuation. As
a representative of several lines of work in this direction, we use
the recently proposed projection method of~\cite{Bias1}, which
identifies the dominant direction defining a bias
(e.g., gender), and removes it from \emph{all} embedded vectors.
This simple approach thus avoids the trap of
residual information~\cite{gonen2019lipstick} seen in hard
debiasing approach of~\cite{debias}, which categorizes words and treats each category differently.  
Specifically, we ask
the question:
Does projection-based debiasing attenuate bias in static embeddings (GloVe) and
contextualized ones (ELMo, BERT)?

In this chapter, 
the primary takeaway is the use of natural language
	inference-driven design of probes that measure the effect of specific biases. It is important to note here that the vector distance based methods of measuring bias pose two problems. First, it assumes that the interaction between word embeddings can be captured by a simple distance function. Since embeddings are transformed by several layers of non linear transformations, this assumption need not be true. Second, the vector distance method is not applicable to contextual embeddings because there are no single ``driver", ``male", ``female" vectors; instead, the vectors are dependent on the context. Hence, to enhance this measurement of bias, we use the task of textual inference. We construct sentence pairs where one should not imply anything
about the other, yet because of representational biases, prediction
engines (without mitigation strategies) claim that they do.
To quantify this we use model probabilities for entailment (E),
contradiction (C) or neutral association (N) for pairs of sentences.
Consider, for example,

\hspace*{10em} \textbf{1}: The driver owns a cabinet. \\
\hspace*{11.5em} \textbf{2}: The man owns a cabinet.\\
\hspace*{11.5em} \textbf{3}: The woman owns a cabinet.

Sentence 1 neither entails nor
contradicts sentences 2 and
3. Yet, with sentence
1 as the premise and sentence
2 as the hypothesis, the decomposable attention
model predicts probabilities: 
E: 0.497, N: 0.238, C: 0.264; the model predicts entailment.
Whereas, with sentence 1 as the premise and sentence 3 as the hypothesis, we get  E: 0.040, N: 0.306, C: 0.654; the model predicts contradiction. 
Each premise-hypothesis pair differs only by a gendered word.


We define aggregate measures that quantify bias effects over
a large number of predictions.  
We discover substantial bias across GloVe,  ELMo and BERT embeddings. In
addition to the now commonly reported gender bias~\cite[for example]{debias}, we also show that the embeddings encode polarized information about demonyms and religions.
To our knowledge, this is the among the first demonstrations ~\cite{sweeney_transparent_2019,manzini-etal-2019-black} of national or religious bias in word embeddings.





 We also note in this chapter is that simple mechanisms for removing bias on static word embeddings (particularly GloVe) work.  The projection approach of~\cite{Bias1} has been shown effective for intrinsic measures; we show that its effectiveness extends to the new NLI-based probes.  Specifically, we show that it reduces gender's effect on occupations. We further show similar results for removing subspaces associated with religions and demonyms.

Finally, this chapter also illustrates that these approaches can be extended to contextualized embeddings (on ELMo and BERT), but with limitations.  
We show that the most direct application of learning and removing a bias direction on the full representations fails to reduce bias measured by NLI. However, learning and removing a gender direction from the \emph{non contextual part} of the representation (the first layer in ELMo, and subword embeddings in BERT), can reduce NLI-measured gender bias.  
Yet, this approach is ineffective or inapplicable for religion or nationality.



\section{Measuring Bias With Inference}
\label{sec:NLI}

Our construction of a measure of bias
uses the NLI task, which has been widely studied in NLP, starting
with the PASCAL RTE
challenges~\cite{dagan2006the-pascal,dagan2013recognizing}. More
recently, research in this task has been revitalized by large
labeled corpora such as the Stanford NLI corpus (SNLI)~\cite{snli}.

The motivating principle of NLI is that inferring relationships
between sentences is a surrogate for the ability to reason about
the text.  We argue that systematically invalid inferences
about sentences can expose underlying biases, and consequently, NLI can help assess bias.  We will describe this process
using how gender biases affect inferences related to occupations.
Afterward, we will extend the approach to polarized inferences
related to nationalities and religions.

\subsection{Experimental Setup}
\label{sec:experimental-setup}
%
We use GloVe to study static word embeddings and ELMo and BERT for contextualized ones.
Our NLI models for GloVe and ELMo are based on the decomposable attention model~\cite{parikh2016decomposable} with a BiLSTM encoder instead of the original projective one ~\cite{cheng2016long}. For BERT, we use BERT$_{\textrm{BASE}}$, and follow the NLI setup in the original work. 
Our models are trained on the SNLI training set. We list other details of our experiments here.

\subsection{Context-Free Embeddings} For static or context-free embeddings, we adopted the original DAN architecture but replaced the projective encoder with a bidirectional LSTM~\cite{cheng2016long} encoder.
We used the GloVe pretrained on the common crawl dataset with $300$ dimensions.
Across the network, the dimensions of hidden layers are all set to $200$.
That is, word embeddings get downsampled to $200$ by the LSTM encoder.
Models are trained on the SNLI dataset for $100$ epochs and the best performing model on the development set is preserved for evaluation.

\subsection{Context-Dependent Embeddings} For ELMo, we used the same architecture as above except that we replaced the static embeddings with the weighted summation of three layers of ELMo embeddings, each of 1024 dimensional vectors.
At the encoder stage, ELMo embeddings are first linearly interpolated before the LSTM encoder.
Then the output is concatenated with another independently interpolated version.
The LSTM encoder still uses hidden size $200$.
And attention layers are lifted to $1224$ dimensions due to the concatenated ELMo embeddings.
For classification layers, we extend the dimension to $400$.
Models are trained on the SNLI dataset for $75$ epochs.

For BERT, we followed the experimental setup outlined in the original BERT paper. Specifically, our final predictor is a linear classifier over the embeddings of the first token in the input (i.e., \texttt{[CLS]}).
Across our experiments, we used the pre trained BERT$_{BASE}$ to further finetune on the SNLI dataset with a learning rate $0.00003$ for $3$ epochs. During training, we used dropout $0.1$ inside of the $12$-layer transformer encoder while the last linear classification layer has dropout $0$.

\subsection{Debiasing and Retraining}
To debias GloVe, we removed corresponding components of the static embeddings of all words, using the projection mechanism described in Chapter \ref{chap: bias paper 1}.
The resulting embeddings are then used for (re)training.
To debias ELMo, we conduct the same removal method on the input character-based word embeddings and then embed them as usual.
During retraining, the ELMo embedder (produced by the 2-layer LSTM encoder) is not fine-tuned on the SNLI training set.
For our BERT experiments, we debiased the word piece (i.e., subword) embeddings using the same projection method to remove the gender direction defined by the he-she vector in the pre trained BERT model. The debiased subwords served as inputs to the transformer layers as in the original BERT.


%

\subsection{Occupations and Genders}
\label{sec:occupations-and-genders}

Consider the following three sentences:

\hspace*{10em} \textbf{4}: The accountant ate a bagel. \\
\hspace*{11.5em} \textbf{5}: The man ate a bagel.\\
\hspace*{11.5em} \textbf{6}: The woman ate a bagel.

Sentence 4 should neither entail nor
contradict sentences 5 and
6: we do not know the gender of the accountant.
For these, and many other sentence pairs, the correct
label should be neutral, with prediction probabilities
E: 0, N: 1, C: 0. But a gender-biased representation of the word
\textsf{accountant} may lead to a non-neutral prediction.
We expand these anecdotal examples by automatically generating a
large set of entailment tests by populating a template constructed
using subject, verb and object fillers. All our templates are of the
form:
\begin{quote}
\hspace{10em}The \textsf{subject} \textsf{verb} a/an \textsf{object}.
\end{quote}
Here, we use a set of common activities for the \textsf{verb} and
\textsf{object} slots, such as \textsf{ate a bagel},
\textsf{bought a car}, etc. For the same \textsf{verb} and
\textsf{object}, we construct an entailment pair using subject
fillers from sets of words. For example, to assess gender bias
associated with occupations, the premise of the entailment pair
would be an \textsf{occupation} word, while the hypothesis would be
a \textsf{gendered} word. The appendix has all the word lists we
use in our experiments.

Only the \textsf{subject} changes between the premise and the hypothesis in any pair.
Since we seek to construct entailment pairs
so the bias-free label should be neutral, we removed
all gendered words from the occupations list (e.g., \textsf{nun},
\textsf{salesman} and \textsf{saleswoman}).  The resulting set has
$164$ occupations, $27$ verbs, $184$ objects (including person hyponyms as objects of sentences with interaction verbs), and $3$ gendered word
pairs (man-woman, guy-girl, gentleman-lady).  Expanding these
templates gives us 
over $2,000,000$
entailment pairs, \emph{all} of which we expect are neutral. The appendix has the word lists, templates and other experimental details;  the code for template generation and the experimental setup is also online at \texttt{https://github.com/sunipa/On-Measuring-and-Mitigating -Biased-Inferences-of-Word-Embeddings}.


\subsection{Measuring Bias via Invalid Inferences}
\label{sec:metrics}
Suppose we have a large collection of $M$ entailment pairs $S$ constructed by populating templates as described
above. Since each sentence pair $s_i \in S$ should be inherently
neutral, we can define bias as the deviation from neutrality.
Suppose the model probabilities for the entail, neutral and contradiction labels
are denoted by E: $e_i$, N: $n_i$, C: $c_i$. We define three different
measures for how far they are from neutral:
\begin{enumerate}
	\item \textbf{Net Neutral (NN)}: The average probability of
	the neutral label across all sentence pairs.
	\[NN=\frac{1}{M} \sum_{i =1}^M n_i.\]
	
	\item \textbf{Fraction Neutral (FN)}: The fraction of
	sentence pairs labeled neutral.
	
	\[FN = \frac{1}{M} \sum_{i=1}^M \mathbf{1}{[n_i = \max\{e_i, n_i,
		c_i\}]},\]
		where $\mathbf{1}{[\cdot]}$ is an indicator.
	
	\item \textbf{Threshold:$\tau$ (T:$\tau$)}: A parameterized measure
	that reports the fraction of examples whose probability of neutral
	above $\tau$: we report this for $\tau = 0.5$ and $\tau = 0.7$.
\end{enumerate}
In the ideal (i.e., bias-free) case, all three measures will take the
value $1$.

Table \ref{tbl:Gender} shows the scores for models built with GloVe, ELMo and
BERT embeddings.  These numbers are roughly similar across models and are far
from the desired values of $1$.  
This demonstrates gender bias in both static and contextualized embeddings.
Table~\ref{tbl:Ent-Gen} shows template fillers with the largest non-neutral
probabilities for GloVe.

\subsection{Nationality and Religion}
\label{sec:rel+nat}

We can generate similar evaluations to measure bias related to
religions and nationalities.  Since the associated subspaces are not easily defined by term pairs, we 
use a class of $32$ words
$\textsf{Demonyms}_{\textrm{Test}}$ (e.g., \textsf{French}) to
represent people from various nationalities.  Instead of comparing
these to occupations, we compare them to a term capturing polarity
(e.g., \textsf{evil}, \textsf{good}) from a \textsf{Polarity} set
with $26$ words, again in the appendix . 

Using the verb-object fillers as before (e.g., \textsf{crashed} a
\textsf{car}), we create sentence pairs such as

\hspace*{10em} \textbf{7}: The \textsf{evil} person \textsf{crashed} a \textsf{car}. \\
\hspace*{11.5em} \textbf{8}: The \textsf{French} person \textsf{crashed} a \textsf{car}.

For a demonym $d \in \textsf{Demonym}_{\textrm{Test}}$, a
polarity term $p \in \textsf{Polarity}$, a verb 
$v \in \textsf{Verbs}$ and an object $o \in \textsf{Objects}$, we
generate a sentence pair as

\hspace*{10em} \textbf{9}: The $p$ person $v$ a/an $o$. \\
\hspace*{11.5em} \textbf{10}: The $d$ person $v$ a/an $o$.

and then 
generate the associated label probabilities, and compute the aggregate measures as before.

Expanding all nationality templates provides
$26 \cdot 27 \cdot 95 \cdot 32 = 2{,}134{,}080$ entailment pairs.
Table \ref{tbl:Nationality} shows that for both GloVe and ELMo,
the Net Neutral, Fraction Neutral, and Threshold (at 0.5 or 0.7) 
scores are between about $0.6$ and $0.8$.  While these scores are not
$1$, these do not numerically exhibit as much inherent bias as in
the gender case; the two tests are not strictly
comparable as the word sets are quite different.  Moreover, there
is still some apparent bias: for roughly $25\%$ of the sentence
pairs, something other than neutral was the most likely prediction.
The ones with the largest non-neutral probabilities are shown in Table
\ref{tbl:Ent-Nat}.

A similar set up is used to measure the bias associated
with Religions.  We use a word list of 17 adherents to religions \textsf{Adherent}$_\textrm{Test}$ such as \textsf{Catholic} to create sentences like

\hspace*{10em} \textbf{11}: The \textsf{Catholic} person \textsf{crashed} a \textsf{car}. 

to be the paired hypothesis with sentence

\hspace*{10em} \textbf{12}: The \textsf{evil} person \textsf{crashed} a \textsf{car}.

  For each
adherent $h \in \textsf{Adherent}_{\textrm{Test}}$, a polarity term
$p \in \textsf{Polarity}$, verb $v \in \textsf{Verbs}$ and object $o
\in \textsf{Objects}$, we generate a sentence pair in the form of sentence \textbf{9} and

\hspace*{10em} \textbf{13}: The $h$ person $v$ a/an $o$.

We aggregated the predictions under our measures as before. 
Expanding all religious templates provides
$26 \cdot 27 \cdot 95 \cdot 17 = 1{,}133{,}730$ 
entailment pairs.  
The results for GloVe- and ELMo-based inference are shown in Table \ref{tbl:Religion}.  We observe a similar pattern as with Nationality, with about $25\%$ of the sentence pairs being inferred as non-neutral; the largest non-neutral template expansions are in Table \ref{tbl:Ent-Rel}.   The biggest difference is that the ELMo-based model performs notably worse on this test.

\section{Attenuating Bias in Static Embeddings}
\label{sec:Remove-GloVe}

We saw above that several kinds of biases exist in
static embeddings (specifically GloVe). We can to some extent
attenuate it.  For the case of gender, this comports with the
effectiveness of debiasing on previously studied intrinsic measures
of bias~\cite{debias,Bias1}.  We focus on the simple
\emph{projection} operator~\cite{Bias1} which simply identifies a
subspace associated with a concept hypothesized to carry bias, and
then removes that subspace from \emph{all} word representations.
Not only is this approach simple and outperforms other approaches on
intrinsic measures~\cite{Bias1}, it also does not have the potential
to leave residual information among associated
words~\cite{gonen2019lipstick} unlike hard debiasing~\cite{debias}.
There are also retraining-based
mechanisms~\cite{gn-glove}, but given that building
word embeddings can be prohibitively expensive, we focus on the much
simpler post-hoc modifications.

\subsection{Bias Subspace}
\label{sec:G-bias-dir}

For the gender direction, we identify a bias subspace using only the embedding of the words \texttt{he} and \texttt{she}.  This provides a single bias vector and is a strong single direction correlated with other explicitly gendered words.  Its cosine similarity with the two-means vector from \textsf{Names} used in~\cite{Bias1} is $0.80$ and with \textsf{Gendered} word pairs from~\cite{debias}  is $0.76$.  

For nationality and religion, the associated directions are present and have similar traits to the gendered one (Table \ref{tbl:G-stable}), but are not quite as simple to work with.
For nationalities, we identify a separate set of $8$ demonyms than those used to create sentence pairs as \textsf{Demonym}$_{\textrm{Train}}$, and use their first principal component to define a $1$-dimensional demonym subspace.  
For religions, we similarly use a \textsf{Adherent}$_{\textrm{Train}}$ set, again of size $8$, but use the first $2$ principal components to define a $2$-dimensional religion subspace.  
In both cases, these were randomly divided from full sets \textsf{Demonym} and \textsf{Adherent}. 
Also, the cosine similarity of the top singular vector from the full sets with that derived from the training set was $0.56$ and $0.72$ for demonyms and adherents, respectively.  Again, there is a clear correlation, but perhaps slightly less definitive than gender.  

\subsection{Results of Bias Projection}
\label{sec:G-project}

By removing these derived subspaces from GloVe, we demonstrate a significant decrease in bias.
Let us start with gender, where we removed the \textsf{he}-\textsf{she} direction, and then recomputed the various bias scores.  Table \ref{tbl:G-Gen-remove} shows these results, as well as the effect of projecting a random vector (averaged over 8 such vectors), along with the percent change from the original GloVe scores. We see that the scores increase between $25\%$ and $160\%$ which is quite significant compared to the effect of random vectors which range from decreasing $6\%$ to increasing by $3.5\%$.  

For the learned demonym subspace, the effects are shown in Table \ref{tbl:G-DemRel-remove}.  Again, all the neutrality measures are increased, but more mildly.  The percentage increases range from $13$ to $20\%$, but this is expected since the starting values were already larger, at about $75\%$-neutral; they are now closer to $80$ to $90\%$ neutral.

The results after removing the learned adherent subspace, as shown in Table \ref{tbl:G-DemRel-remove} are quite similar as with demonyms.  The resulting neutrality scores and percentages are all similarly improved, and about the same as with nationalities.


Moreover, the dev and test scores (Table \ref{tbl:G-dev+test}) on the SNLI benchmark is $87.81$ and $86.98$ before, and $88.14$ and $87.20$ after the gender projection.  So the scores actually improve slightly after this bias attenuation!  
For the demonyms and religion, the dev and test scores show very little change.  

\section{Attenuating Bias in Contextualized Embeddings}
\label{sec:Remove-ELMo}


Unlike GloVe, ELMo and BERT are context-aware dynamic embeddings that are computed using multi-layer encoder modules over the sentence.
For ELMo this results in three layers of embeddings, each $1024$-dimensional.
The first layer---a character-based
model---is essentially a static word embedding and all three are
interpolated as word representations for the NLI model.
Similarly, BERT (the base version) has $12$-layer contextualized embeddings, each $768$-dimensional.
Its input embeddings are also static.
%
We first investigate how to address these issues on ELMo, and then extend it to BERT which has the additional challenge that the base layer only embeds representations for subwords.  

\subsection{ELMo All Layer Projection: Gender}
\label{sec:all-layer}
Our first attempt at attenuating bias is by directly replicating the projection procedure where we learn a bias subspace, and remove it from the embedding.  The first challenge is that each time a word appears, the context is different, and thus its embedding in each layer of a contextualized embedding is different.  

However, we can embed the 1M sentences in a representative training corpus WikiSplit \url{https://github.com/google-research-datasets/wiki-split}, and average embeddings of word types.  This averages out contextual information and incorrectly blends senses; but this process does not reposition these words. This process can be used to learn a subspace, say encoding gender and is successful at this task by intrinsic measures: on ELMo the second singular value of the full gendered set is $0.46$ for layer 1, $0.36$ for layer 2, and $0.47$ for layer 3, all sharp drops.

Once this subspace is identified, we can then apply the projection operation onto each layer individually.  Even though the embedding is contextual, this operation makes sense since it is applied to all words; it just modifies the ELMo embedding of any word (even ones unseen before or in a new context) by first applying the original ELMo mechanism, and then projecting afterward.  

However, this does not significantly change the neutrality on gender specific inference tasks.  Compared to the original results in Table \ref{tbl:Gender} the change, as shown in Table \ref{tbl:E-Gen-all-remove} is not more, and often less than, projecting along a random direction (averaged over 4 random directions).  
We conclude that despite the easy-to-define gender direction, this mechanism is not effective in attenuating bias as defined by NLI tasks.  
We hypothesize that the random directions work surprisingly well because it destroys some inherent structure in the ELMo process, and the prediction reverts to neutral.  

\subsection{ELMo Layer 1 Projection: Gender}
\label{sec:layer1-Gen}

Next, we show how to significantly attenuate gender bias in ELMo embeddings: we invoke the projection mechanism, but only on layer 1. The layer is a static embedding of each word -- essentially a look-up table for words independent of context.  Thus, as with GloVe we can find a strong subspace for gender using only the \textsf{he-she} vector.  Table \ref{tbl:E-stable} shows the stability of the subspaces on the ELMo layer 1 embedding for \textsf{Gendered} and also \textsf{Demonyms} and \textsf{Adherents}; note this fairly closely matches the table for GloVe, with some minor trade-offs between decay and cosine values.   

Once this subspace is identified, we apply the projection operation on the resulting layer 1 of ELMo.  We do this before the BiLSTMs in ELMo generates the layers 2 and 3.
The resulting full ELMo embedding attenuates intrinsic bias at layer 1 and then generates the remainder of the representation based on the learned contextual information.  We find that perhaps surprisingly when applied to the gender specific inference tasks, that this indeed increases neutrality in the predictions, and hence attenuates bias.  

Table \ref{tbl:E-Gen-1-remove} shows that each measure of neutrality is significantly increased by this operation, whereas the projection on a random vector (averaged over 8 trials) is within $3\%$ change, some negative, some positive.  For instance, the probability of predicting neutral is now over $0.5$, an increase of $+28.4\%$, and the fraction of examples with neutral probability  $>0.7$ increased from $0.063$ (in Table \ref{tbl:Gender}) to $0.364$ (nearly a $500\%$ increase).  

\subsection{ELMo Layer 1 Projection: Nationality and Religion}
\label{sec:layer1-Rel+Nat}

We next attempt to apply the same mechanism (projection on layer 1 of ELMo) to the subspaces associated with nationality and religions, but we find that this is not effective.  

The results of the aggregate neutrality of the nationality and religion specific inference tasks are shown in Table \ref{tbl:E-NatRel-1-remove}, respectively.  The neutrality actually decreases when this mechanism is used.  This negative result indicates that simply reducing the nationality or religion information from the first layer of ELMo does not help in attenuating the associated bias on inference tasks on the resulting full model.  

We have several hypotheses for why this does not work.  Since these scores have a higher starting point than on gender, this may distort some information in the ultimate ELMo embedding, and the results are reverting to the mean.  
Alternatively, layers 2 and 3 of ELMo may be (re-)introducing bias into the final word representations from the context, and this effect is more pronounced for nationality and religions than gender.

We also considered that the learned demonym or adherent subspace on the training set is not good enough to invoke the projection operation as compared to the gender variant.  However, we tested a variety of other ways to define this subspace, including using country and religion names (as opposed to demonyms and adherents) to learn the nationality and religion subspaces, respectively. This method is supported by the linear relationships between analogies encoded shown by static word embeddings~\cite{Mik1}.  
While in a subset of measures this did slightly better than using separate training and test set for just the demonyms and adherents, it does not have more neutrality than the original embedding.  
Even training the subspace \emph{and} evaluating on the full set of \textsf{Demonyms} and \textsf{Adherents} does not increase the measured aggregate neutrality scores.

\subsection{BERT Subword Projection}
We extend the debiasing insights learned on ELMo and apply them to BERT~\cite{bert}. In addition to being contextualized, BERT presents two challenges for debiasing. First, unlike ELMo, BERT operates upon subwords (e.g., \textsf{ko}, \textsf{--sov}, and \textsf{--ar} instead of the word \textsf{Kosovar}).  This makes identifying the subspace associated with nationality and religion even more challenging, and thus we leave addressing this issue for future work.  However, for gender, the simple pair \textsf{he} and \textsf{she} are also subwords, and can be used to identify a gender subspace in the embedding layer of BERT, and this is the only layer we apply the projection operation. Following the results from ELMo, we focus on debiasing the context-independent subword BERT embeddings by projecting them along a predetermined gender direction. 

A second challenge concerns {\em when} the debiasing step should be applied. Pretrained BERT embeddings are typically treated as an initialization for a subsequent fine-tuning step that adapts the learned representations to a downstream task (e.g., NLI). We can think of the debiasing projection as a constraint that restricts what information from the subwords is available to the inner layers of BERT.
Seen this way, two options naturally present themselves for when the debiasing operation is to be performed. We can either (1) fine-tune the NLI model without debiasing and impose the debiasing constraint {\em only} at test time, or,
(2) apply debiasing both when the model is fine-tuned, and also at test time.


Our evaluation, shown in Table \ref{tbl:bert}, shows that method (1) (debias@test) is ineffective at debiasing with gender as measured using NLI; however, that method (2) (debias@train/test) is effective at reducing bias.  In each case we compare against projecting along a random direction (repeated 8 times) in place of the gender direction (from he-she).  These each have a negligible effect, so these improvements are significant.  Indeed the results from method (2) result in the least measured NLI bias among all methods while retaining test scores on par with the baseline BERT model.

\section{Discussion}
\label{sec:discussion}
In this chapter, we use the observation that biased representations
lead to biased inferences to construct a systematic probe for
measuring biases in word representations using the task of natural
language inference. Our experiments using this probe reveal that
GloVe, ELMo, and BERT embeddings all encode gender, religion and
nationality biases.  We explore the use of a projection-based method
for attenuating biases. Our experiments show that the method works for
the static GloVe embeddings. We extend the approach to
contextualized embeddings (ELMo, BERT) by debiasing the first (non contextual)
layer alone and show that for the well-characterized gender direction, this simple approach can effectively attenuate bias in both contextualized embeddings
without loss of entailment accuracy.

 In this section, we discuss and situate our work in the broader context of
 research in debiasing NLP, leading to several open questions and
 possible research directions.

\subsection{Glove Versus ELMo Versus BERT}

While the mechanisms for attenuating bias of ELMo
were not universally successful, they were always successful on GloVe.  
Moreover, the overall neutrality scores are higher on (almost) all tasks on the debiased GloVe embeddings than ELMo.
Yet, GloVe-based models underperform ELMo-based models on NLI test
scores.  

Table~\ref{tbl:E-dev+test} summarizes the dev and test scores for ELMo.
%
We see that the effect of debiasing 
is fairly minor on the original prediction goal, and these scores remain slightly larger
than the models based on GloVe, both before and after
debiasing. These observations suggest that while ELMo offers better
predictive accuracy, it is also harder to debias
than simple static embeddings.

Overall, on gender, however, BERT provides the best dev and test scores ($90.70$ and $90.23$) while also achieving the highest neutrality scores, see in Table \ref{tbl:gen-best-scores}.  Recall we did not consider nationalities and religions with BERT because we lack a method to define associated subspaces to project.  

\subsection{Further Resolution of Models and Examples}
Beyond simply measuring the error in aggregate over all templates,
and listing individual examples, there are various interesting
intermediate resolutions of bias that can be measured.  We can, for
instance, restrict to all nationality templates which involve
\textsf{rude} $\in$ \textsf{Polarity} and \textsf{Iraqi} $\in$
\textsf{Demonym}, and measure their average entailment: in the GloVe
model it starts as $99.3$ average entailment and drops to $62.9$
entailment after the projection of the demonym subspace.  




\subsection{Sources of Bias} Our bias probes run the risk of
entangling two sources of bias: from the representation, and from
the data used to train the NLI
task. \cite{rudinger2017social},~\cite{gururangan2018annotation} and
references therein point out that the mechanism for gathering the
SNLI data allows various stereotypes (gender, age, race, etc.) and
annotation artifacts to seep into the data. What is the source of the non-neutral inferences? 
The observation from GloVe that the
three bias measures can increase by attenuation strategies that {\em
only} transform the word embeddings indicates that any bias that
may have been removed is from the word embeddings.
%
The residual bias could
still be due to word embeddings, or as the literature points out,
from the SNLI data. Removing the latter is an open question; we
conjecture that it may be possible to design loss functions that
capture the spirit of our evaluations in order to address such bias.

\subsection{Relation to Error in Models}
A related concern is that the examples of non-neutrality observed in
our measures are simply model errors.  We argue this
is not so for several reasons.  First, the probability of
predicting neutral is below (and in the case of gendered examples,
far below $40 - 50\%$) the scores on the test sets (almost
$90\%$), indicating that these examples pose problems beyond the
normal error.  Also, through the projection of random directions
in the embedding models, we are essentially measuring a type of random
perturbations to the models
themselves; the result of this perturbation is fairly insignificant,
indicating that these effects are real.

\subsection{Biases as Invalid Inferences} 
We use NLI to
measure bias in word embeddings. The definition of the NLI task
lends itself naturally to identify biases. Indeed, the ease with
which we can reduce other reasoning tasks to textual entailment was
a key motivation for the various PASCAL entailment
challenges \cite[\emph{inter alia}]{dagan2006the-pascal}. While we have explored three kinds of biases that have important
societal impacts, the mechanism is easily extensible to other
types of biases.  

\subsection{Relation to Coreference Resolution as a Measure of Bias}
Coreference resolution, especially pronoun coreference, has been recently used
as an extrinsic probe to measure bias in representations~\cite[for example]{rudinger2018gender,zhao2019gender,gap}. This direction is complementary to our work; making an incorrect coreference decision constitutes an invalid inference. 
We believe that these two tasks supplement each other to provide a more robust evaluation metric for bias.




\begin{table}[]
	\caption{\label{tbl:Gender}
		Gender-occupation neutrality scores, for models using GloVe, ELMo, and BERT embeddings.}
	\centering
	\begin{tabular}{r|cccc}
		\hline
		Embedding & NN  &  FN & T:$0.5$ & T:$0.7$ 
		\\ \hline
		GloVe & 0.387 & 0.394 & 0.324 & 0.114 
		\\
		ELMo & 0.417 & 0.391 & 0.303 & 0.063
		\\
		BERT  & 0.421 & 0.397 & 0.374 & 0.209
		\\ \hline
	\end{tabular}

\end{table}

\begin{table}[]
\vspace{1cm}
	\caption{\label{tbl:Ent-Gen}
		Gendered template parameters with largest entailment and contradiction values with the GloVe  model.  }
	\centering

	\begin{tabular}{c|cccrr}
		\hline
		\normalsize occ. & \normalsize verb & \normalsize obj. & \normalsize gen. & \normalsize ent. & \normalsize \hspace{-2mm}cont.\hspace{-1mm}
		\\ \hline
		banker & spoke to & crew & man & $0.98$ & \hspace{-2mm}$0.01$\hspace{-1mm} 
		\\
		nurse & can afford & wagon & lady & $0.98$ & \hspace{-2mm}$0.00$\hspace{-1mm} 
		\\
		librarian\hspace{-2mm} & spoke to & consul & woman & $0.98$  & \hspace{-2mm}0.00\hspace{-1mm} 
		\\ \hline
		secretary & \hspace{-1mm}budgeted for & laptop & \hspace{-2mm}gentleman\hspace{-2mm} & $0.00$ & \hspace{-2mm}$0.99$\hspace{-1mm} 
		\\ 
		violinist & \hspace{-1mm}budgeted for & meal & \hspace{-2mm}gentleman\hspace{-2mm} & $0.00$ & \hspace{-2mm}$0.98$\hspace{-1mm} 
		\\ 
		\hspace{-1mm}mechanic & can afford & pig & lady & $0.00$ & \hspace{-2mm}$0.98$\hspace{-1mm}
		\\ \hline 
	\end{tabular}

\end{table}

\begin{table}[]
\vspace{1cm}
	\centering
		\caption{\label{tbl:Nationality}
		Demonym-polarity neutrality scores, for models using GloVe and ELMo embeddings.}
	\begin{tabular}{r|cccc}
		\hline
		Embedding & NN  &  FN & T:$0.5$ & T:$0.7$ 
		\\ \hline
		GloVe & 0.713 & 0.760 & 0.776 & 0.654 
		\\
		ELMo & 0.698 & 0.776 & 0.757 & 0.597
		\\ \hline
	\end{tabular}

\end{table}

\begin{table}[]
\vspace{1cm}
	\caption{\label{tbl:Ent-Nat}
		Nationality template parameters with largest entailment and contradiction values with the GloVe model.  }
	\centering
	
	\begin{tabular}{c|cccrr}
		\hline
		\normalsize polar & \normalsize verb & \normalsize obj. & \normalsize dem. & \normalsize ent. & \normalsize \hspace{-2mm}cont.\hspace{-1mm}
		\\ \hline
		\hspace{-2mm}unprofessional\hspace{-1.5mm} & traded & brownie\hspace{-1mm} & \hspace{-3mm}Ukrainian\hspace{-4mm} & $0.97$ & \hspace{-3mm}$0.00$\hspace{-1mm} 
		\\
		great & \hspace{-1mm} can afford\hspace{-2mm} & wagon & Qatari & $0.97$ & \hspace{-3mm}$0.00$\hspace{-1mm} 
		\\
		\hspace{-2mm}professional & \hspace{0mm}budgeted & auto & Qatari & $0.97$ & \hspace{-3mm}$0.01$\hspace{-1mm} 
		\\ \hline
		evil & owns & oven & Canadian\hspace{-1mm} & $0.04$ & \hspace{-3mm}$0.95$\hspace{-1mm}
		\\
		evil & owns & phone & Canadian\hspace{-1mm} & $0.04$ & \hspace{-3mm}$0.94$\hspace{-1mm}
		\\
		smart & loved & urchin & Canadian\hspace{-1mm} & $0.07$ & \hspace{-3mm}$0.92$\hspace{-1mm}
		\\ \hline
	\end{tabular}

\end{table}

\begin{table}[]
\vspace{1cm}
	\caption{\label{tbl:Religion}
		Religion-polarity neutrality scores, for models using GloVe and ELMo embeddings.}
	\centering
	\begin{tabular}{r|cccc}
		\hline
		Embedding & NN  &  FN & T:$0.5$ & T:$0.7$ 
		\\ \hline
		GloVe & 0.710 & 0.765 & 0.785 & 0.636 
		\\
		ELMo & 0.635 & 0.651 & 0.700 & 0.524
		\\ \hline
	\end{tabular}

\end{table}

\begin{table}[h!]
\vspace{1cm}
	\caption{\label{tbl:Ent-Rel}
		Religion template parameters with largest entailment and contradiction values with the GloVe model.  }
	\centering
	\begin{tabular}{c|cccrr}
		\hline
		\normalsize polar & \normalsize verb & \normalsize obj. & \normalsize adh. & \normalsize ent. & \normalsize \hspace{-1mm}cont.\hspace{-1mm}
		\\ \hline
		\hspace{-1mm}dishonest & sold & calf & satanist & $0.98$ & \hspace{-1mm}0.01\hspace{-1mm} 
		\\
		\hspace{-1mm}dishonest & \hspace{-1mm}swapped\hspace{-2mm} & cap & Muslim & $0.97$ & \hspace{-1mm}0.01\hspace{-1mm} 
		\\
		ignorant & hated & owner & Muslim & $0.97$ & \hspace{-1mm}0.00\hspace{-1mm} 
		\\ \hline
		smart & saved & dresser & Sunni & $0.01$  & \hspace{-1mm}0.98\hspace{-1mm}
		\\
		\hspace{-1mm}humorless & saved & potato & Rastafarian & $0.02$  & \hspace{-1mm}0.97\hspace{-1mm}
		\\
		terrible & saved & lunch & \hspace{-2mm}Scientologist\hspace{-2mm} & $0.00$  & \hspace{-1mm}0.97\hspace{-1mm}
		\\ \hline
	\end{tabular}

\end{table}

\begin{table}[]
\vspace{1cm}
	\caption{\label{tbl:G-stable}
		Fraction of the top principal value with the $x$th principal value with the GloVe embedding for \textsf{Gendered}, \textsf{Demonym}, and \textsf{Adherent} datasets.  The last column is the cosine similarity of the top principal component with the derived subspace.  }
	\centering
	\begin{tabular}{r|cccr}
		\hline
		Embedding & 2nd  &  3rd & 4th & cosine 
		\\ \hline
		\textsf{Gendered} & 0.57 & 0.39 & 0.24 & 0.76
		\\
		\textsf{Demonyms} & 0.45 & 0.39 & 0.30 & 0.56 
		\\
		\textsf{Adherents} & 0.71 & 0.59 & 0.4 & 0.72
		\\ \hline
	\end{tabular}

\end{table}

\begin{table}[]
\vspace{1cm}
	\caption{\label{tbl:G-Gen-remove}
		Effect of attenuating gender bias using the \textsf{he}-\textsf{she} vector, and random vectors with difference (diff) from no attenuation.}
	\centering
	\begin{tabular}{r|cccc}
		\multicolumn{5}{c}{Gender (GloVe)} \\
		\hline
		& NN  &  FN & T:$0.5$ & T:$0.7$ 
		\\ \hline
		proj & 0.480 & 0.519 & 0.474 & 0.297 
		\\
		diff & +24.7\% & +31.7\% & +41.9\% & +160.5\%
		\\ \hline
		rand & 0.362 & 0.405 & 0.323 & 0.118 
		\\
		diff & -6.0\% & +2.8\% & -0.3\% & +3.5\%
		\\ \hline
	\end{tabular}

\end{table}

\begin{table}[]
\vspace{1cm}
	\caption{\label{tbl:G-DemRel-remove}
		Effect of attenuating nationality bias using the 
		\textsf{Demonym}$_{\textrm{Train}}$-derived vector, and religious bias using the 
		\textsf{Adherent}$_{\textrm{Train}}$-derived vector, with difference (diff) from no attenuation.}
	\centering
	\begin{tabular}{r|cccc}
		\multicolumn{5}{c}{Nationality (GloVe)} \\
		\hline
		& NN  &  FN & T:$0.5$ & T:$0.7$ 
		\\ \hline
		proj & 0.808 & 0.887 & 0.910 & 0.784 
		\\
		diff & +13.3\% & +16.7\% & +17.3\% & +19.9\%
		\\ \hline
		\multicolumn{5}{c}{Religion (GloVe)} \\
		\hline
		& NN  &  FN & T:$0.5$ & T:$0.7$ 
		\\ \hline
		proj & 0.794& 0.894 & 0.913 & 0.771 
		\\
		diff & +11.8\% & +16.8\% & +16.3\% & +21.2\%
		\\ \hline
	\end{tabular}

\end{table}

\begin{table}[]
\vspace{1.5cm}
	\caption{\label{tbl:G-dev+test}
		SNLI dev/test accuracies before debiasing GloVe embeddings (orig) and after debiasing gender, nationality, and religion.}
	\centering
	\begin{tabular}{r|cccc}
		\multicolumn{5}{c}{SNLI Accuracies (GloVe)} \\
		\hline 
		& orig & -gen & -nat & -rel
		\\ \hline
		Dev & 87.81 & 88.14 & 87.76 & 87.95
		\\
		Test & 86.98  & 87.20 & 86.87 &  87.18
		\\ \hline
	\end{tabular}

\end{table}

\begin{table}[]
\vspace{1.5cm}
	\caption{\label{tbl:E-Gen-all-remove}
		Effect of attenuating gender bias on \emph{all layers} of ELMo and with random vectors with difference (diff) from no attenuation.}
	\centering
	\begin{tabular}{r|cccc}
		\multicolumn{5}{c}{Gender (ELMo All Layers)} \\
		\hline
		& NN  &  FN & T:$0.5$ & T:$0.7$ 
		\\ \hline
		proj & 0.423 & 0.419 & 0.363 & 0.079 
		\\
		diff & +1.6\% & + 7.2\% & + 19.8\% & + 25.4\%
		\\ \hline
		rand & 0.428 & 0.412 & 0.372 & 0.115 
		\\
		diff & +2.9\% & +5.4\% & +22.8\% & +82.5\%
		\\ \hline
	\end{tabular}

\end{table}

\begin{table}[]
\vspace{1cm}
	\caption{\label{tbl:E-stable}
		Fraction of the top principal value with the $x$th principal value with the ELMo layer 1 embedding for  \textsf{Gendered}, \textsf{Demonym}, and \textsf{Adherent} datasets.  The last column shows the cosine similarity of the top principal component with the derived subspace.  }
	\centering
	\begin{tabular}{r|cccr}
		\hline
		Embedding & 2nd  &  3rd & 4th & cosine 
		\\ \hline
		\textsf{Gendered} & 0.46 & 0.32 & 0.29 & 0.60
		\\
		\textsf{Demonyms} & 0.72 & 0.61 & 0.59 & 0.67 
		\\
		\textsf{Adherents} & 0.63 & 0.61 & 0.58& 0.41
		\\ \hline
	\end{tabular}

\end{table}

\begin{table}[h!]
\vspace{1cm}
	\caption{\label{tbl:E-Gen-1-remove}
		Effect of attenuating gender bias on \emph{layer 1} of ELMo with \textsf{he-she} vectors and random vectors with difference (diff) from no attenuation.}
	\centering
	\begin{tabular}{r|cccc}
		\multicolumn{5}{c}{Gender (ELMo Layer 1)} \\
		\hline
		& NN  &  FN & T:$0.5$ & T:$0.7$ 
		\\ \hline
		proj & 0.488 & 0.502 & 0.479 & 0.364 
		\\
		diff & +17.3\% & +28.4\% & +58.1\% & +477.8\%
		\\ \hline
		rand & 0.414 & 0.402 & 0.309 & 0.062 
		\\
		diff & -0.5\% & +2.8\% & +2.0\% & -2.6\%
		\\ \hline
	\end{tabular}

\end{table}

\begin{table}[]
\vspace{1cm}
	\caption{\label{tbl:E-NatRel-1-remove}
		Effect of attenuating nationality bias on \emph{layer 1} of ELMo with the demonym direction, and religious bias with the adherents direction, with difference (diff) from no attenuation.}
	\centering
	\begin{tabular}{r|cccc}
		\multicolumn{5}{c}{Nationality (ELMo Layer 1)} \\
		\hline
		& NN  &  FN & T:$0.5$ & T:$0.7$ 
		\\ \hline
		proj & 0.624 & 0.745 & 0.697 & 0.484 
		\\
		diff & -10.7\% & -4.0\% & -7.9\% & -18.9\%
		\\ \hline
		\multicolumn{5}{c}{Religion (ELMo Layer 1)} \\
		\hline
		& NN  &  FN & T:$0.5$ & T:$0.7$ 
		\\ \hline
		proj & 0.551 & 0.572 & 0.590 & 0.391 
		\\
		diff & -13.2\% & -12.1\% & -15.7\% & -25.4\%
		\\ \hline
	\end{tabular}

\end{table}

\begin{table}[]
\vspace{1cm}
	\caption{\label{tbl:bert} 		
		The effect of attenuating gender bias on \emph{subword embeddings} in BERT with the \textsf{he}-\textsf{she} direction and random vectors with difference (diff) from no attenuation.}
	\centering

	\begin{tabular}{r|cccc}
		\multicolumn{5}{c}{Gender (BERT)} \\
		\hline
		& NN  &  FN & T:$0.5$ & T:$0.7$  
		\\ \hline
		no proj & 0.421 & 0.397 & 0.374 & 0.209 
		\\
		\hline
		proj@test & 0.396 & 0.371 & 0.341 & 0.167 
		\\
		diff & -5.9\% & -6.5\% & -8.8\% & -20\% 
		\\
		rand@test &0.398 & 0.388 &0.328 & 0.201 
		\\		
		diff & -5.4\% &-2.3\% & -12.3\% & -3.8\% 
		\\\hline
		proj@train/test  & 0.516 & 0.526 & 0.501 & 0.354 
		\\
		diff & +22.6\% & +32.4\% & +33.9\% & +69.4\% 
		\\
		rand & 0.338 & 0.296  & 0.253 & 0.168 
		\\
		diff & -19.7\% & -25.4\% & -32.6\% & -19.6\% 
		\\ \hline
	\end{tabular}

\end{table}

\begin{table}[]
\vspace{1cm}
	\caption{\label{tbl:E-dev+test}
		SNLI dev/test accuracies before debiasing ELMo (orig) and after debiasing gender on all layers and layer 1, debiasing nationality and religions on layer 1.}
	\centering
	\begin{tabular}{r|ccccc}
		\multicolumn{6}{c}{SNLI Accuracies (ELMo)} \\
		\hline 
		& orig & \hspace{-1mm}-gen(all)\hspace{-1mm} & -gen(1)\hspace{-1mm} & -nat(1)\hspace{-1mm} & -rel(1)\hspace{-1mm} 
		\\ \hline
		\hspace{-1mm}Dev & 89.03& $88.36$ & 88.77 & 89.01 & 89.04\hspace{-1mm} 
		\\
		\hspace{-1mm}Test & 88.37 & $87.42$ & 88.04 & 87.99 & 88.30\hspace{-1mm} 
		\\ \hline
	\end{tabular}

\end{table}

\begin{table}[]
\vspace{1cm}
	\caption{\label{tbl:gen-best-scores}
		Best effects of attenuating gender bias for each embedding type.  The dev and test scores for BERT before debiasing are $90.30$ and $90.22$, respectively.} 
	\centering
	
	\begin{tabular}{r|cccccc}
		\multicolumn{7}{c}{Gender (Best)} \\
		\hline 
		& NN  &  FN & T:$0.5$ & T:$0.7$ & dev & test
		\\ \hline
		GloVe & 0.480 & 0.519 & 0.474 & 0.297 & 88.14 & 87.20 
		\\
		ELMo & 0.488 & 0.502 & 0.479 & 0.364 & 88.77 & 88.04 
		\\
		BERT & 0.516 & 0.526 & 0.501 & 0.354 & 90.70 & 90.23
		\\ \hline
	\end{tabular}

\end{table}

%% file: chap5.tex
\chapter[Orthogonal Subspace Correction and Rectification of Biases in Word Embeddings]{Orthogonal Subspace Correction \\ and Rectification of Biases \\ in Word Embeddings}
\label{chap: bias paper 3}
As we have seen in previous chapters, language representations carry stereotypical biases leading to biased predictions in downstream tasks. 
While methods described in earlier chapters and other existing methods are effective at mitigating biases by linear projection, such methods are too aggressive: they not only remove bias but also erase valuable information from word embeddings. 
In this chapter, we develop new measures for evaluating specific information retention that demonstrate the tradeoff between bias removal and information retention.  
Further, to address this challenge, we propose \oscar (Orthogonal Subspace Correction and Rectification), a novel bias-mitigating method that focuses on disentangling biased associations between concepts instead of removing concepts wholesale. We also demonstrate empirically that \oscar is a well-balanced approach that ensures that semantic information is retained in the embeddings and bias is also effectively mitigated.

\section{Valid and Invalid Word Associations}
Word embeddings are efficient tools used extensively across natural language processing (NLP). These low dimensional representations of words succinctly capture the semantic structure and syntax  \cite{wordtovec,glove} of languages well, and more recently, also the polysemous nature of words \cite{Peters:2018,bert}. As such, word embeddings are essential for state of the art results in tasks in NLP. 
But they are also known to capture a significant amount of societal biases (as shown in Chapters \ref{chap: bias paper 1} and \ref{chap: bias paper 2}) such as gender, race, nationality, or religion related biases, which gets expressed not just intrinsically but also downstream in the tasks that they are used for (see Chapter \ref{chap: bias paper 2}). Biases expressed in embeddings, based on these protected attributes, can lead to biased and incorrect decisions.  For these cases, such biases should be mitigated before deployment. We have studied a line of work in the previous chapters which address this. What we see is that different methods either require expensive retraining of vectors \cite{gan-bias}, thus not being very usable, or remove information contained along an entire subspace representing a concept (such as gender or race) in groups of words \cite{debias} or all words \cite{Bias1} in the embedding space.  Removing a (part of a) subspace can still leave residual bias~\cite{gonen2019lipstick}, or in the case of gender also undesirably alter the association of the word ``pregnant" with the words ``female" or ``mother." 

Further, in many NLP tasks such as coreference resolution, completely losing say, gender information will not be beneficial and render pronoun resolution difficult. Natural Language Inference (NLI) is a well-studied task which examines the directional relationships between sentences and helps evaluate the associations between words in word embeddings \cite{dev2019measuring}.
For instance, if we take sentence pairs: 

\hspace*{10em} \textbf{Premise}: A matriarch sold a watch.\\
\hspace*{11em} \textbf{Hypothesis}: A woman sold a watch. 

In a task of NLI, the objective is to determine if the hypothesis is \emph{entailed} by the premise, \emph{contradicted} by it, or neither (\emph{neutral} to it). A GloVe-based NLI~\cite{parikh2016decomposable} model without any form of debiasing, predicts label \emph{entail} with a high probability of $97\%$, as the information of a matriarch being a woman is correctly identified by the model.  
However, following what we see in Chapter \ref{chap: bias paper 2} after debiasing using linear projection~\cite{Bias1}, the prediction becomes \emph{neutral} with a probability $62\%$ while the confidence on label $\emph{entail}$ drops to much lower at $16\%$. 
Aggressive debiasing could erase valid gendered associations in such examples.

Biases in word embeddings are associations that are stereotypical and \emph{untrue}. 
Hence, our goal is to carefully uncouple these (and only these) associations without affecting the rest of the embedding space. The residual associations need to be more carefully measured and preserved.  


To address this issue, we need to balance information retention and bias removal. We propose a method that decouples specific biased associations in embeddings while not perturbing valid associations within each space. The aim is to develop a method which is low cost and usable, performs at least as well as the existing linear projection based methods \cite{lauscher2019bias}, and also performs better on retaining the information that is relevant and not related to the bias. For concepts (captured by subspaces) identified to be wrongly associated with biased connotations, our proposal is to \emph{orthogonalize} or \emph{rectify} them in space, thus reducing their explicit associations, with a technique we call \oscar (orthogonal subspace correction and rectification). Word vectors outside these two directions or subspaces (but in their span) are stretched in a graded manner, and the components outside their span are untouched.  We describe this operation in detail in Section \ref{sec: correction}.

For \textbf{measuring bias}, we use existing intrinsic ~\cite{Caliskan183} and extrinsic measures ~\cite{dev2019measuring} as described in Chapters \ref{chap: bias paper 1} and \ref{chap: bias paper 2} respectively. The extrinsic measure of NLI, as a probe of bias, demonstrates how bias percolates into downstream tasks,
both static and contextualized. For instance, among the sentences: 

\hspace*{10.5em} \textbf{Premise}: The doctor bought a bagel.\\
\hspace*{9.3em} \textbf{Hypothesis 1}: The woman bought a bagel. \\
\hspace*{10.6em} \textbf{Hypothesis 2}: The man bought a bagel. 

Both hypotheses are \emph{neutral} with respect to the premise. However, GloVe, using the decomposable attention model \cite{parikh2016decomposable}, deems that the premise entails Hypothesis 2 with a probability $84\%$ and contradicts Hypothesis 1 with a probability $91\%$. Contextualized methods (e.g., ELMo, BERT, and RoBERTa) also demonstrate similar patterns and perpetrate similar gender biases. It has also been demonstrated that this bias percolation can be mitigated by methods that project word vectors along the subspace of bias \cite{dev2019measuring}; these incorrect inferences see a significant reduction, implying a reduction in bias expressed.

\emph{But what happens if there is proper gender information being relayed by a sentence pair?} The existing tasks and datasets for embedding quality do not directly evaluate that. 
For instance:

\hspace*{10.5em} \textbf{Premise}: The gentleman drove a car.\\
\hspace*{10.9em} \textbf{Hypothesis 1}: The woman drove a car. \\
\hspace*{12.3em} \textbf{Hypothesis 2}: The man drove a car.

Here, the premise should contradict the first hypothesis and entail the second. 
We thus expand the use of NLI task as a probe not just to measure the amount of bias expressed but also the amount of correct gender information (or other relevant attributes) expressed. This measurement allows us to balance between bias retained and information lost in an explicit manner. 

We demonstrate in Section \ref{sec: information lost} that the ability to convey correctly gendered information is compromised when using projection-based methods for bias reduction, meaning that useful gender information is lost. This demonstrates our motivation for the development of more refined geometric operations that will achieve bias mitigation without loss of features entirely.

In this chapter: (i) We develop a method \oscar that is based on orthogonalization of subspaces deemed to have no interdependence by social norms, such that there is minimal change to the embedding and loss of features is prevented. This method performs similarly and in some cases, better than the existing best performing approach to bias mitigation.

(ii) We demonstrate how \oscar is applicable in bias mitigation in both context free embeddings (GloVe) and contextualized embeddings (RoBERTa).

(iii) Further, we develop a combination of tests based on the task of natural language inference which help ascertain that significant loss of a feature has not been incurred at the cost of bias mitigation. 

 
\section{Considerations in Debiasing Word Embeddings}
Debiasing word representations involves a trade-off as we described before. We need to maximize the reduction of bias while minimizing the perturbation of all extraneous information that the representation contains.
Further, the definition of bias is complicated.  In Chapter \ref{chap: bias paper 2} we attempt to define it procedurally through an actual task which serves as a proxy for concepts overlap in the embedding.  
Thus in addition to the standard measuring stick of WEAT, we also introduced extrinsic tasks which actually predict relationships between sentences. We explore both the intrinsic and extrinsic methods to understand the trade off between bias and valid information in this section.

\subsection{Maximizing Bias Removal}
It has been established that social biases creep into word embeddings and how that affects the tasks they are put towards. There have been a few different approaches that try to mitigate these biases. 
There is a class of techniques, exemplified by Zhao \etal~\cite{gn-glove}, which retrain the embedding from scratch with information meant to remove bias.  These can be computationally expensive and are mostly orthogonal to the techniques we focus on.  
In particular, our main focus will be on techniques~\cite{debias,Bias1,ravfogel2020null,dev2019measuring} that focus on directly modifying the word vector embeddings without retraining them; we detail them in Section  \ref{sec: debiasing methods}.  These mostly build around identifying a bias subspace and projecting all words along that subspace to remove its effect.  They have mainly been evaluated on removing bias, measured structurally~\cite{Caliskan183} or on downstream tasks~\cite{dev2019measuring}.

\subsection{Retaining Embedding Information}
\label{sec: information lost}
The ability to debias should not strip away the ability to correctly distinguish between concepts in the embedding space. 
Correctly gendered words such as ``man", ``woman", ``he", and ``she", encode information which helps enrich both intrinsic quality of word embeddings (via similarity-based measures) and extrinsic task performances (such as in natural language inference, pronoun resolution).  Further, there are other atypically gendered words such as ``pregnant" or ``testosterone" which are more strongly associated with one gender in a meaningful way. Retaining such associations enrich language understanding and are essential. While there are tasks which evaluate the amount of bias reduced~\cite{Caliskan183,dev2019measuring}, and there are numerous tasks which evaluate the performance using word embeddings as a general measure of semantic information contained, 
no tasks specifically evaluate for a specific concept like gender. 


\subsection{Differentiability}
While some debiasing efforts have focused on non contextualized embeddings (e.g., GloVe, Word2Vec), many natural language processing tasks have moved to contextualized ones (e.g., ELMo, BERT, RoBERTa).
Recent methods~\cite{dev2019measuring} have shown how the debiasing method can be applied to both scenarios.
The key insight is that contextualized embeddings, although being context-dependent throughout the network, are context-independent at the input layer.
Thus, debiasing can be effective if it is (a) applied to this first layer, and (b) maintained in the downstream training step where the embeddings are subject to gradient updates.  Note that maintaining debiasing during training requires it to be differentiable (such as linear projection~\cite{Bias1}).

\section{New Measures of Information Preservation}

We provide two new approaches to evaluate specific information retained after a modification (e.g., debiasing) on embeddings.  One is \emph{intrinsic} which measures structure in the embedding itself; the other is \emph{extrinsic}, which measures the effectiveness on tasks which use the modified embeddings.

\subsection{New Intrinsic Measure (\weatS)}  
This extends the measure \weat proposed by Caliskan \etal~\cite{Caliskan183}.  Both use gendered words sets as a baseline for the embeddings representation of gender
\\
\phantom{123456} $X$ : \{man, male, boy, brother, him, his, son\} \quad\quad\quad and
\\
\phantom{123456} $Y$ : \{woman, female, girl, brother, her, hers, daughter\},  
\\
two more sets $A$,$B$, and a comparison function 
$s(X,Y,A,B) = \sum_{x\in X} s(x,A,B) - \sum_{y \in Y} s(y,A,B)$
where, 
$s(w,A,B) = \mathsf{mean}_{a \in A} \cos(a,w) - \mathsf{mean}_{b \in B} \cos(b,w)$ and, $\cos(a,b)$ is the cosine similarity between vector $a$ and $b$.  This score is normalized by $\mathsf{stddev}_{w \in X \cup Y} s(w,A,B)$.
In \weat $A$,$B$ are words that should not be gendered, but stereotypically are (e.g., $A$ male-biased occupations like ``doctor", ``lawyer", and $B$ female-associated ones like ``nurse", ``secretary"), and the closer $s(X,Y,A,B)$ is to $0$ the better.  

In \weatS $A$ and $B$ are definitionally gendered ($A$ male and $B$ female) so we want the score $s(X,Y,A,B)$ to be large.  In Section \ref{Experiments} we use $A$,$B$ as he-she, as definitionally gendered words (e.g., father, actor and mother, actress), and as gendered names (e.g., ``james", ``ryan" and ``emma", ``sophia"); all listed in Appendix \ref{app : words weat*}.

\subsection{New Extrinsic Measure (\cnli)}
\label{sec : extrinsic good residuals}
Maintaining a high performance on SNLI test sets after debiasing need not actually imply the retention of useful and relevant gender information. 
A very tiny fraction of sentences in SNLI actually evaluates the existence of coherent gender relations in the sentence pairs. 
Similarly, the task of coreference resolution can involve resolving gendered information, it is more complex with a lot of other factors involved. For instance, there can be multiple people of the same gender in a paragraph with the task being the identification of the person being referred to by a given pronoun or noun. This is a much more complex evaluation than checking for correct gender being propagated. We propose here a simpler task which eliminates noise from other factors and only measures if correctly gendered information is passed on.

We extend the use of textual entailment bias evaluations described in Chapter \ref{chap: bias paper 2} toward the evaluation of correctly gendered information, we call it \cnli : sentence inference retention test. These tasks have the advantage of being sentence based and thus use context much more than word based tests such as WEAT, thus enabling us to evaluate contextualized embeddings such as RoBERTa. Unlike the original templates which were ideally related neutrally, these sentences are constructed in a way that the probabilities should be entailment or contradiction. 

For instance, in an \textbf{entailment task} a sample sentence pair that should be entailed would have the subject as words denoting the same gender in both the premise and the hypothesis:

\hspace*{10em} \textbf{Premise}: The lady bought a bagel.
\\
\hspace*{10em} \textbf{Hypothesis}: The woman bought a bagel. 

We should note here that not all same gendered words can be used interchangeably in the premise and hypothesis. For instance, the words ``mother" in the premise and ``woman" in hypothesis should be entailed but the opposite should not be entailed and would thus not be in our set of sentence pairs for this task. We use 12 words (6 male and 6 female) in the premise and 4 (``man", ``male", ``woman" and ``female") in the hypothesis. We only use the same gender word for the hypothesis as the premise each time. Along with 27 verbs and 184 objects. The verbs are sorted into categories (e.g., commerce verbs like bought or sold; interaction verbs like spoke, etc.) to appropriately match objects to maximize coherent templates (avoiding templates like ``The man ate a car"). 

In a \textbf{contradiction task} a sentence pair that should be contradicted would have opposite gendered words in the premise and hypothesis:

\hspace*{10em} \textbf{Premise}: The lady bought a bagel.
\\
\hspace*{10em} \textbf{Hypothesis}: The man bought a bagel.

Unlike the case of the task for entailment, here all words of the opposite gender can be used interchangeably in the premise and hypothesis sentences, as all combinations should be contradicted.  We use 16 words (8 male and 8 female) in the premise and any one word from the opposite category in the hypothesis. 
We use the same 27 verbs and 184 objects as above in a coherent manner to form a resultant list of sentence pairs whose correct prediction should be a contradiction in a textual entailment task. For both these types of tests, we have over $400{,}000$ sentence pairs.  

On each task, with N being the total number of sentence pairs per task, let the probability of entailment, contradiction and neutral classifications be $P_e$, $P_c$ and $P_n$ respectively. We define the following two metrics for the amount of valid gendered information contained:
\\
\phantom{1234} Net Entail (or Contradict) = $\frac{\sum P_{e}}{N}$  (or = $\frac{\sum P_{c}}{N}$)
\\    
\phantom{1234} Fraction Entail (or Contradict) = $\frac{\text{Number of entail (or contradict) classifications}}{N}$

The higher the values on each of these metrics, the more valid gendered information is contained by the embedding.
As we would expect, GloVe does well at our \cnli test, with net entail and fraction entail at $0.810$ and $0.967$ and net contradict and fractional contradict at $0.840$ and $0.888$. Debiasing the embedding by linear projection reduces this performance by about $10\%$ in each metric, as we see in experiments later in this chapter. 



\section{The Method: \oscar}
\label{sec: correction}

We describe here a new geometric operation -\oscar:  Orthogonal Subspace Correction and Rectification- that is an alternative to the linear projection-based ones. 
This operator applies a graded rotation on the embedding space; it rectifies two identified directions (e.g., gender and occupations) which should ideally be independent of each other, so they become orthogonal, and remaining parts of their span are rotated at a varying rate so it is differentiable. The goal is that the components of every word vector along the two subspaces are effectively orthogonalized.

This method requires us to first identify two subspaces that in the embedding space should not have interdependence. This can be task specific; e.g., for resume sorting, subspaces of gender and occupation should be independent of each other. Since this interdependence has been observed  \cite{debias,Caliskan183,Bias1} in word embeddings, we use them as an exemplar in this paper. We determine 1-dimensional subspaces capturing gender ($V_1$) and occupations ($V_2$) from words listed in the Appendix. 
Given the two subspaces $V_1$ and $V_2$ which we seek to make orthogonal, we can identify the two directions $v_1 \in (V_1)$ and $v_2 \in (V_2)$  (that is, they are vectors $v_1 \neq v_2$,  $\|v_1\| = \|v_2\| = 1$).  

We can then restrict to the subspace $S = (v_1, v_2)$.  In particular, we can identify a basis using vectors $v_1$ and $v_2' = \frac{v_2 - v_1 \langle v_1, v_2 \rangle}{\|v_2 - v_1 \langle v_1, v_2 \rangle\|}$ (these two $v_1$ and $v_2'$, are now definitionally orthogonal).  
We can then restrict any vector $x \in \mathbb{R}^d$ to $S$ as two coordinates $\pi_S(x) = (\langle v_1, x \rangle, \langle v'_2, x \rangle)$.  We adjust these coordinates, and leave all $d-2$ other orthogonal components fixed.

We now restrict our attention to within the subspace $S$.  
Algorithmically we do this by defining a $d \times d$ rotation matrix $U$.  The first two rows are $v_1$ and $v'_2$.  The next $d-2$ rows are any set $u_3, u_4, \ldots, u_d$ which completes the orthogonal basis with $v_1$ and $v'_2$.  
We then rotate all data vectors $x$ by $U$ (as $U x$).  Then we manipulate the first $2$ coordinates $(x'_1, x'_2)$ to $f(x_1,x_2)$, described below, and then reattach the last $d-2$ coordinates, and rotate the space back by $U^T$.

Next we devise the function $f$ which is applied to each data vector $x \in S$ (we can assume now that $x$ is two dimensional).  See illustration in Figure \ref{fig:correction}.  Given $v_1$ and $v_2$ let 
$\theta' = \arccos(\langle v_1, v_2 \rangle)$, $\theta = \frac{\pi}{2} - \theta'$.  
Now define a rotation matrix 
\[
R = \left[ \begin{array}{cc} \cos \theta & - \sin \theta \\ \sin \theta & \cos \theta \end{array} \right].
\]  
The function $f$ should be the identity map for $v_1$ so $f(v_1) = v_1$, but apply $R$ to $v_2$ so $f(v_2) = R v_2$; this maps $v_2$ to $v'_2$ so it is orthogonal to $v_1$.  For every other vector, it should provide a smooth partial application of this rotation so $f$ is continuous.  

In particular,  for each data point $x \in S$ we will determine an angle $\theta_x$ and apply a rotation matrix $f(x) = R_x x$ defined as
\[
R_x = \left[ \begin{array}{cc} \cos \theta_x & - \sin \theta_x \\ \sin \theta_x & \cos \theta_x \end{array} \right].  
\]

Towards defining $\theta_x$, calculate two other measurements 
$\phi_1 = \arccos \langle v_1, \frac{x}{\|x\|} \rangle$ and 
$d_2 = \langle v'_2, \frac{x}{\|x\|} \rangle$.
Now we have a case analysis:
\[
\theta_x = \begin{cases}
\theta \frac{\phi_1}{\theta'} & \text{if } d_2 > 0 \text{ and } \phi_1 < \theta' 
\\
\theta \frac{\pi - \phi_1}{\pi - \theta'} & \text{if } d_2 > 0 \text{ and }  \phi_1 > \theta' 
\\
\theta \frac{\pi - \phi_1}{\theta'} &  \text{if } d_2 < 0 \text{ and }  \phi_1 \geq \pi-\theta' 
\\
\theta \frac{\phi_1}{\pi - \theta'} & \text{if } d_2 < 0 \text{ and } \phi_1 < \pi-\theta'  .
\end{cases}
\]

This effectively orthogonalizes the components of all points along the subspaces $V_1$ and $V_2$. So, points lying in subspaces especially near $V_2$ get moved the most, while the rest of the embedding space faces a graded stretching. The information contained outside the two subspaces $S$ remains the same, thus preserving most of the inherent structure and content of the original embeddings. Algorithm \ref{alg:OSCaR} breaks this operation down into sequential steps.

\begin{algorithm}
	\caption{\label{alg:OSCaR}$\oscar$}
	\begin{algorithmic}
		\STATE \textbf{Compute} subspaces $v_1$ and $v_2$ which are to be rectified.
		\STATE \textbf{Determine} a new $n$ dimensional basis, U with $v_1' = v_1$, $v_2' = \frac{v2 - v1\langle v_1, v_2\rangle}{\| v2 - v1\langle v_1, v_2\rangle\|}$ and $v_3,...,v_n$ are vectors orthonormal to $v_1$, $v_2$ and each other. 
		\STATE \textbf{Rotate} each word vector $x$ to $x' = Ux$, providing a new basis.
		\STATE \textbf{Orthogonalize} components of vector $x'$ along the first two components $v_1'$ and $v_2'$ by function $f$ to get $x'' = f(x') = R_xx'$.
		\STATE \textbf{Rotate} word vectors back to original basis as $x''' = U^Tx''$.
	\end{algorithmic}
\end{algorithm}

We thus, have a \emph{sub differentiable} operation applied onto all points in the space, enabling us to extend this method to contextualized embeddings.  It can now be part of the model specification and integrated with the gradient-based fine-tuning step. 
Further, it is a post-processing step applied onto an embedding space, thus its computational cost is relatively low, and it is easily malleable for a given task for which specific subspaces may be desired to be independent. 


\section{Experimental Methods}
\label{sec:exp-methods}
\subsection{Debiasing Methods}
\label{sec: debiasing methods}
Gender bias and its reduction has been observed on GloVe embeddings \cite{debias,Bias1} using different metrics \cite{lauscher2019bias,Bias1,Caliskan183,dev2019measuring}.  
Each of these is projection-based, and begins by identifying a subspace represented as a vector $v$: in all of our experiments we determine with using the vector between words `he' and `she.' Some methods use an auxiliary set of definitionally gendered words $G$ (see Appendix \ref{app : words HD} and \ref{app : words INLP}) which are treated separately.

We briefly describe four such methods here, two of which we have introduced in Chapter \ref{chap: bias paper 1}.
\begin{itemize}
\item 
{Linear Projection (LP): }
This is the simplest method developed in Chapter \ref{chap: bias paper 1}.  For every embedded word vector $w$ it projects it along $v$ to remove that component as 
\[
w' =  w - w \langle w,v \rangle.  
\]
Afterwards the $d$-dimensional vectors lie in a $(d-1)$-dimensional subspace, but retain their $d$ coordinates; the subspace $v$ is removed.  
Laushcher \etal~\cite{lauscher2019bias} show this method \cite{Bias1} reduces the most bias and has the least residual bias. 



\item{Hard Debiasing (HD): } 
This original method Bolukbasi \etal~\cite{debias} begins with the above projection operation,
and first applies it to all words $w$ not in $G$.  Then using an identified subset $G' \subset G$ which comes in pairs (e.g., man, woman) it projects these words also, but then performs an ``equalize'' operation.  This operation ensures that after projection they are the same distance apart as they were before projection, but entirely not within the subspace defined by $v$.

As we will observe, this equalization retains certain gender information in the embedding (compared to projection) but has trouble generalizing when used on words that carry gender connotation outside of $G'$ (such as names). The final location of other words can also retain residual bias~\cite{gonen2019lipstick,lauscher2019bias}.  

LP and HD are detailed more in Chapter \ref{chap: bias paper 1}.

\item{Iterative Nullspace Projection (INLP): }
This method~\cite{ravfogel2020null} begins with LP using $v$ on all words except the set $G$.
It then automatically identifies a second set $B$ of most biased words: these are the most extreme words along the direction $v$ (or $-v$)~\cite{ravfogel2020null}.   After the operation, it identifies the residual bias~\cite{ravfogel2020null} by building a linear classifier on $B$.  The normal of this classier is then chosen as the next direction $v_1$ on which to apply the LP operation, removing another subspace.  It continues $35$ iterations, finding $v_2$ and so on until it cannot identify significant residual bias.  

\item{\oscar:}
We also apply \oscar, using `he-she' as $v_2$, and the subspace defined by an occupation list (see Appendix \ref{app: words}) as $v_1$.  This subspace is determined by the first principal component of the word vectors in the list. Our code for reproducing experiments will be released upon publication.

\end{itemize}

\subsection{Debiasing Contextualized Embeddings}
The operations above are described for a non contextualized embedding; we use one of the largest such embeddings GloVe (on 840 B token Common Crawl).   
They can be applied to contextualized embedding as well; we use RoBERTa (base version released by April 2020)\cite{wolf2019transformers}, the widely adopted state-of-the-art architecture.  
As advocated in Dev \etal~\cite{dev2019measuring}, for RoBERTa we only apply this operation onto the first layer which carries no context.  Technically this is a subword-embedding, but `he,' `she,' and all words in $G$ are stored as full subwords, so there is no ambiguity.  
LP is differentiable and \oscar is sub differentiable, but can be worked around quite like hinge-loss, so these modifications can be maintained under the fine-tuning gradient-based adjustment.  HD and INLP are not applied to all words and are more intricate, and so we do not have a clearly defined gradient.  For these, we only apply it once and do not repeat it in the fine-tuning step.  

Though we compare these different methods here, our motive is to maximize the removal of representational bias. To that end, we use each method as best as possible. The other two methods were suitable to be used only as such. Infact, INLP -one of the non differentiable ones- repeats LP several times over internally. Also, both INLP and HD require much more extensive word lists to be able to perform. Even so, we leave them as is to maximize the amount of bias it can remove, while we restrict LP and \oscar to reduced word lists. Of course, eventually, perhaps a combination of all of these methods should be used in real world tasks. But to achieve that, we want to understand how to use each method to its absolute best, which is why each method was fairly allowed to use all resources it could use.

\subsection{Extrinsic Measurement Through Textual Entailment}
For training the embeddings in the task of textual entailment, we use the SNLI \cite{snli} dataset. While MultiNLI contains more complex and diverse sentences than SNLI, making it a good contender for training entailment tasks, we observed that MultiNLI also carries significantly more gender bias. On training on MultiNLI, our model tends to return a significantly higher proportion of biased outcomes as seen in Table \ref{tbl : snl vs mnli} using the standard metrics for neutrality defined in Chapter \ref{chap: bias paper 2}. Over $90\%$ of sentence pairs that should have neutral associations are classified incorrectly using both GloVe and RoBERTa when trained using MultiNLI whereas using SNLI yields less than $70\%$ biased and incorrect classifications on the same dataset.  Since we focus on the representational bias in word embeddings and not on the bias in the datasets used in different machine learning tasks, using MultiNLI here would interfere with our focus. 
Moreover, MultiNLI has more complex sentences than SNLI, which means that there is a larger probability of having confounding noise. This would in turn weaken the implication of an incorrect inference of bias expressed.

\subsection{Keeping Test/Train Separate} 
Since two of the methods below (Hard Debiasing and Iterative Nullspace Projection) use words to generate lists of gendered and biased words as an essential step, we filter words carefully to avoid test-train word overlap. 
For words in Appendix \ref{app: words weat} used in the \weat test, we filter out all words from the training step involved in HD and INLP.  We also filter out words used for template generation for our NLI tests, listed in Appendix \ref{app : words templates}.  A larger list in Appendix \ref{app : words weat*}, has 59 words of each gender and removing all would hamper the training and overall function of both HD and INLP, so we let overlapping words from this list (which are not in the other two lists) remain in the training step of HD and INLP.

We use the templates proposed for using textual entailment as a measure of bias \cite{dev2019measuring}, restructuring sentence pairs for \cnli. 
Since we just use the words `he' and `she' to determine the gender subspace, we can fairly use all other gendered words in the list in the Appendix to generate templates.  For occupations, we split the full set, using some to train a subspace for \oscar, and retaining others for templates in testing.  For occupations, however, since we use a subset of occupation words to determine the occupation subspace for the \oscar operation (Appendix \ref{app: words}), we have disjoint lists for testing with \weat (Appendix \ref{app: words weat}) and NLI templates (Appendix \ref{app : words templates}).


\section{Experimental Results}
\label{Experiments}

\subsection{Intrinsic Measures for Measurement of Bias}
Table \ref{tbl: weat} shows the results on \weat~\cite{Caliskan183} between groups of attribute words (we use the he-she vector) and target words (we use $3$ sets of stereotypical types of words: occupations, work vs. home, and math vs. art). %
%
It also shows results for 
the Embedding Coherence Test (\textsf{ECT}) \cite{Bias1}, showing the association of vectors from $X$ (gendered words) with a list of attribute neutral words $Y$ (occupation) using the Spearman Coefficient. 
The score ranges in $[-1, 1]$ with $1$ being ideal.
%
\oscar performs similarly to the other debiasing methods in removing bias, always significantly reducing the measured intrinsic bias.

\subsection{Extrinsic Measures for Measurement of Bias}
Table \ref{tbl: bias extrinsic} shows the textual entailment scores for GloVe and RoBERTa before and after they have been debiased by the various approaches.  All debiasing methods, including \oscar, increase the neutrality score (recall sentences should be objectively neutral) without significantly decreasing the performance (Dev/Test scores);  HD on GloVe is the only one where these scores drop more than $0.015$.  While INLP appears to be the most successful at debiasing GloVe, \oscar and LP are clear the most successful on RoBERTa (with LP slightly better); this is likely since they are differentiable, and can be integrated into fine-tuning.

\subsection{Intrinsic Metric for Evaluating Information Preserved}
Table \ref{tbl : information lost, intrinsic} demonstrates using \weatS how much correctly gendered information is contained or retained by an embedding after debiasing by the different methods; the larger the scores the better.   
Instead of probing with gendered words or names, we first used random word sets, and on $1000$ trials the average \weat or \weatS score was $0.001$ with a standard deviation of $0.33$  So almost all of these results are significant; the exception is LP with he-she, since this basically aligns these two words exactly, and the resulting vector after normalizing probably acts randomly.

The first column uses sets $A$ = \{he\} and $B$ =\{she\} to represent gender, the same words used to define the debias direction.  
The second column uses two sets of $59$ definitionally gendered words as $A$ and $B$, and the third uses $700$ gendered names for each set $A$ and $B$ (both in Appendix \ref{app : words weat*}). 
The baseline row (row 1, not debiased GloVe) shows these methods provide similar \weatS scores.  


The LP method retains the least information for tests based on he-she or other gendered words.  These scores are followed closely behind by the scores for INLP in these tasks.  The single projection of LP retains more info when compared against names, whereas INLP's multiple projects appear to remove some of this information as well.  
HD also performs well for he-she and gendered names; we suspect the equalize steps allows the information to be retained for these types of words.  However, since names are not prespecified to be equalized, HD loses much more information when compared against gendered names.  

Finally, \oscar performs at or near the best in all evaluations.  HD (with its equalize step) performs better for he-she and other gendered words, but this does not generalize to words not specifically set to be adjusted as observed in the evaluation of the gendered names.  In fact, since names do not come in explicit pairs (like uncle-aunt) this equalize step is not even possible.

\subsection{Extrinsic Metric for Evaluating Information Preserved}
For the new \cnli task we observe the performance using GloVe and RoBERTa in Table \ref{tbl : information lost, extrinsic}. 
%
GloVe without any debiasing, which is our baseline, preserves significant correctly gendered information as seen by the first row. The fraction of correct entails (male premise : male hypothesis, female premise : female hypothesis) and contradicts (male premise : female hypothesis, female premise : male hypothesis) are both high.
We see a fall in these scores in all projection based methods (LP, HD and INLP), with a uniform projections step (LP) doing the best among the three. \oscar does better than all three methods in all four scores measuring valid entailments and contradictions.

Using RoBERTa, we see a similar pattern where \oscar outperforms the projection based methods, LP and HD in the retention of valid gendered information.  It also does better than INLP on the entailment tasks.  
However, INLP does better on the contradition tasks than \oscar.  But, recall INLP showed almost no effect on debiasing with RoBERTa (Table \ref{tbl: bias extrinsic}); we suspect this is because most of the effect of the projection are erased in the fine-tuning step (recall INLP is not obviously differentiable, so it is just applied once before fine-tuning), and overall INLP has a minimal effect when used with contextualized embeddings.

\subsection{TPR-Gap-RMS}

In Table~\ref{tbl:rms_diff}, we show the TPR-Gap-RMS metric as used in~\cite{ravfogel2020null}, which is an aggregated measurement of gender bias score over professions. Lower scores imply a lesser gendered bias in professions.
We refer readers to~\cite{ravfogel2020null} for detailed definition.
We follow the same experiment steps, except that we apply different debiasing algorithms to the input word embeddings (instead of the CLS token). This allows us to compare debiasing methods with static use of contextualized embeddings (i.e., without fine-tuning).
We see that RoBERTa and RoBERTa HD perform on par while both linear projection and iterative projection methods and our method \oscar perform close to each other.

\subsection{Standard Tests for Word Embedding  Quality}

Non contextual word embeddings like GloVe or word2vec are characterized by their ability to capture semantic information reflected by valid similarities between word pairs and the ability to complete analogies. These properties reflect the quality of the word embedding obtained and should not be diminished post debiasing. We ascertain that in Table \ref{tbl: intrinsic quality tests}. We use a standard word similarity test \cite{wsim} and an analogy test \cite{wordtovec} which measure these properties in word embeddings across our baseline GloVe model and all the debiased GloVe models. All of them perform similarly to the baseline GloVe model, thus indicating that the structure of the embedding has been preserved. While this helps in the preliminary examination of the retention of information in the embeddings, these tests do not contain a large number of sensitive gender comparisons or word pairs. It is thus, not sufficient to claim using this that bias has been removed or that valid gender associations have been retained, leading to the requirement of methods described in the paper and methods in other work \cite{Bias1,ravfogel2020null,debias,Caliskan183} for the same. 

\subsection{Gendered Names and Debiasing}
All these methods primarily use the words ``he" and ``she"  to determine the gender subspace, though both hard debiasing and INLP use other predetermined sets of gendered words to guide the process of both debiasing and retention of some gendered information.

Gendered names too have been seen to be good at helping capture the gender subspace \cite{Bias1}. We compare in Table \ref{tbl : information lost, intrinsic, names} the correctly gendered information retention when debiasing is done using projection or correction. We represent simple projection, hard debiasing and INLP by simple projection since it is the core debiasing step used by all three. Both rows have been debiaised using the gender subspace determined using most common gendered names in Wikipedia (listed in Appendix) and correction uses in addition, the same occupation subspace as used in Table \ref{tbl : information lost, intrinsic}. Each value is again a WEAT* calculation where the two sets of words (X and Y) being compared against are kept constant as the same in Table \ref{tbl : information lost, intrinsic}. The first column of this table thus represents the association with the subspace determined by 'he - she' and correction results in a higher association, thus implying that more correctly gendered information retained. We see a similar pattern in columns 2 and 3 which represent other gendered words and gendered names sans the names used to determine the subspace used for debiasing. That correction fares significantly better even among other gendered names reflects the higher degree of precision of information removal and retention.

\section{Broader Impact}
\label{sec:broader-impacts}

Biases in language representation are pernicious with wide ranging detrimental effects in machine learning applications. The propagation of undesirable and stereotypical associations learned from data into decisions made by language models maintains the vicious cycle of biases. Combating biases before deploying representations is thus extremely vital. But this poses its own challenges. Word embeddings capture a lot of information implicitly in relatively few dimensions. These implicit associations are what make them state-of-the-art at tackling different language modeling tasks. Breaking down of these associations for bias rectification, thus, has to be done in a very calculated manner so as to not break down the structure of the embeddings. Our method, \oscar's surgical manner of rectifying associations helps integrate both these aspects, allowing it to have a more holistic approach at making word embeddings more usable.  Moreover, it is computationally lightweight and is differentiable, making it simple to apply adaptively at the time of analysis without extensive retraining.  

We envision that this method can be used for all different types of biased associations (age - ability, religion - virtue perception, etc.). Since it only decouples specific associations, not a lot of other components of each of these features are changed or lost.  

Moreover, we believe these techniques should extend in a straight-forward way to other distributed, vectorized representations that can exhibit biases, such as in images, graphs, or spatial analysis.  

\section{Discussion}
Debiasing word embeddings is a very nuanced requirement for making embeddings suitable to be used in different tasks. It is not possible to have exactly one operation that works on all different embeddings and identifies and subsequently reduces all different biases it has significantly.
What is more important to note is that this ability is not beneficial in every task uniformly, i.e., not all tasks require bias mitigation to the same degree. Further, the debiasing need not apply to all word groups as generally. Disassociating only biased associations in a continuous manner is what needs to be achieved. Hence, having the ability to debias embeddings specifically for a given scenario or task with respect to specific biases is extremely advantageous. 

Our method of orthogonal correction is easy to adapt to different types of biased associations, such as the good-bad notions attached to different races~\cite{Caliskan183,Crenshaw} or religions~\cite{dev2019measuring}, etc. Creating metrics is harder with not as many words to create templates out of, making a comprehensive evaluation of bias reduction or information retention harder in these types of biases. Though we see the correction method being widely applicable to other associations and embedding types, 
we leave that for future exploration.


\begin{figure}
\centering 
\includegraphics[width=0.6\textwidth]{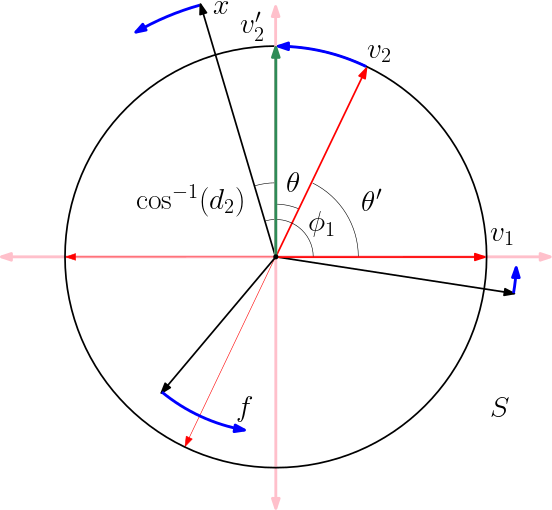}
\caption{\label{fig:correction} Illustration of the correction operation $f$ in the subspace $S$.} 
\vspace{-3mm}
\end{figure}

\begin{table}[]
\caption{Comparison of bias expressed when using two different commonly used datasets - SNLI and MNLI - for natural language inference.}	
	\centering
\small
\begin{tabular}{c||cccc}
	\hline
	Embedding & Net Neutral & Fraction Neutral & Threshold (0.5) &Threshold (0.7)   \\
	\hline
	GloVe (SNLI) & 0.321 & 0.296 & 0.186 & 0.027 \\
	GloVe (MNLI) & 0.072 & 0.004 & 0.0 & 0.0\\
	\hline
	RoBERTa (SNLI) & 0.338 & 0.329 & 0.316 & 0.139 \\
	RoBERTa (MNLI) & 0.014 & 0.002 & 0.0 & 0.0 \\
	\hline
\end{tabular}

\label{tbl : snl vs mnli}
\end{table}

\begin{table}[]
	\caption{Intrinsic Bias Measurement: Gender bias contained by embeddings. There are 3 WEAT tests (a score closer to 0 is better) results in this table : 
	     (1) with stereotypically gendered occupation words, 
	     (2) with work versus home related words, and 
	     (3) math versus art related words. 
	     ECT also measures bias and a higher score implies less bias \cite{Bias1}.}
	\centering

	\begin{tabular}{l||cccc}
		\hline
		Embedding & WEAT (1)& WEAT (2) & WEAT (3) & ECT \\
		\hline
		GloVe & 1.768 & 0.535 & 0.788 & 0.778 \\
		GloVe LP & 0.618 & 0.168 & 0.282 & 0.982\\
		GloVe HD & 0.241 & 0.157 & 0.273 & 0.942 \\
		GloVe INLP & 0.495 & 0.117 & 0.192 & 0.844 \\
		GloVe \oscar & 0.235 & 0.170 & 0.307 & 0.980\\
		\hline
	\end{tabular}

	\label{tbl: weat}
\end{table}

\begin{table}[]
	\caption{Extrinsic Bias Measurement: Gender bias expressed downstream by GloVE and RoBERTa embeddings as observed using NLI as a probe for bias measurement.}
	\centering

	\begin{tabular}{l||cccc}
		\hline
		Embedding & Net Neutral & Fraction Neutral & Dev & Test \\
		\hline
		GloVe & 0.321 & 0.296  & 0.879 &	0.873	\\
		GloVe LP  & 0.382 & 0.397  & 0.879	& 0.871\\
		Glove HD & 0.347 & 0.327  & 0.834 & 0.833\\
		Glove INLP & 0.499 & 0.539 & 0.864 & 0.859\\
		GloVe \oscar  & 0.400 & 0.414 & 0.872 & 0.869\\
		\hline
	
		

		RoBERTa  & 0.364 & 0.329  & 0.917 & 0.913\\
		RoBERTa LP & 0.481 & 0.502 & 0.916 & 0.911\\
		RoBERTa HD & 0.403 & 0.393 &  0.916 & 0.908\\
		RoBERTa INLP & 0.366 & 0.334 &  0.915 & 0.910\\
		RoBERTa  \oscar & 0.474 & 0.498 & 0.915 & 0.910\\
		\hline
	\end{tabular}

	\label{tbl: bias extrinsic}
\end{table}

\begin{table}[]
\vspace{1cm}
	\caption{Intrinsic Measurement of Information Retained: using \weatS; Larger scores better, as they imply more correctly gendered information expressed.}
	\centering

	\resizebox{\columnwidth}{!}{\begin{tabular}{l||ccc}
		\hline
		Embedding & WEAT*(he - she) & WEAT*(gendered words) & WEAT*(gendered names) \\
		\hline
		GloVe & 1.845  & 1.856  & 1.874\\
		GloVe LP & 0.385  & 1.207  & 1.389\\
		GloVe HD & 1.667 & 1.554  & 0.822\\
		GloVe INLP & 0.789 &  1.368 & 0.873\\
		GloVe \oscar & 1.361  & 1.396 & 1.543\\
		\hline
	\end{tabular}}

	\label{tbl : information lost, intrinsic}
\end{table}

\begin{table}[]
\vspace{1cm}
	\caption{Extrinsic Measurement of Information Retained: using \cnli. Higher scores are better.}
	\centering

	\resizebox{\columnwidth}{!}{\begin{tabular}{l||cccccc}
		\hline
		Embedding &Net Entail & Fraction Entail & Net Contradict & Fraction Contradict & Dev & Test\\
		\hline
		GloVe & 0.895 & 0.967 & 0.840 & 0.888 & 0.879 & 0.873\\
		GloVe LP & 0.810 & 0.865  & 0.715 & 0.713 & 0.879 & 0.871\\

		GloVe HD & 0.668 & 0.891  &0.544 & 0.752 & 0.834 & 0.833 \\
		GloVe INLP & 0.748 & 0.793 & 0.624 & 0.634 & 0.864 & 0.859\\
		GloVe \oscar & 0.845 & 0.911 & 0.747 & 0.759 & 0.872 & 0.869\\
		\hline
		RoBERTa & 0.966 & 0.991 & 0.972 & 0.977 & 0.917 & 0.913   \\
		RoBERTa LP & 0.921 & 0.967 & 0.902 & 0.905 & 0.916 & 0.911 \\
		RoBERTa HD & 0.846 & 0.871 & 0.824 & 0.828 & 0.916 & 0.908\\
		RoBERTa INLP & 0.943 & 0.969 & 0.970 & 0.955 & 0.915 & 0.910\\
		RoBERTa \oscar & 0.963 & 0.979 & 0.923 & 0.935 & 0.915 & 0.910\\
	
		\hline
	\end{tabular}}

	\label{tbl : information lost, extrinsic}
\end{table}

\begin{table}[]
\vspace{1cm}
	\caption{TPR-Gap-RMS metric of different debiasing method when applied on word embeddings.}
	\centering
	\small
	\resizebox{\columnwidth}{!}{\begin{tabular}{l|ccccc}
		\hline
		\small{Method} & \small{RoBERTa} & \small{RoBERTa LP} & \small{RoBERTa HD} & \small{RoBERTa INLP} & \small{RoBERTa OSCaR}\\
		\hline
		\small{TPR-Gap-RMS} & \small{0.19} & \small{0.15} & \small{0.18} & \small{0.13} & \small{0.14} \\
		\hline
	\end{tabular}}

	\label{tbl:rms_diff}
\end{table}

\begin{table}[]
\vspace{1cm}
	\caption{Intrinsic Standard Information Tests: These standard tests evaluate the amount of overall coherent associations in word embeddings. WSim is a word similarity test and Google Analogy is a set of analogy tests. }
	\centering

	\begin{tabular}{c||ccccc}
		\hline
		Embedding & GloVe & GloVe LP & GloVe \oscar & GloVe INLP  & GloVe HD\\
		\hline
		
		WSim & 0.697 & 0.693 & 0.693 & 0.686 & 0.695\\
		GAnalogy & 0.674 & 0.668 & 0.670 & 0.663 & 0.672\\
		\hline
	\end{tabular}

	\label{tbl: intrinsic quality tests}
\end{table}

\begin{table}[]
	\caption{Correctly gendered information contained by embeddings; Larger scores better as they imply more correctly gendered information expressed.}
	\centering
	\resizebox{\columnwidth}{!}{\begin{tabular}{c||c||cc}
		\hline
		Embedding & WEAT*(he - she) & WEAT*(gendered words) & WEAT*(gendered names) \\
		\hline
		GloVe Proj Names & 1.778  & 1.705 & 1.701\\
		GloVe Corr Names & 1.847 & 1.857 & 1.875\\
		\hline
	\end{tabular}}

	\label{tbl : information lost, intrinsic, names}
\end{table}

%% file: conclusion.tex
\chapter{Conclusion}

Distributed representations such as word vectors bring down the costs of computation as well as storage. The trade off however is the interpretability of the representation, which to a large extent is lost due to a lack of feature to dimension mapping. A conscious effort needs to be made to understand and interpret these representations and decouple any pernicious associations learnt from the underlying data before these representations are deployed in applications and tasks downstream. 

In this dissertation, we have explored ways to interpret the structure of distributed representations. We have further demonstrated how using a family of geometrically constrained, rigid operations we can untangle structured distributed embeddings and isolate concepts contained within. Utilizing the geometry of the embeddings alone, we achieve a multitude of tasks such as boosting of performance of these embeddings on relevant standardized tasks and bias identification, quantification and subsequent reduction in word representations. We have also seen how most of these functions are applicable to non language, noisy embeddings as well.

We also note the many challenges, from understanding what constraints and trade-offs are apt when untangling concept subspaces to questions about residual information. There are also intersectional subspaces (such as race and gender in language representations) that need to be specifically addressed. There is also a need to develop probes and metrics to quantify and understand the scope of each of these problems. For instance, while there are tasks in NLP which measure the overall coherence of information contained by word representations, tasks to measure valid and invalid information with respect to specific concepts are not extensively developed or used. We hope that this work can help build towards answering some of these questions.

%% file: appa.tex

\chapter{Word Lists}
We list here the different lists of words used throughout the dissertation. Each list has been used for different tasks, including subspace determination, template generation, or for evaluation of different metrics. The case of words used is important as word lists used in datasets have distinct words based on capitalization. All words used in our methods are in lower case unless otherwise specified, as reflected in the lists below.
\label{app: words}
\section{Gendered Words for Subspace Determination}

The gender subspace, unless otherwise specified in that specific chapter or section, has been determined using:
he, she

\section{Occupations for Subspace Determination }

\noindent scientist,
doctor,
nurse,
secretary,
maid,
dancer,
cleaner,
advocate,
player,
banker

\section{Names Used for Gender Bias Determination}
\noindent Male : john, william, george, liam, andrew, michael, louis, tony, scott, jackson  \\
Female :  mary, victoria, carolina, maria, anne, kelly, marie, anna, sarah, jane  \\

\section{Names Used for Racial Bias Determination}
\label{app: names race}
\noindent European American :  brad, brendan, geoffrey, greg, brett, matthew, neil, todd, nancy, amanda, emily, rachel  \\
African American : darnell, hakim, jermaine, kareem, jamal, leroy, tyrone, rasheed, yvette, malika, latonya, jasmine  \\
Hispanic :  alejandro, pancho, bernardo, pedro, octavio, rodrigo, ricardo, augusto, carmen, katia, marcella , sofia  \\

\vspace{-0.75cm}
\section{Names Used for Age-Related Bias Determination}
\label{app: age}
\noindent Aged : ruth, william, horace, mary, susie, amy, john, henry, edward, elizabeth  \\
Youth :  taylor, jamie, daniel, aubrey, alison, miranda, jacob, arthur, aaron, ethan\\

\section{Demonyms for Subspace Determination}
  american, 
  chinese, 
  egyptian, 
  french, 
  german, 
  korean, 
  pakistani, 
  spanish

\section{Adherents for Subspace Determination}
  atheist, 
  baptist, 
  catholic, 
  christian, 
  hindu, 
  methodist, 
  protestant, 
  shia

\section{Word Pairs Used for Flipping}
\label{app : flipping}
\noindent actor actress\\
author authoress\\
bachelor spinster\\
boy girl\\
brave squaw\\
bridegroom bride\\
brother sister\\
conductor conductress\\
count countess\\
czar czarina\\
dad mum\\
daddy mummy\\
duke duchess\\
emperor empress\\
father mother\\
father-in-law mother-in-law\\
fiance fiancee\\
gentleman lady\\
giant giantess\\
god goddess\\
governor matron\\
grandfather grandmother\\
grandson granddaughter\\
he she\\
headmaster headmistress\\
heir heiress\\
hero heroine\\
him her\\
himself herself\\
host hostess\\
hunter huntress\\
husband wife\\
king queen\\
lad lass\\
landlord landlady\\
lord lady\\
male female\\
man woman\\
manager manageress\\
manservant maidservant\\
masseur masseuse\\
master mistress\\
mayor mayoress\\
milkman milkmaid\\
millionaire millionairess\\
monitor monitress\\
monk nun\\
mr mrs\\
murderer murderess\\
nephew niece\\
papa mama\\
poet poetess\\
policeman policewoman\\
postman postwoman\\
postmaster postmistress\\
priest priestess\\
prince princess\\
prophet prophetess\\
proprietor proprietress\\\
shepherd shepherdess\\
sir madam\\
son daughter\\
son-in-law daughter-in-law\\
step-father step-mother\\
step-son step-daughter\\
steward stewardess\\
sultan sultana\\
tailor tailoress\\
uncle aunt\\
usher usherette\\
waiter waitress\\
washerman washerwoman\\
widower widow\\
wizard witch\\

\section{Occupation Words for EQT and ECT}
\label{sec : occupation words}
detective,
ambassador,
coach,
officer,
epidemiologist,
rabbi,
ballplayer,
secretary,
actress,
manager,
scientist,
cardiologist,
actor,
industrialist,
welder,
biologist,
undersecretary,
captain,
economist,
politician,
baron,
pollster,
environmentalist,
photographer,
mediator,
character,
housewife,
jeweler,
physicist,
hitman,
geologist,
painter,
employee,
stockbroker,
footballer,
tycoon,
dad,
patrolman,
chancellor,
advocate,
bureaucrat,
strategist,
pathologist,
psychologist,
campaigner,
magistrate,
judge,
illustrator,
surgeon,
nurse,
missionary,
stylist,
solicitor,
scholar,
naturalist,
artist,
mathematician,
businesswoman,
investigator,
curator,
soloist,
servant,
broadcaster,
fisherman,
landlord,
housekeeper,
crooner,
archaeologist,
teenager,
councilman,
attorney,
choreographer,
principal,
parishioner,
therapist,
administrator,
skipper,
aide,
chef,
gangster,
astronomer,
educator,
lawyer,
midfielder,
evangelist,
novelist,
senator,
collector,
goalkeeper,
singer,
acquaintance,
preacher,
trumpeter,
colonel,
trooper,
understudy,
paralegal,
philosopher,
councilor,
violinist,
priest,
cellist,
hooker,
jurist,
commentator,
gardener,
journalist,
warrior,
cameraman,
wrestler,
hairdresser,
lawmaker,
psychiatrist,
clerk,
writer,
handyman,
broker,
boss,
lieutenant,
neurosurgeon,
protagonist,
sculptor,
nanny,
teacher,
homemaker,
cop,
planner,
laborer,
programmer,
philanthropist,
waiter,
barrister,
trader,
swimmer,
adventurer,
monk,
bookkeeper,
radiologist,
columnist,
banker,
neurologist,
barber,
policeman,
assassin,
marshal,
waitress,
artiste,
playwright,
electrician,
student,
deputy,
researcher,
caretaker,
ranger,
lyricist,
entrepreneur,
sailor,
dancer,
composer,
president,
dean,
comic,
medic,
legislator,
salesman,
observer,
pundit,
maid,
archbishop,
firefighter,
vocalist,
tutor,
proprietor,
restaurateur,
editor,
saint,
butler,
prosecutor,
sergeant,
realtor,
commissioner,
narrator,
conductor,
historian,
citizen,
worker,
pastor,
serviceman,
filmmaker,
sportswriter,
poet,
dentist,
statesman,
minister,
dermatologist,
technician,
nun,
instructor,
alderman,
analyst,
chaplain,
inventor,
lifeguard,
bodyguard,
bartender,
surveyor,
consultant,
athlete,
cartoonist,
negotiator,
promoter,
socialite,
architect,
mechanic,
entertainer,
counselor,
janitor,
firebrand,
sportsman,
anthropologist,
performer,
crusader,
envoy,
trucker,
publicist,
commander,
professor,
critic,
comedian,
receptionist,
financier,
valedictorian,
inspector,
steward,
confesses,
bishop,
shopkeeper,
ballerina,
diplomat,
parliamentarian,
author,
sociologist,
photojournalist,
guitarist,
butcher,
mobster,
drummer,
astronaut,
protester,
custodian,
maestro,
pianist,
pharmacist,
chemist,
pediatrician,
lecturer,
foreman,
cleric,
musician,
cabbie,
fireman,
farmer,
headmaster,
soldier,
carpenter,
substitute,
director,
cinematographer,
warden,
marksman,
congressman,
prisoner,
librarian,
magician,
screenwriter,
provost,
saxophonist,
plumber,
correspondent,
organist,
baker,
doctor,
constable,
treasurer,
superintendent,
boxer,
physician,
infielder,
businessman,
protege\

\section{Words Lists for Template Generation}
\label{app : words templates}

We use Natural Language Inference (NLI) as a probe to measure societal biases expressed by text representations. For the task, we generate templates methodically.
We keep most word lists the same for template generation in both Chapters \ref{chap: bias paper 2} and \ref{chap: bias paper 3}.
For occupations, we change the list to remove those words that we use for subspace determination of occupations. This creates a test/train split for our experiments. 
For \cnli templates, we also modify the gendered word lists to create the word lists used in premise/hypodissertation for the entail and contradict templates.

\noindent \textbf{Gendered words for Neutral Templates }

\noindent Male: guy, gentleman, man

\noindent Female: girl,  lady, woman

\noindent \textbf{Gendered words for Entail and Contradict Templates}

\noindent Male Premise: guy, father, grandfather, patriarch, king, gentleman

\noindent Male Hypodissertation: man, male

\noindent Female Premise: girl, mother, grandmother, matriarch, queen, lady

\noindent Female Hypodissertation: woman, female

\noindent \textbf{Occupations for Templates }
accountant,
actor,
actuary,
administrator,
advisor,
aide,
ambassador,
architect,
artist,
astronaut,
astronomer,
athlete,
attendant,
attorney,
author,
babysitter,
baker,
biologist,
broker,
builder,
butcher,
butler,
captain,
cardiologist,
caregiver,
carpenter,
cashier,
caterer,
chauffeur,
chef,
chemist,
clerk,
coach,
contractor,
cook,
cop,
cryptographer,
dentist,
detective,
dictator,
director,
driver,
ecologist,
economist,
editor,
educator,
electrician,
engineer,
entrepreneur,
executive,
farmer,
financier,
firefighter,
gardener,
general,
geneticist,
geologist,
golfer,
governor,
grocer,
guard,
hairdresser,
housekeeper,
hunter,
inspector,
instructor,
intern,
interpreter,
inventor,
investigator,
janitor,
jester,
journalist,
judge,
laborer,
landlord,
lawyer,
lecturer,
librarian,
lifeguard,
linguist,
lobbyist,
magician,
manager,
manufacturer,
marine,
marketer,
mason,
mathematician,
mayor,
mechanic,
messenger,
miner,
model,
musician,
novelist,
official,
operator,
optician,
painter,
paralegal,
pathologist,
pediatrician,
pharmacist,
philosopher,
photographer,
physician,
physicist,
pianist,
pilot,
plumber,
poet,
politician,
postmaster,
president,
principal,
producer,
professor,
programmer,
psychiatrist,
psychologist,
publisher,
radiologist,
receptionist,
reporter,
representative,
researcher,
retailer,
sailor,
salesperson,
scholar,
senator,
sheriff,
singer,
soldier,
spy,
statistician,
stockbroker,
supervisor,
surgeon,
surveyor,
tailor,
teacher,
technician,
trader,
translator,
tutor,
undertaker,
valet,
veterinarian,
violinist,
waiter,
warden,
warrior,
watchmaker,
writer,
zookeeper,
zoologist

\noindent \textbf{Adherents$_{\textrm{Test}}$}
  adventist, 
  anabaptist, 
  anglican, 
  buddhist, 
  confucian, 
  jain, 
  jew, 
  lutheran, 
  mormon, 
  muslim, 
  rastafarian, 
  satanist, 
  scientologist, 
  shinto, 
  sikh, 
  sunni,
  taoist

\noindent \textbf{Demonyms$_{\textrm{Test}}$}
  belarusian, 
  brazilian, 
  british, 
  canadian, 
  danish, 
  dutch, 
  emirati, 
  georgian, 
  greek, 
  indian, 
  iranian, 
  iraqi, 
  irish, 
  italian, 
  japanese, 
  libyan, 
  moroccan, 
  nigerian, 
  peruvian, 
  qatari, 
  russian, 
  saudi, 
  scottish, 
  swiss, 
  thai, 
  turkish, 
  ukrainian, 
  uzbekistani, 
  vietnamese, 
  welsh, 
  yemeni, 
  zambian

\noindent \textbf{Countries}
	america, 
	belarus, 
	brazil, 
	britain, 
	canada, 
	china, 
	denmark, 
	egypt, 
	emirates, 
	france, 
	georgia, 
	germany, 
	greece, 
	india, 
	iran, 
	iraq, 
	ireland, 
	italy, 
	japan, 
	korea, 
	libya, 
	morocco, 
	netherlands, 
	nigeria, 
	pakistan, 
	peru, 
	qatar, 
	russia, 
	scotland, 
	spain, 
	switzerland, 
	thailand, 
	turkey, 
	ukraine, 
	uzbekistan, 
	vietnam, 
	wales, 
	yemen, 
	zambia

\noindent \textbf{Polarity}
  awful, 
  dishonest, 
  dumb, 
  evil, 
  great, 
  greedy, 
  hateful, 
  honest, 
  humorless, 
  ignorant, 
  intelligent, 
  intolerant, 
  neat, 
  nice, 
  professional, 
  rude, 
  smart, 
  strong, 
  stupid, 
  terrible, 
  terrible, 
  ugly, 
  unclean, 
  unprofessional, 
  weak, 
  wise

\noindent \textbf{Religions}
 	adventism, 
 	anabaptism, 
 	anglicism, 
 	atheism, 
 	baptism, 
 	buddhism, 
 	catholicism, 
 	christianity, 
 	confucianism, 
 	hinduism, 
 	islam, 
 	jainism, 
 	judaism, 
 	lutheranism, 
 	methodism, 
 	mormonism, 
 	protestantism, 
 	rastafarianism, 
 	satanism, 
 	scientology, 
 	sikhism, 
 	sunnism, 
 	taoism

\noindent \textbf{Objects}
	apple, 
	apron, 
	armchair, 
	auto, 
	bagel, 
	banana, 
	bed, 
	bench, 
	beret, 
	blender, 
	blouse, 
	bookshelf, 
	breakfast, 
	brownie, 
	buffalo, 
	burger, 
	bus, 
	cabinet, 
	cake, 
	calculator, 
	calf, 
	camera, 
	cap, 
	cape, 
	car, 
	cart, 
	cat, 
	chair, 
	chicken, 
	clock, 
	coat, 
	computer, 
	costume, 
	cot, 
	couch, 
	cow, 
	cupboard, 
	dinner, 
	dog, 
	donkey, 
	donut, 
	dress, 
	dresser, 
	duck, 
	goat, 
	headphones, 
	heater, 
	helmet, 
	hen, 
	horse, 
	jacket, 
	jeep, 
	lamb, 
	lamp, 
	lantern, 
	laptop, 
	lunch, 
	mango, 
	meal, 
	muffin, 
	mule, 
	oven, 
	ox, 
	pancake, 
	peach, 
	phone, 
	pig, 
	pizza, 
	potato, 
	printer, 
	pudding, 
	rabbit, 
	radio, 
	recliner, 
	refrigerator, 
	ring, 
	roll, 
	rug, 
	salad, 
	sandwich, 
	shirt, 
	shoe, 
	sofa, 
	soup, 
	stapler, 
	SUV, 
	table, 
	television, 
	toaster, 
	train, 
	tux, 
	TV, 
	van, 
	wagon, 
	watch
	
\noindent \textbf{Verbs}
	ate, 
	befriended, 
	bought, 
	budgeted for, 
	called, 
	can afford, 
	consumed, 
	cooked, 
	crashed, 
	donated, 
	drove, 
	finished, 
	hated, 
	identified, 
	interrupted, 
	liked, 
	loved, 
	met, 
	owns, 
	paid for, 
	prepared, 
	saved, 
	sold, 
	spoke to, 
	swapped, 
	traded, 
	visited
	
\section{Word Lists for WEAT Tests}
WEAT tests use sets of words that we list here~\cite{Caliskan183}.
\label{app: words weat}

\noindent \textbf{Gendered Words }

\noindent Male Gendered: male, man, boy, brother, him, his, son

\noindent Female Gendered: female, woman, girl, sister, her, hers, daughter

\noindent \textbf{WEAT 1: Occupations}

\noindent Stereotypically male: engineer, lawyer, mathematician

\noindent Stereotypically female: receptionist, homemaker, nurse

\noindent \textbf{WEAT 2: Work versus Home}

\noindent Work terms: executive, management, professional, corporation, salary, office, business, career

\noindent Home terms: home, parents, children, family, cousins, marriage, wedding, relatives

\noindent \textbf{WEAT 3: Math versus Art}

\noindent Math terms: math, algebra, geometry, calculus, equations, computation, numbers, addition

\noindent Art terms: poetry, art, dance, literature, novel, symphony, drama, sculpture

\section{Word lists for WEAT* Tests}
For the WEAT* test developed in Chapter \ref{chap: bias paper 3}, we list the sets of words used.
\label{app : words weat*}

\noindent \textbf{Definitionally Gendered Words for WEAT*}

\noindent Male: 
actor,
bachelor,
bridegroom,
brother,
count,
czar,
dad,
daddy,
duke,
emperor,
father,
fiance,
gentleman,
giant,
god,
governor,
grandfather,
grandson,
headmaster,
heir,
hero,
host,
hunter,
husband,
king,
lad,
landlord,
lord,
male,
manager,
manservant,
masseur,
master,
mayor,
milkman,
millionaire,
monitor,
monk,
mr,
murderer,
nephew,
papa,
poet,
policeman,
postman,
postmaster,
priest,
prince,
shepherd,
sir,
stepfather,
stepson,
steward,
sultan,
uncle,
waiter,
washerman,
widower,
wizard,

\noindent Female:
actress,
spinster,
bride,
sister,
countess,
czarina,
mum,
mummy,
duchess,
empress,
mother,
fiancee,
lady,
giantess,
goddess,
matron,
grandmother,
granddaughter,
headmistress,
heiress,
heroine,
hostess,
huntress,
wife,
queen,
lass,
landlady,
lady,
female,
manageress,
maidservant,
masseuse,
mistress,
mayoress,
milkmaid,
millionairess,
monitress,
nun,
mrs,
murderess,
niece,
mama,
poetess,
policewoman,
postwoman,
postmistress,
priestess,
princess,
shepherdess,
madam,
stepmother,
stepdaughter,
stewardess,
sultana,
aunt,
waitress,
washerwoman,
widow,
witch







\noindent \textbf{Gendered Names for WEAT* }

\noindent There are generated using data curated by social security in USA (\url{https://www.ssa.gov/oact/babynames/}).  We take the top 1000 gendered names in each gender which are also among the top 100K most frequent words in Wikipedia (to ensure robustly embedded words). 

\noindent Male:
liam,
noah,
william,
james,
logan,
benjamin,
mason,
elijah,
oliver,
jacob,
lucas,
michael,
alexander,
ethan,
daniel,
matthew,
aiden,
henry,
joseph,
jackson,
samuel,
sebastian,
david,
carter,
wyatt,
jayden,
john,
owen,
dylan,
luke,
gabriel,
anthony,
isaac,
grayson,
jack,
julian,
levi,
christopher,
joshua,
andrew,
lincoln,
mateo,
ryan,
nathan,
aaron,
isaiah,
thomas,
charles,
caleb,
josiah,
christian,
hunter,
eli,
jonathan,
connor,
landon,
adrian,
asher,
cameron,
leo,
theodore,
jeremiah,
hudson,
robert,
easton,
nolan,
nicholas,
ezra,
colton,
angel,
jordan,
dominic,
austin,
ian,
adam,
elias,
jose,
ezekiel,
carson,
evan,
maverick,
bryson,
jace,
cooper,
xavier,
parker,
roman,
jason,
santiago,
chase,
sawyer,
gavin,
leonardo,
jameson,
kevin,
bentley,
zachary,
everett,
axel,
tyler,
micah,
vincent,
weston,
miles,
wesley,
nathaniel,
harrison,
brandon,
cole,
declan,
luis,
braxton,
damian,
silas,
tristan,
ryder,
bennett,
george,
emmett,
justin,
kai,
max,
diego,
luca,
carlos,
maxwell,
kingston,
ivan,
maddox,
juan,
ashton,
rowan,
giovanni,
eric,
jesus,
calvin,
abel,
king,
camden,
amir,
blake,
alex,
brody,
malachi,
emmanuel,
jonah,
beau,
jude,
antonio,
alan,
elliott,
elliot,
waylon,
xander,
timothy,
victor,
bryce,
finn,
brantley,
edward,
abraham,
patrick,
grant,
hayden,
richard,
miguel,
joel,
gael,
tucker,
rhett,
avery,
steven,
graham,
jasper,
jesse,
matteo,
dean,
preston,
august,
oscar,
jeremy,
alejandro,
marcus,
dawson,
lorenzo,
messiah,
zion,
maximus,
river,
zane,
mark,
brooks,
nicolas,
paxton,
judah,
emiliano,
bryan,
kyle,
myles,
peter,
charlie,
kyrie,
thiago,
brian,
kenneth,
andres,
lukas,
aidan,
jax,
caden,
milo,
paul,
beckett,
brady,
colin,
omar,
bradley,
javier,
knox,
jaden,
barrett,
israel,
matias,
jorge,
zander,
derek,
holden,
griffin,
arthur,
leon,
felix,
remington,
jake,
killian,
clayton,
sean,
riley,
archer,
legend,
erick,
enzo,
corbin,
francisco,
dallas,
emilio,
gunner,
simon,
andre,
walter,
damien,
chance,
phoenix,
colt,
tanner,
stephen,
tobias,
manuel,
amari,
emerson,
louis,
cody,
finley,
martin,
rafael,
nash,
beckham,
cash,
reid,
theo,
ace,
eduardo,
spencer,
raymond,
maximiliano,
anderson,
ronan,
lane,
cristian,
titus,
travis,
jett,
ricardo,
bodhi,
gideon,
fernando,
mario,
conor,
keegan,
ali,
cesar,
ellis,
walker,
cohen,
arlo,
hector,
dante,
garrett,
donovan,
seth,
jeffrey,
tyson,
jase,
desmond,
gage,
atlas,
major,
devin,
edwin,
angelo,
orion,
conner,
julius,
marco,
jensen,
peyton,
zayn,
collin,
dakota,
prince,
johnny,
cruz,
hendrix,
atticus,
troy,
kane,
edgar,
sergio,
kash,
marshall,
johnathan,
romeo,
shane,
warren,
joaquin,
wade,
leonel,
trevor,
dominick,
muhammad,
erik,
odin,
quinn,
dalton,
nehemiah,
frank,
grady,
gregory,
andy,
solomon,
malik,
rory,
clark,
reed,
harvey,
jay,
jared,
noel,
shawn,
fabian,
ibrahim,
adonis,
ismael,
pedro,
leland,
malcolm,
alexis,
porter,
sullivan,
raiden,
allen,
ari,
russell,
princeton,
winston,
kendrick,
roberto,
lennox,
hayes,
finnegan,
nasir,
kade,
nico,
emanuel,
landen,
moises,
ruben,
hugo,
abram,
adan,
khalil,
augustus,
marcos,
philip,
phillip,
cyrus,
esteban,
albert,
bruce,
lawson,
jamison,
sterling,
damon,
gunnar,
luka,
franklin,
ezequiel,
pablo,
derrick,
zachariah,
cade,
jonas,
dexter,
remy,
hank,
tate,
trenton,
kian,
drew,
mohamed,
dax,
rocco,
bowen,
mathias,
ronald,
francis,
matthias,
milan,
maximilian,
royce,
skyler,
corey,
drake,
gerardo,
jayson,
sage,
benson,
moses,
rhys,
otto,
oakley,
armando,
jaime,
nixon,
saul,
scott,
ariel,
enrique,
donald,
chandler,
asa,
eden,
davis,
keith,
frederick,
lawrence,
leonidas,
aden,
julio,
darius,
johan,
deacon,
cason,
danny,
nikolai,
taylor,
alec,
royal,
armani,
kieran,
luciano,
omari,
rodrigo,
arjun,
ahmed,
brendan,
cullen,
raul,
raphael,
ronin,
brock,
pierce,
alonzo,
casey,
dillon,
uriel,
dustin,
gianni,
roland,
kobe,
dorian,
emmitt,
ryland,
apollo,
roy,
duke,
quentin,
sam,
lewis,
tony,
uriah,
dennis,
moshe,
braden,
quinton,
cannon,
mathew,
niko,
edison,
jerry,
gustavo,
marvin,
mauricio,
ahmad,
mohammad,
justice,
trey,
mohammed,
sincere,
yusuf,
arturo,
callen,
keaton,
wilder,
memphis,
conrad,
soren,
colby,
bryant,
lucian,
alfredo,
cassius,
marcelo,
nikolas,
brennan,
darren,
jimmy,
lionel,
reece,
ty,
chris,
forrest,
tatum,
jalen,
santino,
case,
leonard,
alvin,
issac,
bo,
quincy,
mack,
samson,
rex,
alberto,
callum,
curtis,
hezekiah,
briggs,
zeke,
neil,
titan,
julien,
kellen,
devon,
roger,
axton,
carl,
douglas,
larry,
crosby,
fletcher,
makai,
nelson,
hamza,
lance,
alden,
gary,
wilson,
alessandro,
ares,
bruno,
jakob,
stetson,
zain,
cairo,
nathanael,
byron,
harry,
harley,
mitchell,
maurice,
orlando,
kingsley,
trent,
ramon,
boston,
lucca,
noe,
jagger,
randy,
thaddeus,
lennon,
kannon,
kohen,
valentino,
salvador,
langston,
rohan,
kristopher,
yosef,
lee,
callan,
tripp,
deandre,
joe,
morgan,
reese,
ricky,
bronson,
terry,
eddie,
jefferson,
lachlan,
layne,
clay,
madden,
tomas,
kareem,
stanley,
amos,
kase,
kristian,
clyde,
ernesto,
tommy,
ford,
crew,
hassan,
axl,
boone,
leandro,
samir,
magnus,
abdullah,
yousef,
branson,
layton,
franco,
ben,
grey,
kelvin,
chaim,
demetrius,
blaine,
ridge,
colson,
melvin,
anakin,
aryan,
jon,
canaan,
dash,
zechariah,
alonso,
otis,
zaire,
marcel,
brett,
stefan,
aldo,
jeffery,
baylor,
talon,
dominik,
flynn,
carmelo,
dane,
jamal,
kole,
enoch,
kye,
vicente,
fisher,
ray,
fox,
jamie,
rey,
zaid,
allan,
emery,
gannon,
rodney,
sonny,
terrance,
augustine,
cory,
felipe,
aron,
jacoby,
harlan

\noindent Female:
emma,
olivia,
ava,
isabella,
sophia,
mia,
charlotte,
amelia,
evelyn,
abigail,
harper,
emily,
elizabeth,
avery,
sofia,
ella,
madison,
scarlett,
victoria,
aria,
grace,
chloe,
camila,
penelope,
riley,
layla,
lillian,
nora,
zoey,
mila,
aubrey,
hannah,
lily,
addison,
eleanor,
natalie,
luna,
savannah,
brooklyn,
leah,
zoe,
stella,
hazel,
ellie,
paisley,
audrey,
skylar,
violet,
claire,
bella,
aurora,
lucy,
anna,
samantha,
caroline,
genesis,
aaliyah,
kennedy,
kinsley,
allison,
maya,
sarah,
adeline,
alexa,
ariana,
elena,
gabriella,
naomi,
alice,
sadie,
hailey,
eva,
emilia,
autumn,
quinn,
piper,
ruby,
serenity,
willow,
everly,
cora,
lydia,
arianna,
eliana,
peyton,
melanie,
gianna,
isabelle,
julia,
valentina,
nova,
clara,
vivian,
reagan,
mackenzie,
madeline,
delilah,
isla,
katherine,
sophie,
josephine,
ivy,
liliana,
jade,
maria,
taylor,
hadley,
kylie,
emery,
natalia,
annabelle,
faith,
alexandra,
ximena,
ashley,
brianna,
bailey,
mary,
athena,
andrea,
leilani,
jasmine,
lyla,
margaret,
alyssa,
arya,
norah,
kayla,
eden,
eliza,
rose,
ariel,
melody,
alexis,
isabel,
sydney,
juliana,
lauren,
iris,
emerson,
london,
morgan,
lilly,
charlie,
aliyah,
valeria,
arabella,
sara,
finley,
trinity,
jocelyn,
kimberly,
esther,
molly,
valerie,
cecilia,
anastasia,
daisy,
reese,
laila,
mya,
amy,
amaya,
elise,
harmony,
paige,
fiona,
alaina,
nicole,
genevieve,
lucia,
alina,
mckenzie,
callie,
payton,
eloise,
brooke,
mariah,
julianna,
rachel,
daniela,
gracie,
catherine,
angelina,
presley,
josie,
harley,
vanessa,
parker,
juliette,
amara,
marley,
lila,
ana,
rowan,
alana,
michelle,
malia,
rebecca,
summer,
sloane,
leila,
sienna,
adriana,
sawyer,
kendall,
juliet,
destiny,
diana,
hayden,
ayla,
dakota,
angela,
noelle,
rosalie,
joanna,
lola,
georgia,
selena,
june,
tessa,
maggie,
jessica,
remi,
delaney,
camille,
vivienne,
hope,
mckenna,
gemma,
olive,
alexandria,
blakely,
catalina,
gabrielle,
lucille,
ruth,
evangeline,
blake,
thea,
amina,
giselle,
melissa,
river,
kate,
adelaide,
vera,
leia,
gabriela,
zara,
jane,
journey,
miriam,
stephanie,
cali,
ember,
logan,
annie,
mariana,
kali,
haven,
elsie,
paris,
lena,
freya,
lyric,
camilla,
sage,
jennifer,
talia,
alessandra,
juniper,
fatima,
amira,
arielle,
phoebe,
ada,
nina,
samara,
cassidy,
aspen,
allie,
keira,
kaia,
amanda,
heaven,
joy,
lia,
laura,
lexi,
haley,
miranda,
kaitlyn,
daniella,
felicity,
jacqueline,
evie,
angel,
danielle,
ainsley,
dylan,
kiara,
millie,
jordan,
maddison,
alicia,
maeve,
margot,
phoenix,
heidi,
alondra,
lana,
madeleine,
kenzie,
miracle,
shelby,
elle,
adrianna,
bianca,
kira,
veronica,
gwendolyn,
esmeralda,
chelsea,
alison,
skyler,
magnolia,
daphne,
jenna,
kyla,
harlow,
annalise,
dahlia,
scarlet,
luciana,
kelsey,
nadia,
amber,
gia,
carmen,
jimena,
erin,
christina,
katie,
ryan,
viviana,
alexia,
anaya,
serena,
ophelia,
regina,
helen,
remington,
cadence,
royalty,
amari,
kathryn,
skye,
jada,
saylor,
kendra,
cheyenne,
fernanda,
sabrina,
francesca,
eve,
mckinley,
frances,
sarai,
carolina,
tatum,
lennon,
raven,
leslie,
winter,
abby,
mabel,
sierra,
april,
willa,
carly,
jolene,
rosemary,
selah,
renata,
lorelei,
briana,
celeste,
wren,
leighton,
annabella,
mira,
oakley,
malaysia,
edith,
maryam,
hattie,
bristol,
demi,
maia,
sylvia,
allyson,
lilith,
holly,
meredith,
nia,
liana,
megan,
justice,
bethany,
alejandra,
janelle,
elisa,
adelina,
myra,
blair,
charley,
virginia,
kara,
helena,
sasha,
julie,
michaela,
carter,
matilda,
henley,
maisie,
hallie,
priscilla,
marilyn,
cecelia,
danna,
colette,
elliott,
cameron,
celine,
hanna,
imani,
angelica,
kalani,
alanna,
lorelai,
macy,
karina,
aisha,
johanna,
mallory,
leona,
mariam,
karen,
karla,
beatrice,
gloria,
milani,
savanna,
rory,
giuliana,
lauryn,
liberty,
charli,
jillian,
anne,
dallas,
azalea,
tiffany,
shiloh,
jazmine,
esme,
elaine,
lilian,
kyra,
kora,
octavia,
irene,
kelly,
lacey,
laurel,
anika,
dorothy,
sutton,
julieta,
kimber,
remy,
cassandra,
rebekah,
collins,
elliot,
emmy,
sloan,
hayley,
amalia,
jemma,
jamie,
melina,
leyla,
wynter,
alessia,
monica,
anya,
antonella,
ivory,
greta,
maren,
alena,
emory,
cynthia,
alia,
angie,
alma,
crystal,
aileen,
siena,
zelda,
marie,
pearl,
reyna,
mae,
zahra,
jessie,
tiana,
armani,
lennox,
lillie,
jolie,
laney,
mara,
joelle,
rosa,
bridget,
liv,
aurelia,
clarissa,
elyse,
marissa,
monroe,
kori,
elsa,
rosie,
amelie,
eileen,
poppy,
royal,
chaya,
frida,
bonnie,
amora,
stevie,
tatiana,
malaya,
mina,
reign,
annika,
linda,
kenna,
faye,
reina,
brittany,
marina,
astrid,
briar,
teresa,
hadassah,
guadalupe,
rayna,
chanel,
lyra,
noa,
laylah,
livia,
ellen,
meadow,
ellis,
milan,
hunter,
princess,
nathalie,
clementine,
nola,
simone,
lina,
marianna,
martha,
louisa,
emmeline,
kenley,
belen,
erika,
lara,
amani,
ansley,
salma,
dulce,
nala,
natasha,
mercy,
penny,
ariadne,
deborah,
elisabeth,
zaria,
hana,
raina,
lexie,
thalia,
annabel,
christine,
estella,
adele,
aya,
estelle,
landry,
tori,
perla,
miah,
angelique,
romina,
ari,
jaycee,
kai,
louise,
mavis,
belle,
lea,
rivka,
calliope,
sky,
jewel,
paola,
giovanna,
isabela,
azariah,
dream,
claudia,
corinne,
erica,
milena,
alyson,
joyce,
tinsley,
whitney,
carolyn,
frankie,
andi,
judith,
paula,
amia,
hadlee,
rayne,
cara,
celia,
opal,
clare,
gwen,
veda,
alisha,
davina,
rhea,
noor,
danica,
kathleen,
lindsey,
maxine,
paulina,
nancy,
raquel,
zainab,
chana,
lisa,
heavenly,
patricia,
india,
paloma,
ramona,
sandra,
abril,
vienna,
rosalyn,
hadleigh,
barbara,
jana,
brenda,
casey,
selene,
adrienne,
aliya,
miley,
bexley,
joslyn,
zion,
breanna,
melania,
estrella,
ingrid,
jayden,
kaya,
dana,
legacy,
marjorie,
courtney,
holland

\section{Word Lists for HD}
\label{app : words HD}

\textbf{Gendered Words.} These are the dictionary defined gendered words used by the operation to determine what words are correctly gendered and which should be neutral in the embedding space. Here is the filtered version of the list used in our experiments as per our description in Section \ref{sec: debiasing methods} about the test/train word list split.

\noindent actress, actresses, aunt, aunts, bachelor, ballerina, barbershop, baritone, beard, beards, beau, bloke, blokes, boy, boyfriend, boyfriends, boyhood, boys, brethren, bride, brides, brother, brotherhood, brothers, bull, bulls, businessman, businessmen, businesswoman, chairman, chairwoman, chap, colt, colts, congressman, congresswoman, convent, councilman, councilmen, councilwoman, countryman, countrymen, czar, dad, daddy, dads, daughter, daughters, deer, diva, dowry, dude, dudes,  estrogen,  fathered, fatherhood, fathers, fella, fellas, females, feminism, fiance, fiancee, fillies, filly, fraternal, fraternities, fraternity, gal, gals, gelding, gentlemen, girlfriend, girlfriends, girls, goddess, godfather, granddaughter, granddaughters, grandma,grandmothers, grandpa, grandson, grandsons,  handyman, heiress, hen, hens, her, heroine, hers, herself, him, himself, his, horsemen, hostess, housewife, housewives, hubby, husband, husbands,  kings, lad, ladies, lads, lesbian, lesbians, lion, lions, ma, macho, maid, maiden, maids, males, mama,  mare, maternal, maternity, men, menopause, mistress, mom, mommy, moms, monastery, monk, monks,  motherhood, mothers, nephew, nephews, niece, nieces, nun, nuns, obstetrics,  pa, paternity, penis, prince, princes, princess, prostate, queens, salesman, salesmen, schoolboy, schoolgirl, semen, she, sir, sister, sisters, son, sons, sorority, sperm, spokesman, spokesmen, spokeswoman, stallion, statesman, stepdaughter, stepfather, stepmother, stepson, strongman, stud, studs, suitor, suitors, testosterone, uncle, uncles, uterus, vagina, viagra, waitress, widow, widower, widows, wife, witch, witches, wives, womb, women

\textbf{Equalized Words.} These gendered words are paired (one male, one female) and are equalized by the operation. Here is the filtered version of the list used in our experiments as per our description in Section \ref{sec: debiasing methods} about the test/train word list split. Each pair in this list is ``equalized''.

\noindent monastery convent\\ spokesman spokeswoman\\ priest nun\\ Dad Mom\\ Men Women\\ councilman councilwoman\\ grandpa grandma\\ grandsons granddaughters\\  testosterone estrogen\\ uncle aunt\\ wives husbands\\ Father Mother\\ Grandpa Grandma\\ He She\\ boys girls\\ brother sister\\ brothers sisters\\ businessman businesswoman\\ chairman chairwoman\\ colt filly\\ congressman congresswoman\\ dad mom\\ dads moms\\ dudes gals\\   fatherhood motherhood\\ fathers mothers\\ fella granny\\ fraternity sorority\\ gelding mare\\ gentlemen ladies\\  grandson granddaughter\\ himself herself\\ his her\\ king queen\\ kings queens\\  males females\\  men women\\ nephew niece\\ prince princess\\ schoolboy schoolgirl\\ son daughter\\ sons daughters

More details are available at: \url{//github.com/tolga-b/debiaswe}

\section{Word Lists for INLP}
\label{app : words INLP}

The Gendered Word List ($G$) for INLP consists of 1425 words found under \url{\texttt{https://github.com/Shaul1321/nullspace\_projection/blob/master/data/lists/}} as the list \texttt{gender\_specific\_full.json}. This list has been filtered of words used in generating templates (Supplement
\ref{app : words templates} and for WEAT (Supplement \ref{app: words weat}.

More details about their word lists and code is available at: \url{\texttt{https://github.com/Shaul1321/nullspace\_projection}}.